\documentclass{article}
\usepackage{boxinz-macros}
\usepackage{graphicx}
\usepackage{url}
\usepackage{authblk}
\usepackage{dsfont}
\usepackage{float}
\usepackage[colorlinks]{hyperref}
\usepackage{natbib}
\usepackage{xcolor}

\providecommand{\keywords}[1]{\textbf{Keywords:} #1}
\usepackage{placeins}
\usepackage{enumitem}
\usepackage{amsmath}
\usepackage{xspace}
\usepackage{mathabx}
\usepackage{subcaption}
\usepackage[colorlinks]{hyperref}
\usepackage{xcolor}
\usepackage{placeins}
\usepackage{enumitem}
\usepackage[compact]{titlesec}
\usepackage{bibunits}
\usepackage{wrapfig}
\usepackage{amsmath}

\newcommand{\RegM}{\mathbf{C}}
\newcommand{\diag}{\mathrm{diag}}

\allowdisplaybreaks
\title{SMART: A Spectral Transfer Approach to Multi-Task Learning}

\author[1]{Boxin Zhao}
\author[2]{Mladen Kolar}
\author[3]{Jinchi Lv}

\affil[1]{University of Chicago}
\affil[2]{University of Southern California}
\affil[3]{University of Southern California}

\date{}

\begin{document}

\begin{bibunit}[plainnat]

\maketitle

\begin{abstract}
\normalsize

Multi-task learning is effective for related applications, but its performance can deteriorate when the target sample size is small. Transfer learning can borrow strength from related studies; yet, many existing methods rely on restrictive bounded-difference assumptions between the source and target models. We propose SMART, a spectral transfer method for multi-task linear regression that instead assumes spectral similarity: the target left and right singular subspaces lie within the corresponding source subspaces and are sparsely aligned with the source singular bases. Such an assumption is natural when studies share latent structures and enables transfer beyond the bounded-difference settings. SMART estimates the target coefficient matrix through structured regularization that incorporates spectral information from a source study. Importantly, it requires only a fitted source model rather than the raw source data, making it useful when data sharing is limited. Although the optimization problem is nonconvex, we develop a practical ADMM-based algorithm. We establish general, non-asymptotic error bounds and a minimax lower bound in the noiseless-source regime. Under additional regularity conditions, these results yield near-minimax Frobenius error rates up to logarithmic factors. Simulations confirm improved estimation accuracy and robustness to negative transfer, and analysis of multi-modal single-cell data demonstrates better predictive performance.
The Python implementation of SMART, along with the code to reproduce all experiments in this paper, is publicly available at \url{https://github.com/boxinz17/smart}.

\end{abstract}


\section{Introduction}

Many modern scientific studies require estimating multiple related regression models simultaneously. A particularly important example arises in genomics, where one seeks to characterize the gene--protein association networks across tissues or cell types~\citep{tang2023explainable}. Recent single-cell multi-modal technologies now measure gene expression (mRNA) and protein abundance (surface markers) in the same cell~\citep{luecken2021sandbox}, producing paired high-dimensional observations across modalities. A natural target is therefore a gene--protein association matrix that summarizes the dominant cross-modal relationships. Since cellular systems are often driven by a relatively small number of coordinated functional programs, these cross-modal associations are well-modeled by a low-rank coefficient matrix, with each latent factor corresponding to a gene--protein module.

Such settings also motivate learning across related studies. Multi-task learning can improve estimation by borrowing strength across tasks, but its gains are often limited precisely when transfer is most needed: some studies have only modest sample sizes. Such imbalances are routine in biomedical data, where rare cell types, disease-specific states, or difficult experimental conditions are sparsely observed, whereas other populations are much better sampled. In gene--protein association studies across cell types, for example, a data-poor target population may still share core biological programs with a data-rich source population. Transfer learning is thus a natural strategy for improving estimation in the target study by leveraging information learned from related sources~\citep{pan2009survey}.

Two obstacles make this problem challenging. First, many high-dimensional transfer-learning methods assume a \emph{bounded-difference assumption}, namely that the source and target coefficients are close under some metric~\citep{li2023estimation,park2025transfer}. In the multivariate low-rank setting, \citet{park2025transfer} required the contrast between the target coefficient matrix and each informative source to have small nuclear norm, which ties together similarity in singular directions and similarity in effect magnitude. This can become restrictive in complex biological systems: large shifts in individual coefficients or singular values may occur even when the underlying low-dimensional modules are largely preserved. For instance, gene–protein associations may vary substantially across cell types or treatment conditions, whereas pathways of co-regulated genes and protein complexes remain aligned, leading to a large nuclear-norm contrast despite strong subspace similarity. Second, many transfer-learning procedures need direct access to the source data~\citep{li2022transfer}, whereas in practice, privacy, consent, or platform restrictions may leave only a fitted source model or other summary representation available.

Motivated by these challenges, we propose a spectral notion of transfer specifically designed for low-rank multi-task regression. Instead of insisting that source and target coefficients be numerically close, we require a \emph{spectral containment condition}, which states that the left and right singular subspaces of the target coefficient matrix lie within those of a clean source matrix, along with a \emph{sparse alignment condition}, which stipulates that the target singular vectors admit sparse representations in the source singular bases. These assumptions \textit{formalize} a biologically relevant setting in which the target repurposes a subset of latent modules already present in the source. They also constitute the \textit{central} conceptual innovation of our framework: in contrast to bounded-difference approaches, they permit arbitrarily large shifts in singular values and individual coefficients while preserving the shared low-dimensional structure that enables effective transfer.

Building on this new perspective, we propose \textsc{SMART}, a spectral transfer method for low-rank multi-task regression that uses only a fitted source regression matrix and injects its spectral information through structured regularization. This \textit{source-free} design is practically important when the raw source data cannot be shared. The resulting optimization problem is nonconvex, so we develop an efficient ADMM-based algorithm that addresses the orthogonality constraints through manifold optimization. On the theoretical side, we establish general, non-asymptotic error upper bounds, and a minimax lower bound in the noiseless-source regime; under additional conditions, these results yield near-minimax Frobenius error rates up to logarithmic factors. Empirically, \textsc{SMART} delivers clear gains in small-sample regimes in both simulations and a multi-modal single-cell application, demonstrating that subspace-level transfer can remain powerful even when the source and target effect magnitudes differ substantially.

\subsection{Related works}

\textbf{Multi-task learning.} Multi-task learning has a long history in computer science~\citep{caruana1997multitask}, especially in neural networks and representation learning~\citep{zhang2021survey}. More recently, the statistical literature has developed a rich theory under the multivariate linear regression framework. At its core is the reduced-rank regression (RRR), which models the coefficient matrix as low-rank to capture response dependence and improve efficiency~\citep{anderson1951estimating,izenman1975reduced}. Modern extensions incorporate sparsity, adaptivity, and scalability: \citet{yuan2007dimension} and \citet{chen2013reduced} studied nuclear-norm penalization for adaptive dimension reduction; the dirty model of \citet{jalali2010dirty} captured both shared and task-specific components; and \citet{bunea2011optimal,bunea2012joint} developed methods for joint variable and rank selection. Sparse formulations further enable simultaneous predictor and factor selection~\citep{chen2012sparse,ma2020adaptive}, while SOFAR extends structured singular value decomposition (SVD) to large-scale association learning~\citep{uematsu2019sofar}. Together, these works provide a flexible toolkit for modeling multi-task relationships through low-rank structure and sparsity. Our work builds on such literature but studies a \textit{distinct} problem: improving a target low-rank multivariate regression by transferring only structured spectral information from an external source model.

\noindent\textbf{Transfer learning.} Transfer learning has long been studied in computer science~\citep{pan2009survey} and has recently received growing attention in statistics, with applications in high-dimensional linear regression~\citep{li2022transfer}, generalized linear models~\citep{li2023estimation}, Gaussian graphical models~\citep{zhao2025trans}, and functional data~\citep{cai2024transferfunctional}. Many statistical approaches rely on a \emph{bounded-difference assumption}, requiring the source and target parameters to be close under a suitable metric. 
There is also a growing interest in source-free transfer learning, where the raw source data are unavailable, and transfer must rely on a pretrained source model; see, for example, \citet{gu2025robust}. More recently, structurally motivated transfer methods have also been developed in other problems, including a rank-based approach for high-dimensional survival analysis \citep{qiao2025rank} and a spectral transfer method for matrix completion \citep{he2024representational}. Although these works also exploit structural similarity across domains, our focus on multivariate low-rank regression leads to a \textit{different} notion of transferable structure, a \textit{different} nonconvex estimator, and a \textit{different} theoretical analysis.

The most closely related work is the concurrent study of \citet{park2025transfer}, which examines transfer learning for multivariate low-rank regression under a bounded nuclear-norm contrast condition between the target and each informative source coefficient matrix. This assumption essentially enforces similarity in both singular directions and effect magnitudes, and thus can \textit{rule out} scenarios where the signal undergoes large but directionally aligned shifts. By contrast, in the single-source setting considered here, we posit exact spectral containment of the target singular subspaces within those of a clean source matrix, along with sparse alignment of the target singular vectors in the source singular bases. This \textit{sharper} structural assumption disentangles subspace reuse from effect-size similarity, allows for arbitrarily large changes in singular values, and directly informs the regularization design and theoretical analysis of \textsc{SMART}. An illustrative example and additional discussion of the differences from \citet{park2025transfer} are provided in Appendix~\ref{sec:example-spectral-vs-nuclear}.

\subsection{Organization and notation}

\textbf{Organization.} The rest of the paper is organized as follows. Section~\ref{sec:problem-setup-smart} introduces the model setup and the spectral transfer framework. Section~\ref{sec:methodology-smart} presents \textsc{SMART}, including its ADMM-based optimization and hyperparameter selection. Section~\ref{sec:theory-analysis-smart} establishes theoretical guarantees, including finite-sample error bounds and a minimax lower bound. Section~\ref{sec:simulation-smart} reports simulation examples, Section~\ref{sec:real-world-data-smart} presents the multi-modal single-cell application, and Section~\ref{sec:conclusion-smart} concludes with some discussions. All the proofs and additional details on the experiments and optimization are provided in the Appendix.
The Python implementation of our method, along with the code to reproduce all experiments in this paper, is publicly available at
\begin{center}
\url{https://github.com/boxinz17/smart}
\end{center}

\noindent\textbf{Notation.} For a vector $\mathbf{v} \in \mathbb{R}^d$, define the $\ell_p$-norm as $\|\mathbf{v}\|_p = (\sum_{i=1}^d |v_i|^p)^{1/p}$ for $1 \leq p < \infty$, and $\|\mathbf{v}\|_\infty = \max_i |v_i|$. For a matrix $\mathbf{A} \in \mathbb{R}^{m \times n}$, we use $|\cdot|$ to denote entrywise norms and $\|\cdot\|$ the matrix norms: $|\mathbf{A}|_1 = \sum_{i,j} |\mathbf{A}_{ij}|$, $|\mathbf{A}|_0 = \sum_{i,j} \mathds{1}\{\mathbf{A}_{ij} \neq 0\}$, $|\mathbf{A}|_\infty = \max_{i,j} |\mathbf{A}_{ij}|$, $\|\mathbf{A}\|_2$ is the largest singular value, and $\|\mathbf{A}\|_{\mathrm{F}} = (\sum_{i,j} |\mathbf{A}_{ij}|^2)^{1/2}$. Denote by $\Vert \mathbf{A} \Vert_{1,2} \coloneqq \bigl( \sum^n_{j=1} \Vert \mathbf{A}_{:j} \Vert^2_1 \bigr)^{\frac{1}{2}}$. The inner product is $\langle \mathbf{A}, \mathbf{B} \rangle = \operatorname{tr}(\mathbf{A}^\top \mathbf{B}) = \sum_{i,j} \mathbf{A}_{ij}\mathbf{B}_{ij}$, and $\operatorname{vec}(\mathbf{A})$ denotes column-stacking. For a symmetric $\mathbf{A}$, $\lambda_{\min}(\mathbf{A})$ and $\lambda_{\max}(\mathbf{A})$ represent its smallest and largest eigenvalues. For $\mathbf{A} \in \mathbb{R}^{n_1 \times n_2}$ and $\mathbf{B} \in \mathbb{R}^{m_1 \times m_2}$, their Kronecker product is $\mathbf{A} \otimes \mathbf{B} \in \mathbb{R}^{n_1 m_1 \times n_2 m_2}$ with entries $[\mathbf{A}_{ij}\mathbf{B}_{\ell m}]_{i,j,\ell,m}$. Given $\mathcal{S} \subseteq [m]\times[n]$ and $\mathbf{A} \in \mathbb{R}^{m \times n}$, $\mathbf{A}_{\mathcal{S}}$ denotes the matrix with entries $\mathbf{A}_{ij}$ if $(i,j)\in \mathcal{S}$ and $0$ otherwise.

We adopt standard asymptotic notation: for sequences $\{f(n)\}$ and $\{g(n)\}$, $f(n)=O(g(n))$ or $f(n)\lesssim g(n)$ means that $f(n)\le Cg(n)$ for some $C>0$ and large $n$; $f(n)=\Omega(g(n))$ or $f(n)\gtrsim g(n)$ means that $f(n)\ge c g(n)$ for some $c>0$; $f(n)=\Theta(g(n))$ or $f(n)\asymp g(n)$ means that $c_1 g(n)\le f(n)\le c_2 g(n)$ for constants $c_1,c_2>0$; and $f(n)=o(g(n))$ means that $f(n)/g(n)\to 0$. The universal constants may vary across lines. We also write $a\wedge b=\min(a,b)$, $a\vee b=\max(a,b)$, and $\lfloor x\rfloor$ for the greatest integer $\le x$. Denote by $\mathbb{S}^{d\times d}$ the set of symmetric $d\times d$ matrices.

\section{Model setting and problem setup}
\label{sec:problem-setup-smart}

Let $\{ (\mathbf{x}_i, \mathbf{y}_i) \}_{i=1}^n$ be independent and identically distributed (i.i.d.) observations with covariates $\mathbf{x}_i \in \mathbb{R}^p$ and responses $\mathbf{y}_i \in \mathbb{R}^q$. We study the multi-task linear regression model
\begin{equation}
\label{eq:model-assump}
\mathbf{Y} = \mathbf{X} \RegM^{*} + \mathbf{E},
\end{equation}
where $\mathbf{Y} = [\mathbf{y}_1, \ldots, \mathbf{y}_n]^{\top} \in \mathbb{R}^{n \times q}$ is the response matrix, $\mathbf{X} = [\mathbf{x}_1, \ldots, \mathbf{x}_n]^{\top} \in \mathbb{R}^{n \times p}$ is the design matrix, and $\mathbf{E} = [\mathbf{e}_1, \ldots, \mathbf{e}_n]^{\top} \in \mathbb{R}^{n \times q}$ is the error matrix, which is assumed to be independent of $\mathbf{X}$. Each column of $\mathbf{Y}$ corresponds to one task across all observations, whereas each row records the responses for a single observation. The coefficient matrix $\RegM^{*} \in \mathbb{R}^{p \times q}$ encodes the linear relationships between the covariates and tasks.

To model the structure common to all tasks, we posit that $\RegM^{*}$ is low-rank and admits the singular value decomposition (SVD)
\begin{equation*}
\RegM^{*} = \mathbf{U}^{*} \mathbf{D}^{*} \mathbf{V}^{* \top} 
= \sum_{j=1}^{r} d^{*}_j \mathbf{u}^{*}_j \mathbf{v}^{* \top}_j,
\end{equation*}
where $r = \mathrm{rank}(\RegM^{*}) \leq \min(p,q)$, 
$\mathbf{D}^{*} = \mathrm{diag}(d^{*}_1, \ldots, d^{*}_r)$ collects the singular values $d^{*}_1 \geq \cdots \geq d^{*}_r > 0$, and $\mathbf{U}^{*} \in \mathbb{R}^{p \times r}$ and $\mathbf{V}^{*} \in \mathbb{R}^{q \times r}$ are orthonormal matrices containing the left and right singular vectors, respectively. The columns of $\mathbf{U}^{*}$ define a low-dimensional subspace of covariates that is pertinent to the tasks, whereas the columns of $\mathbf{V}^{*}$ encode structure shared among tasks. Unless noted otherwise, we treat $r$ as known. In applications, $r$ may be chosen following the selection rule of~\citet{bunea2011optimal}.

For example, in the gene–protein networks, the low-rank structure of $\RegM^*$ reflects the idea that a small number of latent regulatory programs drive coordinated gene and protein expression. Each latent component defines a gene–protein module representing a shared pathway through which subsets of genes influence subsets of proteins. Modeling $\RegM^*$ as low-rank therefore yields a parsimonious, interpretable description of cross-modal interactions while improving statistical efficiency and robustness in high-dimensional single-cell data.

In addition to the target dataset, we assume access to a source regression matrix $\RegM^{(0)} \in \mathbb{R}^{p \times q}$ obtained from a related biological context, such as a different cell type or experimental condition within the same gene–protein framework. For instance, in single-cell multi-omics, one may leverage previously estimated gene–protein associations from a reference cell type to improve estimation in a new but related cell type. Let the singular vector matrices corresponding to $\RegM^{(0)}$ be orthonormal and ordered by singular values, and write
\begin{equation*}
\RegM^{(0)} = \mathbf{U}^{(0)} \mathbf{D}^{(0)} \mathbf{V}^{(0)\top} 
= \sum_{j=1}^{r_0} d_j^{(0)} \mathbf{u}^{(0)}_j \mathbf{v}^{(0)\top}_j,
\end{equation*}
where $1 \leq r_0 \leq \min(p,q)$ is the rank, $\mathbf{U}^{(0)} = [\mathbf{u}^{(0)}_1, \ldots, \mathbf{u}^{(0)}_{r_0}] \in \mathbb{R}^{p \times r_0}$ and $\mathbf{V}^{(0)} = [\mathbf{v}^{(0)}_1, \ldots, \mathbf{v}^{(0)}_{r_0}] \in \mathbb{R}^{q \times r_0}$ are orthonormal singular vector matrices, and $\mathbf{D}^{(0)} = \mathrm{diag}(d^{(0)}_1, \ldots, d^{(0)}_{r_0})$ contains non-increasing positive singular values $d^{(0)}_1 \geq \cdots \geq d^{(0)}_{r_0} > 0$. The alignment measures defined below are specified with respect to this fixed selection of singular vectors. If the singular values of the source are all distinct, this basis is unique except for an overall sign ambiguity.

To improve the estimation of the target coefficient matrix $\RegM^*$ using auxiliary information from the source matrix $\RegM^{(0)}$, we impose some structural assumptions that capture their spectral similarity. Our framework targets settings where the main latent factors of the target are inherited from or aligned with those of the source. Formally, we assume the singular subspaces of the target matrix are contained in those of the source.
\begin{assump}[Spectral containment]
\label{assump:relation-target-source}
Assume that $r_0 \ge r$, and
\[
\mathrm{col}(\mathbf{U}^*) \subseteq \mathrm{col}(\mathbf{U}^{(0)}) 
\quad \text{and} \quad 
\mathrm{col}(\mathbf{V}^*) \subseteq \mathrm{col}(\mathbf{V}^{(0)}).
\]
\end{assump}

The condition above is biologically plausible in many settings. In our real-data analysis (Section~\ref{sec:real-world-data-smart}), for example, the target cell type (ILC1) exhibits a more specialized and functionally constrained regulatory program than the source cell type (NK), which circulates systemically and participates in a broader range of immune pathways. Consequently, the latent gene--protein modules active in ILC1 are naturally spanned by those present in NK, consistent with the containment condition~\citep{vivier2018innate}.

Assumption~\ref{assump:relation-target-source} above ensures that the dominant directions underlying the low-rank structure of $\RegM^*$ are fully captured by the spectral span of $\RegM^{(0)}$. However, when $r_0$ is large, this containment condition alone can still be weak, since $\mathrm{col}(\mathbf{U}^{(0)})$ and $\mathrm{col}(\mathbf{V}^{(0)})$ may include many directions that are irrelevant to the target. Under Assumption~\ref{assump:relation-target-source}, denote by 
\[
\mathbf{P}_u \coloneqq \mathbf{U}^{(0)\top} \mathbf{U}^* \in \mathbb{R}^{r_0 \times r},
\qquad
\mathbf{P}_v \coloneqq \mathbf{V}^{(0)\top} \mathbf{V}^* \in \mathbb{R}^{r_0 \times r}.
\]
Then we have that $\mathbf{U}^* = \mathbf{U}^{(0)} \mathbf{P}_u$ and $\mathbf{V}^* = \mathbf{V}^{(0)} \mathbf{P}_v$. We measure the complexity of these source-basis representations by
\begin{equation}
\label{eq:sparse-level-alignment}
\begin{aligned}
s^*_u &\coloneqq \left\vert \mathbf{P}_u \right\vert_0 
= \left\vert \mathbf{U}^{* \top} \mathbf{U}^{(0)} \right\vert_0 
= \sum_{i=1}^r \#\left\{1 \leq j \leq r_0 : \langle \mathbf{u}^*_i, \mathbf{u}^{(0)}_j \rangle \neq 0 \right\}, \\
s^*_v &\coloneqq \left\vert \mathbf{P}_v \right\vert_0 
= \left\vert \mathbf{V}^{* \top} \mathbf{V}^{(0)} \right\vert_0 
= \sum_{i=1}^r \#\left\{1 \leq j \leq r_0 : \langle \mathbf{v}^*_i, \mathbf{v}^{(0)}_j \rangle \neq 0 \right\}.
\end{aligned}
\end{equation}
We further impose the following condition.

\begin{assump}[Sparse alignment]
\label{assump:spectral-sparstiy}
Assume that $s^*_u$ and $s^*_v$ are small relative to $rp$ and $rq$, respectively.
\end{assump}

Together, Assumptions~\ref{assump:relation-target-source} and~\ref{assump:spectral-sparstiy} above define our structural similarity framework. The spectral containment condition ensures that the target subspaces lie within the source subspaces, while the sparse alignment condition requires this representation to be parsimonious. By construction, we have that $s^*_u, s^*_v \leq rr_0$; hence, when $r_0$ is much smaller than $\min(p,q)$, this relative sparsity holds automatically. When $r_0$ is large, the condition requires that each target singular vector involves only a small number of source directions on average.

Indeed, Assumption \ref{assump:spectral-sparstiy} is well-motivated in many structured biological systems. Even when two related conditions share broad regulatory programs, the target condition often activates only a subset of those pathways. In such cases, each latent component of the target can be represented using only a small number of directions from the source spectral basis, reflecting the selective reuse of functional modules. The sparse alignment condition formalizes such phenomenon, enabling efficient information transfer while preserving interpretability. A concrete illustration is given in Section~\ref{sec:real-world-data-smart}, where the target cell type reuses only a few regulatory axes from its broader source counterpart.

In practice, the clean source matrix $\RegM^{(0)}$ is typically unobserved, and we only have access to a noisy approximation
\begin{equation*}
\widetilde{\RegM}^{(0)} = \RegM^{(0)} + \mathbf{E}^{(0)},
\end{equation*}
where $\mathbf{E}^{(0)} \in \mathbb{R}^{p \times q}$ is an additive noise matrix. We therefore work with the singular value decomposition of the observed noisy matrix
\begin{equation*}
\widetilde{\RegM}^{(0)} 
= \widetilde{\mathbf{U}}^{(0)} \widetilde{\mathbf{D}}^{(0)} \widetilde{\mathbf{V}}^{(0)\top} 
= \sum_{j=1}^{\widetilde{r}_0} \widetilde{d}_j^{(0)} \widetilde{\mathbf{u}}_j^{(0)} \widetilde{\mathbf{v}}_j^{(0)\top},
\end{equation*}
where $\widetilde{r}_0 = \mathrm{rank}(\widetilde{\RegM}^{(0)})$ and $\widetilde{d}_1^{(0)} > \cdots > \widetilde{d}_{\widetilde{r}_0}^{(0)} > 0$. We do \textit{not} assume that $\widetilde{\RegM}^{(0)}$ is low-rank. The strict inequalities ensure that the nonzero singular vectors are uniquely ordered up to the sign, a mild generic condition under perturbation by noise or numerical error.

\section{SMART algorithm}
\label{sec:methodology-smart}

In this section, we introduce SMART (\textbf{S}pectral \textbf{M}ulti-task \textbf{A}daptive \textbf{R}egression \textbf{T}ransfer), our source-guided estimator for multi-task reduced-rank regression. The method combines low-rank regression with structured regularization induced by spectral information from a related source study.

To this end, we extend the singular vectors $\widetilde{\mathbf{U}}^{(0)}$ and $\widetilde{\mathbf{V}}^{(0)}$ of the noisy source matrix $\widetilde{\RegM}^{(0)}$ to full orthonormal bases of $\mathbb{R}^p$ and $\mathbb{R}^q$ by appending arbitrary but fixed orthonormal complements $\widetilde{\mathbf{u}}^{(0)}_{\widetilde{r}_0+1}, \ldots, \widetilde{\mathbf{u}}^{(0)}_p$ and $\widetilde{\mathbf{v}}^{(0)}_{\widetilde{r}_0+1}, \ldots, \widetilde{\mathbf{v}}^{(0)}_q$. The estimator below is defined relative to this fixed completion. For any $(r_u, r_v)$ with $0 \leq r_u \leq p$ and $0 \leq r_v \leq q$, let us define
\begin{equation*}
\widetilde{\mathbf{U}}^{\mathrm{s}}(r_u) \coloneqq \left[ \widetilde{\mathbf{u}}^{(0)}_1, \ldots, \widetilde{\mathbf{u}}^{(0)}_{r_u} \right], \quad 
\widetilde{\mathbf{U}}^{\mathrm{s}}_\perp(r_u) \coloneqq \left[ \widetilde{\mathbf{u}}^{(0)}_{r_u+1}, \ldots, \widetilde{\mathbf{u}}^{(0)}_p \right],
\end{equation*}
and similarly for the right singular vectors
\begin{equation*}
\widetilde{\mathbf{V}}^{\mathrm{s}}(r_v) \coloneqq \left[ \widetilde{\mathbf{v}}^{(0)}_1, \ldots, \widetilde{\mathbf{v}}^{(0)}_{r_v} \right], \quad 
\widetilde{\mathbf{V}}^{\mathrm{s}}_\perp(r_v) \coloneqq \left[ \widetilde{\mathbf{v}}^{(0)}_{r_v+1}, \ldots, \widetilde{\mathbf{v}}^{(0)}_q \right].
\end{equation*}
We interpret $\widetilde{\mathbf{U}}^{\mathrm{s}}(0)$, $\widetilde{\mathbf{V}}^{\mathrm{s}}(0)$, $\widetilde{\mathbf{U}}^{\mathrm{s}}_\perp(p)$, and $\widetilde{\mathbf{V}}^{\mathrm{s}}_\perp(q)$ as empty matrices of compatible dimensions.

Given structural parameters $(r_u, r_v)$, we estimate the low-rank components $(\widehat{\mathbf{U}}, \widehat{\mathbf{D}}, \widehat{\mathbf{V}})$ by solving
\begin{equation}
\label{eq:estimator-obj}
\begin{aligned}
\widehat{\mathbf{U}}, \widehat{\mathbf{D}}, \widehat{\mathbf{V}} \in \arg \min_{\mathbf{U}, \mathbf{D}, \mathbf{V}} \Bigg\{ &
\frac{1}{2n} \left\Vert \mathbf{Y} - \mathbf{X} \mathbf{U} \mathbf{D} \mathbf{V}^{\top} \right\Vert^2_{\mathrm{F}}
+ \lambda_u \left| \widetilde{\mathbf{U}}^{\mathrm{s}\top}_{\perp}(r_u) \mathbf{U} \mathbf{D} \right|_1
+ \lambda_v \left| \widetilde{\mathbf{V}}^{\mathrm{s}\top}_{\perp}(r_v) \mathbf{V} \mathbf{D} \right|_1
\Bigg\} \\
& \text{subject to} \quad \mathbf{U}^{\top} \mathbf{U} = \mathbf{I}_r, \quad \mathbf{V}^{\top} \mathbf{V} = \mathbf{I}_r, \quad \mathbf{D} \geq 0,
\end{aligned}
\end{equation}
where $\mathbf{U} \in \mathbb{R}^{p \times r}$, $\mathbf{V} \in \mathbb{R}^{q \times r}$, and $\mathbf{D} = \mathrm{diag}(d_1, \ldots, d_r)$ with $d_j \geq 0$. The resulting estimator is given by 
\[
\widehat{\RegM} \coloneqq \widehat{\mathbf{U}} \widehat{\mathbf{D}} \widehat{\mathbf{V}}^{\top}.
\]

The optimization criterion in \eqref{eq:estimator-obj} above combines a least-squares loss with two source-guided penalties. The term $|\widetilde{\mathbf{U}}^{\mathrm{s}\top}_{\perp}(r_u)\mathbf{U}\mathbf{D}|_1$ penalizes the %coordinates 
entries of the left factor $\mathbf{U}\mathbf{D}$ along source directions outside the leading $r_u$-dimensional left subspace, and $|\widetilde{\mathbf{V}}^{\mathrm{s}\top}_{\perp}(r_v)\mathbf{V}\mathbf{D}|_1$ plays a similar role for the right factor $\mathbf{V}\mathbf{D}$. Thus, coordinates along the first $r_u$ left source directions and the first $r_v$ right source directions remain unpenalized, whereas coordinates in the complementary directions are shrunk toward zero. Such formulation favors solutions that use only a sparse set of penalized source directions, in line with Assumption~\ref{assump:spectral-sparstiy}.

As shown later in Section~\ref{sec:theory-analysis-smart}, the structural ranks $(r_u, r_v)$ control a trade-off between the approximation error induced by source noise and the statistical estimation error. Smaller values of $(r_u, r_v)$ impose stronger source-guided shrinkage and are most beneficial when the source singular vectors are accurate. As the source noise increases, larger values of $(r_u, r_v)$ enlarge the unpenalized subspaces and reduce sensitivity to noisy source directions.

\subsection{Optimization algorithm}
\label{sec:opt-smart}

In this section, we describe an ADMM-type algorithm for numerically solving \eqref{eq:estimator-obj}. The factorized objective is nonconvex because of the orthogonality constraints and the coupling among $\mathbf{U}$, $\mathbf{D}$, and $\mathbf{V}$. The algorithm alternates updates of $\mathbf{U}$, $\mathbf{D}$, and $\mathbf{V}$, enforcing orthogonality through projections onto the Stiefel manifold. In practice, we initialize SMART with the truncated SVD of a warm-start estimator, such as the ridge or Lasso, to improve the numerical stability.

To simplify notation, let us write $\widetilde{\mathbf{U}}^{\mathrm{s}}_{\perp}(r_u)$ as $\widetilde{\mathbf{U}}^{\mathrm{s}}_{\perp}$ and $\widetilde{\mathbf{V}}^{\mathrm{s}}_{\perp}(r_v)$ as $\widetilde{\mathbf{V}}^{\mathrm{s}}_{\perp}$ throughout this section. We can then rewrite \eqref{eq:estimator-obj} as the equivalent constrained problem
\begin{equation}
\label{eq:ADMM_problem}
\begin{aligned}
(\widehat{\bm{\Theta}}, \widehat{\bm{\Omega}}) \in \arg \min_{\bm{\Theta}, \bm{\Omega}} & \left\{
\frac{1}{2n} \left\Vert \mathbf{Y} - \mathbf{X} \mathbf{U} \mathbf{D} \mathbf{V}^{\top} \right\Vert^2_{\mathrm{F}}
+ \lambda_u \left| \bm{\Omega}_u \right|_1
+ \lambda_v \left| \bm{\Omega}_v \right|_1
\right\}, \\
\text{subject to} \quad & \mathbf{U}^{\top} \mathbf{U} = \mathbf{I}_r, \quad \mathbf{V}^{\top} \mathbf{V} = \mathbf{I}_r, \quad \mathbf{D} \geq 0, \\
& \widetilde{\mathbf{U}}^{\mathrm{s}\top}_{\perp} \mathbf{U} \mathbf{D} = \bm{\Omega}_u, \quad \widetilde{\mathbf{V}}^{\mathrm{s}\top}_{\perp} \mathbf{V} \mathbf{D} = \bm{\Omega}_v,
\end{aligned}
\end{equation}
where $\bm{\Theta} = (\mathbf{U}, \mathbf{D}, \mathbf{V})$, $\bm{\Omega} = (\bm{\Omega}_u, \bm{\Omega}_v)$, and $\mathbf{D}$ is diagonal with nonnegative entries.
The corresponding augmented Lagrangian is given by 
\begin{equation*}
\begin{aligned}
L_{\rho}(\bm{\Theta}, \bm{\Omega}, \bm{\Gamma}) &=
\frac{1}{2n} \left\Vert \mathbf{Y} - \mathbf{X} \mathbf{U} \mathbf{D} \mathbf{V}^{\top} \right\Vert^2_{\mathrm{F}}
+ \lambda_u \left| \bm{\Omega}_u \right|_1 + \lambda_v \left| \bm{\Omega}_v \right|_1 \\
&\quad + \left\langle \bm{\Gamma}_u, \widetilde{\mathbf{U}}^{\mathrm{s}\top}_{\perp} \mathbf{U} \mathbf{D} - \bm{\Omega}_u \right\rangle
+ \left\langle \bm{\Gamma}_v, \widetilde{\mathbf{V}}^{\mathrm{s}\top}_{\perp} \mathbf{V} \mathbf{D} - \bm{\Omega}_v \right\rangle \\
&\quad + \frac{\rho}{2} \left\Vert \widetilde{\mathbf{U}}^{\mathrm{s}\top}_{\perp} \mathbf{U} \mathbf{D} - \bm{\Omega}_u \right\Vert^2_{\mathrm{F}}
+ \frac{\rho}{2} \left\Vert \widetilde{\mathbf{V}}^{\mathrm{s}\top}_{\perp} \mathbf{V} \mathbf{D} - \bm{\Omega}_v \right\Vert^2_{\mathrm{F}},
\end{aligned}
\end{equation*}
where $\bm{\Gamma} = (\bm{\Gamma}_u, \bm{\Gamma}_v)$ are dual variables and $\rho > 0$ is a penalty parameter.

Starting from an initial point, the SMART algorithm alternates 
\begin{align*}
\mathbf{U}^{(t+1)} &\leftarrow \arg \min_{ \mathbf{U} : \mathbf{U}^{\top} \mathbf{U} = \mathbf{I}_r } L_{\rho} \left( \bigl( \mathbf{U}, \mathbf{D}^{(t)}, \mathbf{V}^{(t)} \bigr), \bm{\Omega}^{(t)}, \bm{\Gamma}^{(t)} \right), \\
\mathbf{V}^{(t+1)} &\leftarrow \arg \min_{ \mathbf{V} : \mathbf{V}^{\top} \mathbf{V} = \mathbf{I}_r } L_{\rho} \left( \bigl( \mathbf{U}^{(t+1)}, \mathbf{D}^{(t)}, \mathbf{V} \bigr), \bm{\Omega}^{(t)}, \bm{\Gamma}^{(t)} \right), \\
\mathbf{D}^{(t+1)} &\leftarrow \arg \min_{ \mathbf{D} : \mathbf{D} \geq 0 } L_{\rho} \left( \bigl( \mathbf{U}^{(t+1)}, \mathbf{D}, \mathbf{V}^{(t+1)} \bigr), \bm{\Omega}^{(t)}, \bm{\Gamma}^{(t)} \right),
\end{align*}
followed by
\[
\bm{\Omega}^{(t+1)} \leftarrow \arg \min_{\bm{\Omega}} L_{\rho}( \bm{\Theta}^{(t+1)}, \bm{\Omega}, \bm{\Gamma}^{(t)}),
\]
and the dual updates
\begin{align*}
\bm{\Gamma}^{(t+1)}_u &\leftarrow \bm{\Gamma}^{(t)}_u + \rho \left( \widetilde{\mathbf{U}}^{\mathrm{s}\top}_{\perp} \mathbf{U}^{(t+1)} \mathbf{D}^{(t+1)} - \bm{\Omega}^{(t+1)}_u \right), \\
\bm{\Gamma}^{(t+1)}_v &\leftarrow \bm{\Gamma}^{(t)}_v + \rho \left( \widetilde{\mathbf{V}}^{\mathrm{s}\top}_{\perp} \mathbf{V}^{(t+1)} \mathbf{D}^{(t+1)} - \bm{\Omega}^{(t+1)}_v \right).
\end{align*}

The $\mathbf{U}$- and $\mathbf{V}$-subproblems are addressed on the Stiefel manifold, whereas the updates for $\mathbf{D}$, $\bm{\Omega}_u$, and $\bm{\Omega}_v$ have straightforward closed-form expressions. Further algorithmic details can be found in Appendix~\ref{sec:opt-details}.

\subsection{Hyperparameter selection}
\label{sec:hyper-param-smart}

We select the SMART hyperparameters $(\lambda_u, \lambda_v, r, r_u, r_v)$ as follows.

\textbf{Regularization parameters.}
For fixed $(r, r_u, r_v)$, we choose $(\lambda_u, \lambda_v)$ by minimizing the BIC criterion
\[
\mathrm{BIC}(\lambda_u, \lambda_v)
= nq \log \left\{\frac{\|\mathbf{Y} - \mathbf{X} \widehat{\RegM}(\lambda_u, \lambda_v)\|_{\mathrm{F}}^2}{nq}\right\}
+ \left( |\widehat{\bm{\Omega}}_u|_0 + |\widehat{\bm{\Omega}}_v|_0 \right)\log(nq),
\]
where
\[
\widehat{\RegM} = \widehat{\mathbf{U}}\widehat{\mathbf{D}}\widehat{\mathbf{V}}^\top, \qquad
\widehat{\bm{\Omega}}_u = \widetilde{\mathbf{U}}^{\mathrm{s}\top}_{\perp}(r_u)\widehat{\mathbf{U}}\widehat{\mathbf{D}}, \qquad
\widehat{\bm{\Omega}}_v = \widetilde{\mathbf{V}}^{\mathrm{s}\top}_{\perp}(r_v)\widehat{\mathbf{V}}\widehat{\mathbf{D}}.
\]
Because $(r, r_u, r_v)$ are fixed in this setting, the unpenalized dimension $r(r_u + r_v + 1)$ remains the same for all $(\lambda_u, \lambda_v)$. Consequently, the criterion only needs to adjust for the changes in the penalized support sizes.

\textbf{Rank parameters.}
We determine the target rank $r$ using the criterion of~\citet{bunea2011optimal} applied to $(\mathbf{Y}, \mathbf{X})$. The source truncation levels $(r_u, r_v)$ are then selected via a grid search combined with $K$-fold cross-validation, choosing the pair that yields the lowest average prediction error.

\textbf{Practical workflow.}
In applications, we first obtain an estimate of $r$, then carry out cross-validation over a reasonably sized grid of $(r_u, r_v)$ values, and subsequently select $(\lambda_u, \lambda_v)$ using the BIC for each candidate pair. This procedure strikes a practical compromise between computational burden and statistical precision.

\section{Theoretical properties}
\label{sec:theory-analysis-smart}

In this section, we derive theoretical upper bounds on the error of the SMART estimator and, in the noiseless-source setting, a nearly matching minimax lower bound, differing only by logarithmic factors.

Let us first introduce some notation that will be used throughout the technical analysis. We extend the source singular vectors $\mathbf{U}^{(0)}$ and $\mathbf{V}^{(0)}$ to full orthonormal bases of $\mathbb{R}^p$ and $\mathbb{R}^q$ by appending arbitrary orthonormal complements
\[
\mathbf{U}^{(0)}_\perp \coloneqq [\mathbf{u}^{(0)}_{r_0+1}, \ldots, \mathbf{u}^{(0)}_p],
\qquad
\mathbf{V}^{(0)}_\perp \coloneqq [\mathbf{v}^{(0)}_{r_0+1}, \ldots, \mathbf{v}^{(0)}_q],
\]
so that
\[
\mathrm{col}(\mathbf{U}^{(0)}_\perp) = \mathrm{col}(\mathbf{U}^{(0)})^\perp,
\qquad
\mathrm{col}(\mathbf{V}^{(0)}_\perp) = \mathrm{col}(\mathbf{V}^{(0)})^\perp.
\]
Define the full bases
\[
\mathbf{U}^{(0)}_{\mathrm{full}} \coloneqq [\mathbf{U}^{(0)} \ \mathbf{U}^{(0)}_\perp],
\qquad
\mathbf{V}^{(0)}_{\mathrm{full}} \coloneqq [\mathbf{V}^{(0)} \ \mathbf{V}^{(0)}_\perp].
\]
For $0 \leq r_u \leq p$ and $0 \leq r_v \leq q$, denote by 
\begin{align*}
\mathbf{U}^{\mathrm{s}}(r_u) &\coloneqq [\mathbf{u}^{(0)}_1, \ldots, \mathbf{u}^{(0)}_{r_u}], \quad 
\mathbf{U}^{\mathrm{s}}_\perp(r_u) \coloneqq [\mathbf{u}^{(0)}_{r_u+1}, \ldots, \mathbf{u}^{(0)}_p], \\
\mathbf{V}^{\mathrm{s}}(r_v) &\coloneqq [\mathbf{v}^{(0)}_1, \ldots, \mathbf{v}^{(0)}_{r_v}], \quad 
\mathbf{V}^{\mathrm{s}}_\perp(r_v) \coloneqq [\mathbf{v}^{(0)}_{r_v+1}, \ldots, \mathbf{v}^{(0)}_q],
\end{align*}
where $\{\mathbf{u}^{(0)}_j\}_{j=1}^{r_0}$ and $\{\mathbf{v}^{(0)}_j\}_{j=1}^{r_0}$ are the fixed ordered source singular vectors introduced in Section~\ref{sec:problem-setup-smart}, and the remaining vectors come from the chosen orthonormal completions. We interpret $\mathbf{U}^{\mathrm{s}}(0)$, $\mathbf{V}^{\mathrm{s}}(0)$, $\mathbf{U}^{\mathrm{s}}_\perp(p)$, and $\mathbf{V}^{\mathrm{s}}_\perp(q)$ as empty matrices of compatible dimensions.

We define the minimal leading truncation levels $r_u^*$ and $r_v^*$ such that the singular subspaces of the target matrix $\RegM^*$ are contained in the truncated source subspaces
\begin{equation}
\label{eq:ru-rv-star}
r_u^* \coloneqq \min \{ r_u : \mathrm{col}(\mathbf{U}^*) \subseteq \mathrm{col}(\mathbf{U}^{\mathrm{s}}(r_u)) \}, \quad
r_v^* \coloneqq \min \{ r_v : \mathrm{col}(\mathbf{V}^*) \subseteq \mathrm{col}(\mathbf{V}^{\mathrm{s}}(r_v)) \}.
\end{equation}
In view of Assumption~\ref{assump:relation-target-source}, we have that $r_u^*, r_v^* \leq r_0$.

To quantify how well the target singular vectors align with the retained source directions, let us define for truncation levels $r_u$ and $r_v$,
\begin{equation}
\label{eq:similar-quant-def}
s_u(r_u) \coloneqq \left\vert \mathbf{U}^{* \top} \mathbf{U}^{\mathrm{s}}_\perp(r_u) \right\vert_0, \quad
s_v(r_v) \coloneqq \left\vert \mathbf{V}^{* \top} \mathbf{V}^{\mathrm{s}}_\perp(r_v) \right\vert_0,
\end{equation}
where $\vert \cdot \vert_0$ counts the number of nonzero entries. Consequently, $s_u(r_u)$ and $s_v(r_v)$ count the nonzero %coordinates 
components of the target singular vectors along the penalized complement directions of the fixed augmented source basis.

By construction, it holds that $s_u(0) = s^*_u$ and $s_v(0) = s^*_v$, as introduced in~\eqref{eq:sparse-level-alignment}. Moreover, $s_u(r_u)$ and $s_v(r_v)$ are nonincreasing in $r_u$ and $r_v$, and satisfy that 
\[
s_u(r_u) = 0 \quad \text{for } r_u \geq r_u^*, 
\qquad
s_v(r_v) = 0 \quad \text{for } r_v \geq r_v^*.
\]

The quantities $r_u^*$, $r_v^*$, $s_u(r_u)$, and $s_v(r_v)$ do not depend on the particular orthonormal completion of the source basis. This follows from Assumption~\ref{assump:relation-target-source}, which guarantees that $\mathbf{U}^{* \top} \mathbf{u}^{(0)}_j = 0$ and $\mathbf{V}^{* \top} \mathbf{v}^{(0)}_j = 0$ for all $j > r_0$.

\subsection{Technical conditions}

We now introduce the technical assumptions required for our theoretical analysis. Throughout %Section~\ref{sec:theory-analysis-smart}, 
this section, we treat the design matrix $\mathbf{X}$ as fixed and impose deterministic conditions that typically hold with high probability when the rows of $\mathbf{X}$ are sampled i.i.d.\ from sub-Gaussian distributions.

We begin with a restricted eigenvalue condition, which ensures stable recovery in high-dimensional regimes.

\begin{assump}[Restricted eigenvalue]
\label{assump:design-RSC}
There exist universal constants $\rho_1, \rho_2 > 0$ such that
\[
\frac{1}{2n} \left\Vert \mathbf{X} \bm{\beta} \right\Vert_2^2 \geq \frac{\rho_2}{2} \left\Vert \bm{\beta} \right\Vert_2^2 - \rho_1 \frac{\log p}{n} \left\Vert \bm{\beta} \right\Vert_1^2 \quad \text{for all } \bm{\beta} \in \mathbb{R}^p.
\]
\end{assump}

The above condition holds with high probability when the rows $\mathbf{x}_1, \ldots, \mathbf{x}_n$ are i.i.d.\ Gaussian with a non-degenerate covariance matrix; see, e.g., \citet[Theorem 7.16]{wainwright2019high}. We next impose the boundedness and non-degeneracy conditions on the design matrix, the true coefficient matrix, and the source coefficient matrix.

\begin{assump}[Constrained singular values]
\label{assump:constrained-norm-design}
The following conditions hold for some universal constants $\nu,d_{\min},d^{(0)}_{\min} > 0$:
\begin{enumerate}[label=(\roman*)]
    \item The design matrix $\mathbf{X} \in \mathbb{R}^{n \times p}$ satisfies that $\sup_{\mathbf{u} \in \mathbb{R}^p : \left\Vert \mathbf{u} \right\Vert_2 \leq 1} \left\Vert \mathbf{X} \mathbf{u} \right\Vert_2^2 \leq \nu n$.
    \item The smallest nonzero singular values of the target and source coefficient matrices satisfy that $d_r^* \geq d_{\min}$ and $d^{(0)}_{r_0} \geq d^{(0)}_{\min}$, respectively.
\end{enumerate}
\end{assump}

To ensure identifiability and perturbation stability of the singular subspaces, we impose a spectral gap condition on both the target matrix $\RegM^*$ and the source matrix $\RegM^{(0)}$.

\begin{assump}[Spectral gap]
\label{assump:eigen-gap}
There exists a universal constant \(0 < \tau < 1\) such that the squared singular values satisfy a relative spectral gap condition
\begin{align}
d^{*2}_{j-1} - d^{*2}_j &\geq \tau d^{*2}_{j-1}, \quad && 2 \leq j \leq r, \label{eq:eigen-gap-target} \\
\left(d^{(0)}_{j-1}\right)^2 - \left(d^{(0)}_j\right)^2 &\geq \tau \left(d^{(0)}_{j-1}\right)^2, \quad && 2 \leq j \leq r_0. \label{eq:eigen-gap-source}
\end{align}
\end{assump}

The condition above is standard in the low-rank estimation literature; see, for example,~\citet[Condition 5]{uematsu2019sofar}. Finally, we impose a distributional assumption on the noise matrix.

\begin{assump}[Gaussian noise]
\label{assump:error-distribution}
Let $\mathbf{E} = [\mathbf{e}_1, \ldots, \mathbf{e}_n]^\top \in \mathbb{R}^{n \times q}$ be the noise matrix. Assume that the rows $\mathbf{e}_1, \ldots, \mathbf{e}_n$ are i.i.d.\ Gaussian with mean zero and covariance $\bm{\Sigma}$, where $\|\bm{\Sigma}\|_2 \leq \alpha_{\max}$ for some $\alpha_{\max} > 0$.
\end{assump}

The Gaussian assumption above simplifies the concentration analysis. Our results can be extended to sub-Gaussian noise with appropriate modifications to the constants; see, e.g., \cite{ShenLv2026}.

\subsection{Upper bound analysis of SMART}

In this section, we establish theoretical guarantees for the SMART estimator. Since the objective in~\eqref{eq:estimator-obj} is nonconvex, direct statistical analysis is challenging. To address such issue, we exploit a two-stage approach based on the convexity-assisted nonconvex optimization (CANO) framework of~\citet{uematsu2019sofar}, which yields non-asymptotic error bounds for nonconvex estimators. The CANO strategy proceeds in two steps. First, an initial estimator, typically convex, provides a coarse but consistent initializer. Second, this initializer is refined by solving the original nonconvex SMART problem within a neighborhood of the initial solution, where the objective exhibits favorable local geometry. Such localization mitigates the effects of nonconvexity and facilitates the technical analysis.

Our theoretical results focus on a global optimizer of the localized SMART problem. The resulting bound separates the statistical estimation error from the approximation error induced by the source noise.

Finally, we emphasize that CANO serves primarily as a theoretical device; practical optimization algorithms for SMART are discussed in Section~\ref{sec:opt-smart}.

\textbf{Stage 1: Initialization.} We start with constructing an initial estimator $\widehat{\RegM}^{\mathrm{init}}$ of the true coefficient matrix $\RegM^*$. For the theoretical analysis, we treat $\widehat{\RegM}^{\mathrm{init}}$ as a black-box initializer satisfying the non-asymptotic error bounds
\begin{equation}
\label{eq:require-initial}
\left\Vert \widehat{\RegM}^{\mathrm{init}} - \RegM^* \right\Vert_{\mathrm{F}} \leq R_{2,n}, 
\quad
\left\Vert \widehat{\RegM}^{\mathrm{init}} - \RegM^* \right\Vert_{1,2} \leq R_{1,n}.
\end{equation}

When $\RegM^*$ is additionally entrywise sparse, one concrete choice is the Lasso estimator
\begin{equation*}
\widehat{\RegM}^{\mathrm{lasso}} \in \arg \min_{ \RegM \in \mathbb{R}^{p \times q} } 
\left\{ \frac{1}{2n} \left\Vert \mathbf{Y} - \mathbf{X} \RegM \right\Vert_{\mathrm{F}}^2 
+ \lambda_0 \left\vert \RegM \right\vert_1 \right\},
\end{equation*}
where $\lambda_0 \geq 0$. The  proposition below provides its non-asymptotic guarantees.

\begin{proposition}[Error bounds for Lasso estimator]
\label{proposition:lasso}
For each $1 \le j \le q$, let $s_j := \left\| \RegM^*_{:j} \right\|_0$ be the sparsity of the $j$th column of $\RegM^*$, and define
\[
s_0 := \sum_{j=1}^q s_j = \left| \RegM^* \right|_0, 
\qquad
s_{\max} := \max_{1 \le j \le q} s_j,
\qquad
s_{*,1,2} := \left( \sum_{j=1}^q s_j^2 \right)^{1/2}.
\]
Assume that Assumptions~\ref{assump:design-RSC}, \ref{assump:constrained-norm-design}(i), and~\ref{assump:error-distribution} hold, and further that $n$ is sufficiently large, i.e., 
\begin{equation}
\label{eq:n-requir-lasso}
n 
>
\frac{
s_{\max}
\left[
32 \nu \alpha_{\max}
\left\{
\log p
+
\left( \frac{c_0^2 - 2}{4} \right)\log p
+
\left( \frac{c_0^2}{4} \right)\log q
\right\}
+
\log p
\right]
}{
\left( 4\rho_2^2 / 9 \right) \wedge \left( \rho_2 / (64\rho_1) \right)
}
\end{equation}
for some universal constant $c_0 > \sqrt{2}$. Then with
\[
\lambda_0 \asymp \sqrt{\frac{\alpha_{\max}\log(pq)}{n}},
\]
the Lasso estimator satisfies that with probability at least $1 - 2(pq)^{1 - c_0^2/2}$,
\[
\left\Vert \widehat{\RegM}^{\mathrm{lasso}} - \RegM^* \right\Vert_{\mathrm{F}}
\lesssim
\sqrt{\frac{\alpha_{\max} s_0 \log(pq)}{n}},
\qquad
\left\Vert \widehat{\RegM}^{\mathrm{lasso}} - \RegM^* \right\Vert_{1,2}
\lesssim
s_{*,1,2}\sqrt{\frac{\alpha_{\max}\log(pq)}{n}}.
\]
\end{proposition}

The proof of the above proposition is given in Appendix~\ref{sec:proof-proposition:lasso}. Other possible initializers include the ordinary least squares, ridge regression, and reduced-rank regression, provided that they satisfy the error bounds in~\eqref{eq:require-initial}.

\textbf{Stage 2: Refined estimation.}  
In the second stage, for fixed truncation levels $(r_u,r_v)$, we solve the SMART optimization problem~\eqref{eq:estimator-obj} over $(\mathbf{U}, \mathbf{D}, \mathbf{V})$ subject to the additional constraints
\begin{equation}
\label{eq:opt-constraints}
\left\Vert \mathbf{U}\mathbf{D}\mathbf{V}^{\top} - \widehat{\RegM}^{\mathrm{init}} \right\Vert_{\mathrm{F}} \leq 2R_{2,n}, 
\qquad 
\left\Vert \mathbf{U}\mathbf{D}\mathbf{V}^{\top} - \widehat{\RegM}^{\mathrm{init}} \right\Vert_{1,2} \leq 2R_{1,n}.
\end{equation}
Let $\widehat{\RegM} = \widehat{\mathbf{U}} \widehat{\mathbf{D}} \widehat{\mathbf{V}}^{\top}$ be a global optimum of this localized problem, and define the estimation error $\widehat{\bm{\Delta}} = \widehat{\RegM} - \RegM^*$. We also introduce the spectral conditioning factor
\[
\eta_r \coloneqq 1 + \frac{1}{\tau} \left( \sum_{j=1}^r \left( \frac{d_1^*}{d_j^*} \right)^2 \right)^{1/2},
\]
where $\tau > 0$ is the spectral gap constant from Assumption~\ref{assump:eigen-gap}.

Finally, let us define the spectral approximation error of the source matrix as
\begin{equation*}
\varepsilon^{(0)} \coloneqq \left\Vert \widetilde{\RegM}^{(0)} - \RegM^{(0)} \right\Vert_2,
\end{equation*}
which measures the fidelity of the noisy source matrix $\widetilde{\RegM}^{(0)}$ to the true source matrix $\RegM^{(0)}$; smaller values correspond to a more accurate source.

We are now ready to state the error upper bound for the SMART estimator.

\begin{theorem}[Error upper bound for SMART]
\label{thm:spectral-trans}
Assume that Assumptions~\ref{assump:relation-target-source}--\ref{assump:error-distribution} and the initialization condition~\eqref{eq:require-initial} hold. Fix truncation levels $r_u$ and $r_v$, and an error tolerance $0 < \delta < 1$. Assume further that sample size $n$ is sufficiently large so that
\begin{equation}
\label{eq:cond-lasso-l2-err}
R_{2,n} \leq \frac{\tau}{18 \sqrt{5}} \cdot \frac{d^*_r}{\sqrt{r}} \cdot \left( \frac{d^*_r}{d^*_1} \right)^{3/2},
\end{equation}
\begin{equation}
\label{eq:cond-O1}
R_{1,n}^2
\lesssim
\alpha_{\max}\eta_r^2
\bigl\{r(r_u+r_v+1)+s_u(r_u)+s_v(r_v)\bigr\}
\cdot
\frac{\log(12r\max(p,q)/\delta)}{\log p},
\end{equation}
and
\begin{equation}
\label{eq:cond-O2}
\frac{9 c^2 \eta_r^2 R_{2,n}^2}{2 \alpha_{\max} d_{\min}^2} 
\Bigg[ 
\tr(\bm{\Sigma}) 
+ c_1 \max \Bigg\{ 
\sqrt{ \tfrac{ \tr(\bm{\Sigma}^2) \log (6p/\delta) }{ c_2 } }, 
\tfrac{ \alpha_{\max} \log (6p/\delta) }{ c_2 } 
\Bigg\} 
\Bigg] 
\leq \log \left( \tfrac{12pr}{\delta} \right)
\end{equation}
for universal constants $c, c_1, c_2 > 0$. Set the regularization parameters as
\begin{equation}
\label{eq:lambdas-in-good-events}
\lambda_u = 8\sqrt{2} \cdot \sqrt{ \frac{ \nu \alpha_{\max} \log (12pr/\delta) }{n} }, 
\qquad
\lambda_v = 2\sqrt{2} \cdot \sqrt{ \frac{ \nu \alpha_{\max} \log (6qr/\delta) }{n} }.
\end{equation}
Then the estimation error of SMART satisfies that with probability at least $1 - \delta$, 
\begin{multline}
\label{eq:err-bd-high-prob-smart}
\left\Vert \widehat{\bm{\Delta}} \right\Vert_{\mathrm{F}} 
\lesssim 
\sqrt{\alpha_{\max}} \, \eta_r \cdot 
\sqrt{ \frac{ \big( r(r_u + r_v + 1) + s_u(r_u) + s_v(r_v) \big) \cdot \log (12r \max(p,q)/\delta) }{ n } } \\
+ \alpha_{\max}^{1/4} (d_1^*)^{1/2} \, \varepsilon_{\mathrm{source}} \cdot \left( \frac{ \log (12r \max(p,q)/\delta) }{ n } \right)^{1/4},
\end{multline}
where the source error $\varepsilon_{\mathrm{source}}$ is given by
\begin{equation}
\label{eq:eps-source}
\begin{aligned}
\varepsilon_{\mathrm{source}} 
=  & \Bigg[
    r \big\{ (r_u^* - r_u)_+ + (r_v^* - r_v)_+ \big\}
    - \big( s_u(r_u) + s_v(r_v) \big)
\Bigg]^{1/2} 
\cdot 
\left( \varepsilon^{(0)} \right)^{1/2} 
\left( 2 d_1^{(0)} + \varepsilon^{(0)} \right)^{1/2} \\[0.5em]
& + \sqrt{r} \cdot \Bigg[
    (r_u^*)^{1/4} \sqrt{ p - \max(r_u^*, r_u) }
    + (r_v^*)^{1/4} \sqrt{ q - \max(r_v^*, r_v) }
\Bigg]^{1/2} \\
& \quad \times
\left( \varepsilon^{(0)} \right)^{1/4} 
\left( 2 d_1^{(0)} + \varepsilon^{(0)} \right)^{1/4},
\end{aligned}
\end{equation}
and $(t)_+ = \max\{t,0\}$ denotes the positive part.
\end{theorem}

The proof of Theorem \ref{thm:spectral-trans} is provided in Section~\ref{sec:proof-upp-bd-smart}.  
The upper bound in~\eqref{eq:err-bd-high-prob-smart} above has two components. The first term, which attains a parametric rate, reflects the estimation error in the ideal setting when the source matrix is perfectly clean (i.e., $\varepsilon^{(0)} = 0$). The factor $r(r_u + r_v + 1) + s_u(r_u) + s_v(r_v)$ corresponds to the effective degrees of freedom of the model. To see this, let us consider the decomposition
\begin{equation*}
\mathbf{U}^* = 
\mathbf{U}^{\mathrm{s}} (r_u) 
\underbrace{\left[ \left( \mathbf{U}^{\mathrm{s}} (r_u) \right)^{\top} \mathbf{U}^* \right]}_{\mathbf{P}^u_1} 
+ \mathbf{U}^{\mathrm{s}}_{\perp} (r_u) 
\underbrace{\left[ \left( \mathbf{U}^{\mathrm{s}}_{\perp} (r_u) \right)^{\top} \mathbf{U}^* \right]}_{\mathbf{P}^u_2}.
\end{equation*}
Here, $\mathbf{P}^u_1 \in \mathbb{R}^{r_u \times r}$ contributes $r r_u$ parameters. For $\mathbf{P}^u_2$, the $\ell_1$-penalty in~\eqref{eq:estimator-obj} encourages sparsity, so only $s_u(r_u)$ nonzero entries need to be estimated. Thus, recovering $\mathbf{U}^*$ requires $r r_u + s_u(r_u)$ parameters. Applying the same argument to $\mathbf{V}^*$ yields $r r_v + s_v(r_v)$ parameters, and estimating $\mathbf{D}^*$ contributes an additional $r$ degrees of freedom. Altogether, the model involves $r(r_u + r_v + 1) + s_u(r_u) + s_v(r_v)$ parameters.

The second term in~\eqref{eq:err-bd-high-prob-smart} accounts for the approximation error caused by noise in the source matrix $\widetilde{\RegM}^{(0)}$, scaled by $\alpha_{\max}^{1/4}$. If the source is noise-free (i.e., $\varepsilon^{(0)} = 0$), then the approximation error $\varepsilon_{\mathrm{source}}$ is zero and this term vanishes. Conversely, when $\varepsilon^{(0)}$ is large, $\varepsilon_{\mathrm{source}}$ increases, indicating stronger misalignment between the noisy source directions and the true target signal. Furthermore, when $r_u \geq r_u^*$ and $r_v \geq r_v^*$—the minimal truncation levels specified in~\eqref{eq:ru-rv-star}—the first component of $\varepsilon_{\mathrm{source}}$ in~\eqref{eq:eps-source} becomes zero. Because the second component grows more slowly with $\varepsilon^{(0)}$, enlarging $r_u$ and $r_v$ beyond these thresholds can significantly mitigate the approximation error.

The structural parameters $(r_u, r_v)$ govern the trade-off between estimation and approximation errors. Since $r r_u + s_u(r_u)$ and $r r_v + s_v(r_v)$ are non-decreasing in $r_u$ and $r_v$, larger values expand the unpenalized subspace, thereby increasing estimation error through higher model complexity while reducing approximation error by allowing the estimator to avoid noisy source directions. When the source is clean and its singular vectors are reliable, the optimal choice is often $r_u = r_v = 0$, allowing the estimator to exploit the source spectral structure most aggressively. With moderate source noise, intermediate values of $(r_u, r_v)$ balance these two effects by retaining informative directions while discarding noisy ones. In the extreme case of a heavily contaminated source, setting $r_u = p$ and $r_v = q$ removes the regularization terms entirely, reducing the procedure to a target-only model and avoiding negative transfer from unreliable source information.

\subsection{Minimax lower bound}

In this section, we derive a minimax lower bound on the expected squared Frobenius risk for multi-task learning with spectral transfer, focusing on an idealized setting where the source matrix $\RegM^{(0)}$ is observed noiselessly, i.e., $\varepsilon^{(0)} = 0$.

We begin by specifying the parameter space of interest. Guided by the spectral containment condition (Assumption~\ref{assump:relation-target-source}) and the projection sparsity notion in~\eqref{eq:sparse-level-alignment}, we define the admissible sets for $\mathbf{U}^*$ and $\mathbf{V}^*$ as
\begin{align*}
\mathcal{U} &\coloneqq \Bigl\{ \mathbf{U} \in \mathbb{R}^{p \times r} : 
\mathbf{U}^{\top} \mathbf{U} = \mathbf{I}_r,\ 
\mathrm{col}(\mathbf{U}) \subseteq \mathrm{col}(\mathbf{U}^{(0)}),\ 
\bigl| \mathbf{U}^{\top} \mathbf{U}^{(0)} \bigr|_0 \leq s^*_u \Bigr\}, \\
\mathcal{V} &\coloneqq \Bigl\{ \mathbf{V} \in \mathbb{R}^{q \times r} : 
\mathbf{V}^{\top} \mathbf{V} = \mathbf{I}_r,\ 
\mathrm{col}(\mathbf{V}) \subseteq \mathrm{col}(\mathbf{V}^{(0)}),\ 
\bigl| \mathbf{V}^{\top} \mathbf{V}^{(0)} \bigr|_0 \leq s^*_v \Bigr\},
\end{align*}
respectively. Further, assume that the singular values of $\RegM^*$ are uniformly bounded away from zero and satisfy the spectral gap condition (Assumption~\ref{assump:eigen-gap}). The parameter space for $\mathbf{D}^*$ is defined as
\begin{equation}
\begin{aligned}
\label{eq:param-space-D}
\mathcal{D} \coloneqq \Bigl\{ \mathbf{D} = \mathrm{diag}(d_1, \ldots, d_r) : &  0 < \underline{c} \leq d_r < \cdots < d_1 \leq \gamma \underline{c} < \infty, \\ 
& d_{j-1}^2 - d_j^2 \geq \tau d_{j-1}^2 \quad \text{for all } 2 \leq j \leq r \Bigr\},
\end{aligned}
\end{equation}
where $\gamma > 1$ is a uniform upper bound on the condition number and $0 < \tau < 1$ is a universal constant.

Finally, we define the parameter space for the target coefficient matrix as
\begin{equation}
\label{eq:param-space-C}
\mathcal{C} \coloneqq \Bigl\{ \RegM \in \mathbb{R}^{p \times q} : 
\RegM = \mathbf{U} \mathbf{D} \mathbf{V}^{\top}, \ 
\mathbf{U} \in \mathcal{U},\ 
\mathbf{D} \in \mathcal{D},\ 
\mathbf{V} \in \mathcal{V} \Bigr\}.
\end{equation}
In the subsequent technical analysis, we treat the design matrix $\mathbf{X}$ as fixed and assume that it satisfies Assumption~\ref{assump:constrained-norm-design}. For simplicity, we also assume that the entries of the noise matrix $\mathbf{E}$ are i.i.d.\ Gaussian with mean zero and variance $\sigma^2$.

We are now ready to present the second main result, which establishes a minimax lower bound for estimating $\RegM^*$ under the squared Frobenius loss.

\begin{theorem}[Minimax lower bound]
\label{thm:lower-bound}
Consider the multi-task linear regression model in~\eqref{eq:model-assump}, with the parameter space $\mathcal{C}$ defined in~\eqref{eq:param-space-C}.
Assume that $\min(s^*_u, s^*_v) \geq r(r+1)$ and $r_0 \geq 4r+1$.
In addition, assume that
\begin{equation}
\label{eq:slog-lwd}
\begin{aligned}
\Bigl((s^*_u \wedge \lfloor r_0 / 2 \rfloor) - r\Bigr)
\log\left( \frac{r_0 - r}{(s^*_u \wedge \lfloor r_0/2 \rfloor) - r} \right)
&\geq \frac{4 \log 2}{c}, \\
\Bigl((s^*_v \wedge \lfloor r_0 / 2 \rfloor) - r\Bigr)
\log\left( \frac{r_0 - r}{(s^*_v \wedge \lfloor r_0/2 \rfloor) - r} \right)
&\geq \frac{4 \log 2}{c},
\end{aligned}
\end{equation}
where $c > 0$ is a universal constant.
If sample size $n$ is sufficiently large such that
\begin{multline}
\label{eq:sample-lower-bd-requirement}
n \geq \left( \frac{\sigma^2 (s^*_u \vee s^*_v)}{\gamma^2 \underline{c}^2 \nu} \right) \\
\vee \left( \frac{c \sigma^2}{8 \underline{c}^2 \nu}
\Bigl((s^*_u \wedge \lfloor r_0 / 2 \rfloor) - r\Bigr)
\log\left( \frac{r_0 - r}{(s^*_u \wedge \lfloor r_0/2 \rfloor) - r} \right) \right) \\
\vee \left( \frac{c \sigma^2}{8 \underline{c}^2 \nu}
\Bigl((s^*_v \wedge \lfloor r_0 / 2 \rfloor) - r\Bigr)
\log\left( \frac{r_0 - r}{(s^*_v \wedge \lfloor r_0/2 \rfloor) - r} \right) \right),
\end{multline}
the minimax expected squared Frobenius risk satisfies that 
\begin{multline}
\label{eq:lower-bound-equation}
\inf_{\widehat{\mathbf{C}}} \sup_{\mathbf{C} \in \mathcal{C}}
\mathbb{E}_{\mathbf{C}}
\left\Vert \widehat{\mathbf{C}} - \mathbf{C} \right\Vert_{\mathrm{F}}^2 \\
\gtrsim \frac{\sigma^2 (s^*_u + s^*_v)}{\gamma^2 n}
+ \sigma^2 \left(
\frac{s^*_u \wedge \lfloor r_0 / 2 \rfloor}{n}
\log\left( \frac{0.75 r_0}{s^*_u \wedge \lfloor r_0/2 \rfloor} \right)
+ \frac{s^*_v \wedge \lfloor r_0 / 2 \rfloor}{n}
\log\left( \frac{0.75 r_0}{s^*_v \wedge \lfloor r_0/2 \rfloor} \right)
\right),
\end{multline}
where $\mathbb{E}_{\mathbf{C}}$ denotes expectation under model~\eqref{eq:model-assump} with coefficient matrix $\mathbf{C}$.
\end{theorem}

The proof of Theorem~\ref{thm:lower-bound} is included in Appendix~\ref{sec:proof-lwd-bd}. Theorem~\ref{thm:lower-bound} above is established using a combination of local metric entropy arguments, adapted from~\cite{cai2013sparse} (see Lemma~\ref{lemma:local-metric-entropy} in the Appendix), and the classical Fano method. A key ingredient is a \textit{new} result establishing both upper and lower bounds for the $\delta$-covering number of the 
%$p \times r$ 
\textit{Stiefel manifold} under the Frobenius norm, extending earlier results for the Grassmann manifold~\citep{szarek1982nets} (see Lemma~\ref{lemma:covering-number-Stiefel-Frob}). We conclude with the following remark on Theorem~\ref{thm:lower-bound}.

\begin{remark}
Comparing the lower bound in~\eqref{eq:lower-bound-equation} to the high-probability upper bound in~\eqref{eq:err-bd-high-prob-smart}, we see that the two results match up to logarithmic factors in the noiseless-source regime when $\gamma$ is treated as a universal constant. To compare the two bounds directly, let us specialize the upper bound to the isotropic noise setting of $\bm{\Sigma} = \sigma^2 \mathbf{I}_q$ assumed in this subsection. Indeed, if $\gamma = O(1)$, the spectral gap condition in~\eqref{eq:param-space-D} entails that $r = O(1)$, since
\begin{equation*}
\underline{c}^2 \tau r 
 \leq  \tau \sum_{j=1}^r d_j^2
 \leq  \sum_{j=2}^r (d_{j-1}^2 - d_j^2) + d_r^2
= d_1^2 \leq \underline{c}^2 \gamma^2,
\end{equation*}
which implies that $r \leq \gamma^2 / \tau = O(1)$. Consequently, with $r_u = r_v = 0$ (so $s^*_u = s_u(0)$ and $s^*_v = s_v(0)$), the upper bound~\eqref{eq:err-bd-high-prob-smart} entails that the SMART estimator $\widehat{\RegM}$ satisfies that with probability at least $1-\delta$, 
\[
\left\Vert \widehat{\RegM} - \RegM^* \right\Vert_{\mathrm{F}}^{2}
 \lesssim 
\frac{\sigma^2 (s^*_u + s^*_v)\,
\log \bigl( c \max(p,q) / \delta \bigr)}{n}
\]
for some universal constant $c$, whereas the lower bound~\eqref{eq:lower-bound-equation} yields that 
\[
\inf_{\widehat{\mathbf{C}}} \sup_{\mathbf{C} \in \mathcal{C}}
\mathbb{E}_{\mathbf{C}}
\left\Vert \widehat{\mathbf{C}} - \mathbf{C} \right\Vert_{\mathrm{F}}^{2}
 \gtrsim  \frac{\sigma^2 (s^*_u + s^*_v)}{n}.
\]
Hence, in the noiseless-source regime, the upper and lower bounds agree up to logarithmic factors in their dependence on $n$, $s^*_u$, and $s^*_v$.
\end{remark}

\section{Simulation examples}
\label{sec:simulation-smart}

In this section, we evaluate the empirical performance of the suggested SMART algorithm through a series of simulation examples. We compare its estimation accuracy to several baseline methods under varying conditions, including different sample sizes, fitted ranks, source truncation levels, and source quality.

Recall that $n$ denotes the number of observations, $p$ the number of input features, and $q$ the number of output responses. Let $r$ be the rank of the true target coefficient matrix $\RegM^*$, and $r_0$ the rank of the source coefficient matrix $\RegM^{(0)}$, with $r \le r_0$. The rank of an estimator $\widehat{\RegM}$ is denoted as $\widehat{r}$. To measure the estimation accuracy, we report the normalized Frobenius error of the estimated coefficient matrix $\widehat{\RegM} = \widehat{\mathbf{U}} \widehat{\mathbf{D}} \widehat{\mathbf{V}}^{\top}$, defined as $\frac{1}{\sqrt{pq}} \left\Vert \widehat{\RegM} - \RegM^* \right\Vert_{\mathrm{F}}$.

\subsection{Data generating process}

We simulate data according to the following procedure:

\begin{enumerate}
    \item[1)] \textbf{Source regression coefficient matrix:} We generate two random matrices $\mathbf{U}' \in \mathbb{R}^{p \times r_0}$ and $\mathbf{V}' \in \mathbb{R}^{q \times r_0}$ with i.i.d.\ standard normal entries, and  orthonormalize their columns via the QR decomposition to obtain $\mathbf{U}^{(0)}$ and $\mathbf{V}^{(0)}$. We then construct a diagonal matrix $\mathbf{D}^{(0)} \in \mathbb{R}^{r_0 \times r_0}$ with singular values linearly spaced between 1.0 and 10.0 in decreasing order. The clean source regression coefficient matrix is defined as
    \[
    \RegM^{(0)} = \mathbf{U}^{(0)} \mathbf{D}^{(0)} \mathbf{V}^{(0)\top}.
    \]
    To model imperfect source knowledge, we add Gaussian noise
    \[
    \widetilde{\RegM}^{(0)} = \RegM^{(0)} + \mathbf{E}^{(0)}, 
    \quad E^{(0)}_{ij} \sim \mathcal{N}(0, \sigma_0^2) \ \text{i.i.d.}
    \]
    The noise level $\sigma_0$ controls the fidelity of the source: smaller values correspond to a higher-quality source, and $\sigma_0 = 0$ corresponds to the noiseless case.

    \item[2)] \textbf{Feature matrix:}
    The input feature matrix $\mathbf{X} \in \mathbb{R}^{n \times p}$ has i.i.d.\ rows drawn from a multivariate normal distribution $\mathcal{N}_p(0, \bm{\Sigma}_x)$, where the covariance matrix $\bm{\Sigma}_x \in \mathbb{R}^{p \times p}$ has entries $(\bm{\Sigma}_x)_{ij} = 0.5^{|i-j|}$.

    \item[3)] \textbf{Target regression coefficient matrix:}
    To construct the target coefficient matrix $\RegM^*$, we randomly select $r$ columns from both $\mathbf{U}^{(0)}$ and $\mathbf{V}^{(0)}$ without replacement, forming $\mathbf{U}^* \in \mathbb{R}^{p \times r}$ and $\mathbf{V}^* \in \mathbb{R}^{q \times r}$. We then generate a diagonal matrix $\mathbf{D}^* \in \mathbb{R}^{r \times r}$ with values linearly spaced between 3.0 and 5.0 in decreasing order. The target regression coefficient matrix is
    $\RegM^* = \mathbf{U}^{*} \mathbf{D}^{*} \mathbf{V}^{*\top}$.
    This construction ensures that in the noiseless-source setting, the spectral containment condition holds exactly and that the target singular vectors are sparse in the source singular bases.

    \item[4)] \textbf{Response matrix:}
    Finally, the response matrix $\mathbf{Y} \in \mathbb{R}^{n \times q}$ is generated according to the multivariate linear model
    \[
    \mathbf{Y} = \mathbf{X} \RegM^* + \mathbf{E} \quad \text{with } E_{ij} \sim \mathcal{N}(0, \sigma^2) \ \text{i.i.d.}
    \]
    The noise standard deviation is fixed at $\sigma = 0.5$ for all experiments.
\end{enumerate}

\subsection{Experimental design}

\begin{figure}[t]
\centering
\includegraphics[width=.75\textwidth]{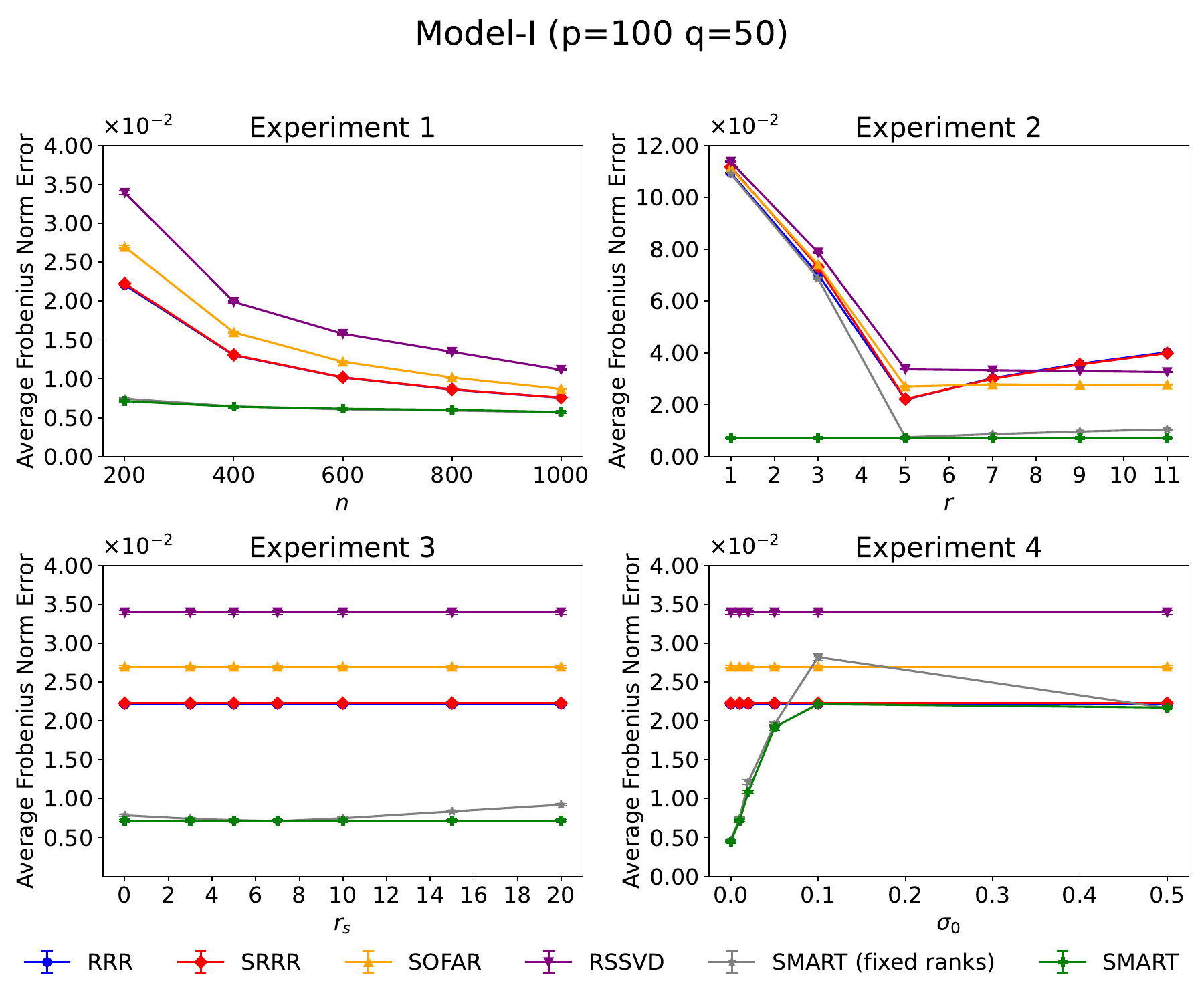}
\caption{Simulation results for Model I.}
\label{fig:simu-exp-model1-smart}
\end{figure}

We evaluate the performance of SMART across multiple simulation settings with varying problem dimensions. Specifically, we consider the following three configurations:

\begin{itemize}
    \item \textbf{Model I:} $p = 100$, $q = 50$;
    \item \textbf{Model II:} $p = 150$, $q = 100$;
    \item \textbf{Model III:} $p = 300$, $q = 200$.
\end{itemize}

To assess SMART under different conditions, we conduct four experiments. In each experiment, we vary one key parameter while holding the others fixed at their default values. The default settings are: sample size \( n = 200 \), estimated rank \( \widehat{r} = 5 \), source subspace ranks \( r_u = r_v = 10 \), and source noise level \( \sigma_0 = 0.01 \). For all models, we fix \( r_0 = 10 \) and \( r = 5 \). Each configuration is repeated 100 times to compute the averages and standard errors.

\begin{enumerate}
    \item[1)] \textbf{Experiment 1:} Vary the sample size \( n \), while keeping \( \widehat{r} \), \( r_u \), \( r_v \), and \( \sigma_0 \) fixed.
    \item[2)] \textbf{Experiment 2:} Vary the estimated rank \( \widehat{r} \), with \( n \), \( r_u \), \( r_v \), and \( \sigma_0 \) held constant.
    \item[3)] \textbf{Experiment 3:} Let \( r_u = r_v = r_s \) and vary \( r_s \), keeping \( n \), \( \widehat{r} \), and \( \sigma_0 \) fixed.
    \item[4)] \textbf{Experiment 4:} Vary the source noise level \( \sigma_0 \), with \( \widehat{r} \), \( r_u \), and \( r_v \) fixed.
\end{enumerate}

In these experiments, the hyperparameters $(\widehat{r}, r_u, r_v)$ for SMART are either fixed or selected automatically. For the fixed-rank variant, we set $(\widehat{r}, r_u, r_v)$ manually and tune $(\lambda_u, \lambda_v)$ using the BIC criterion in Section~\ref{sec:hyper-param-smart}; we denote this method as \textbf{SMART (fixed ranks)}. The fully automated version of SMART jointly selects $(\widehat{r}, r_u, r_v, \lambda_u, \lambda_v)$ using the strategy in Section~\ref{sec:hyper-param-smart}, and we denote it simply as \textbf{SMART}. In Experiments 2 and 3, the horizontal axis varies only the externally imposed fixed-rank specification. We therefore depict the performance of fully automatic SMART as a horizontal reference line evaluated under the corresponding default configuration for each model.

\subsection{Baseline methods}

\begin{figure}[t]
\centering
\includegraphics[width=.75\textwidth]{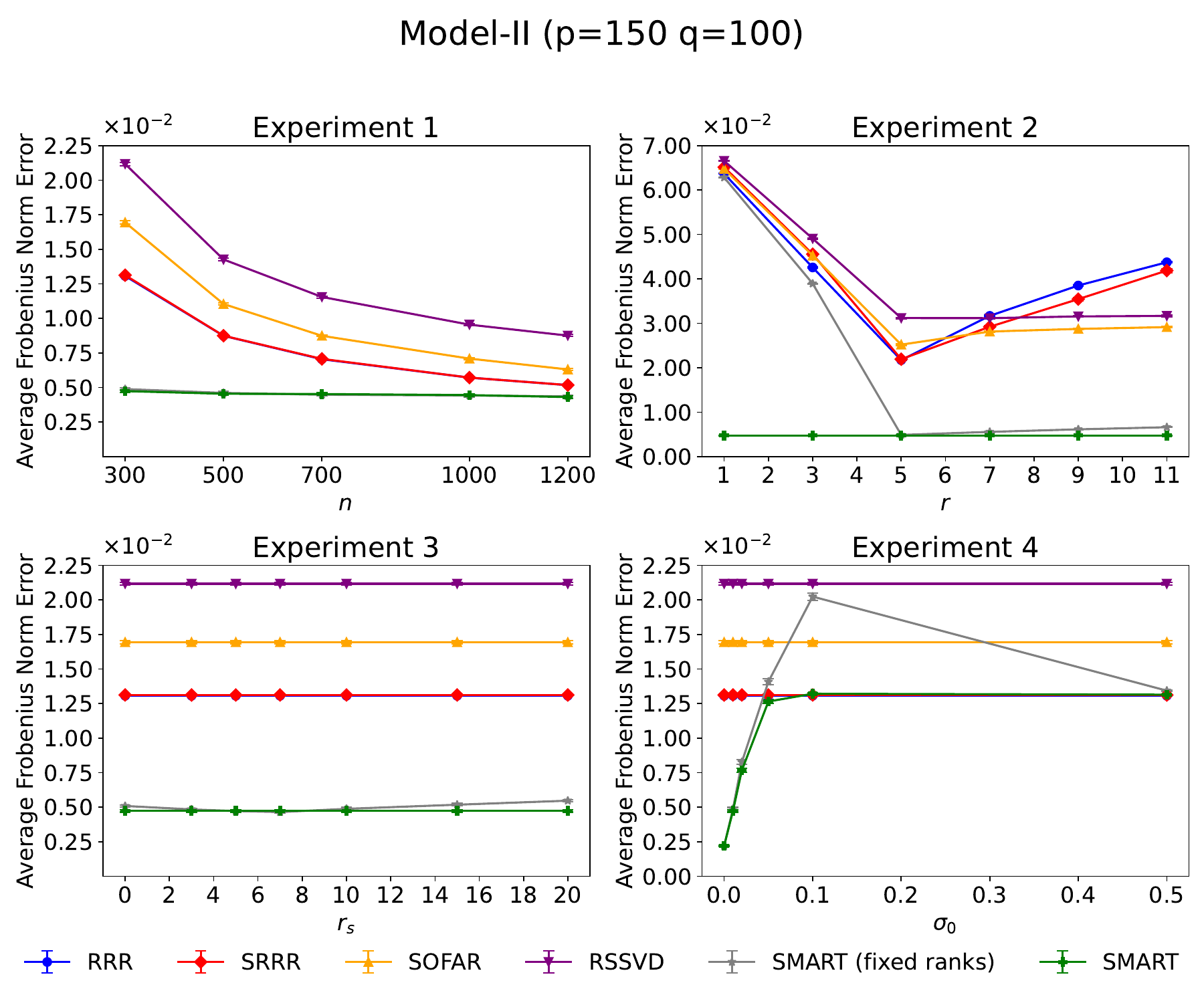}
\caption{Simulation results for Model II.}
\label{fig:simu-exp-model2-smart}
\end{figure}

We compare the suggested SMART algorithm to the following state-of-the-art low-rank regression methods:

\begin{itemize}
    \item \textbf{RRR.} An adaptive low-rank regression method that extends traditional reduced-rank regression by employing adaptive nuclear-norm penalization~\citep{chen2013reduced}.
    
    \item \textbf{SRRR.} Incorporates group Lasso penalties into reduced-rank regression, enabling simultaneous dimension reduction and variable selection~\citep{chen2012sparse}.
    
    \item \textbf{SOFAR.} Enforces two-way sparsity on singular vectors by solving a regularized optimization problem with orthogonality constraints~\citep{uematsu2019sofar}.
    
    \item \textbf{RSSVD.} A layer-wise estimation method that applies adaptive penalties to individual rank-one components~\citep{mukherjee2011reduced,mukherjee2015degrees}.
\end{itemize}

All baseline methods use only the target data and ignore the source information. Comparing SMART against them highlights the benefit of transfer learning. Their performance is unaffected by the source truncation level $r_s$ and the source noise level $\sigma_0$. Accordingly, in Experiments 3 and 4 we present their results as horizontal reference lines evaluated under the corresponding default configuration.

\subsection{Simulation results}

Figures~\ref{fig:simu-exp-model1-smart}--\ref{fig:simu-exp-model3-smart} present the simulation results for Models I--III. Across all experiments, SMART consistently outperforms the baseline methods. In Experiments 2 and 3, the fixed-rank version shows strong robustness to the choice of hyperparameters \( (\widehat{r}, r_u, r_v) \). Experiment 4 further demonstrates that the fully data-driven SMART, which selects \( (\widehat{r}, r_u, r_v) \) via the strategy in Section~\ref{sec:hyper-param-smart}, never underperforms relative to target-only methods. When the source matrix is accurate (low noise), SMART yields substantial gains over the baseline methods; when the source is highly contaminated, SMART automatically downweights or discards unreliable source information and matches the target-only performance. Overall, SMART is robust to the source quality and effectively avoids negative transfer.

\begin{figure}[t]
\centering
\includegraphics[width=.75\textwidth]{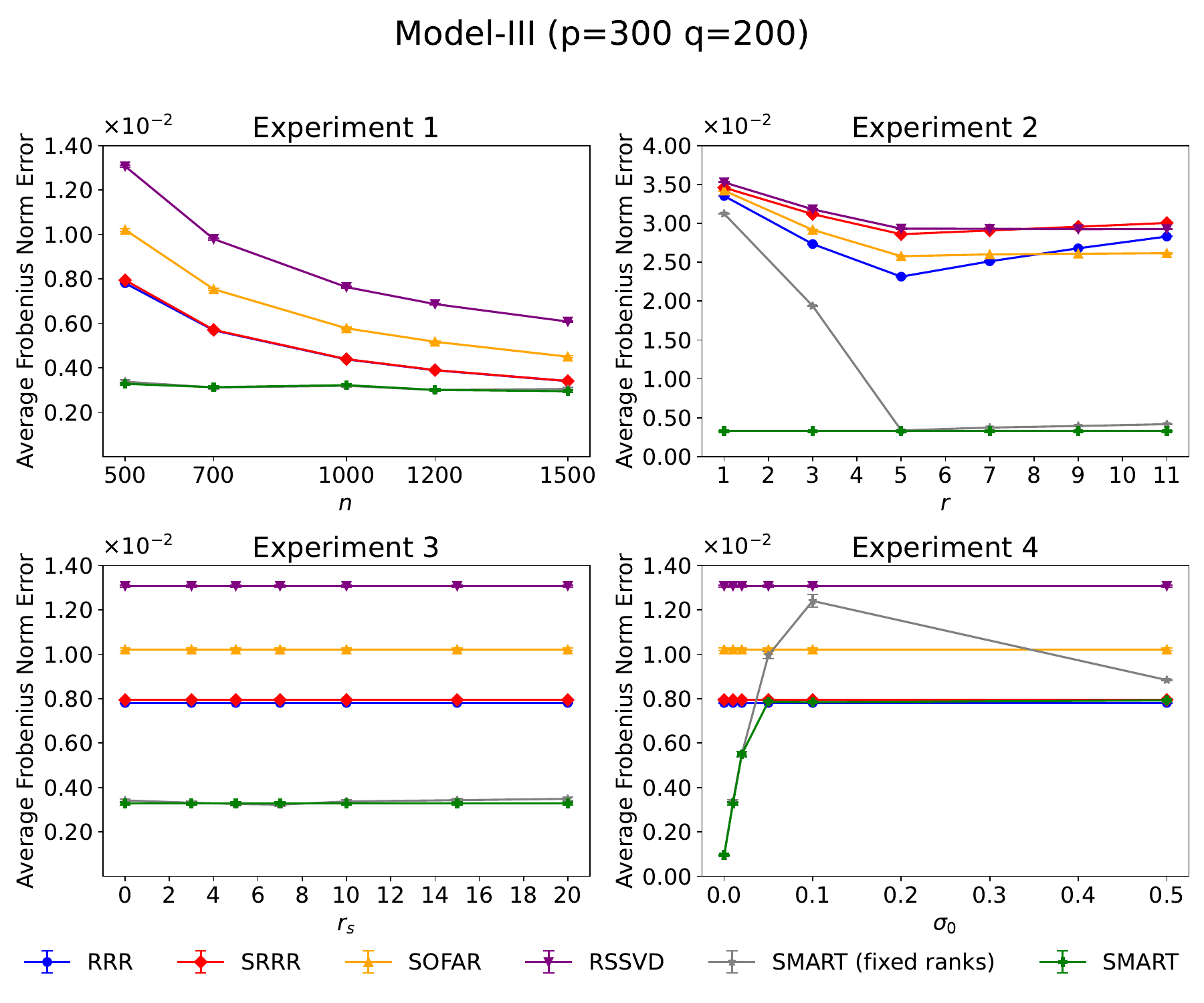}
\caption{Simulation results for Model III.}
\label{fig:simu-exp-model3-smart}
\end{figure}

\section{Application to multi-modal single-cell dataset}
\label{sec:real-world-data-smart}

\textbf{Data preprocessing.}
We constructed GEX and ADT matrices from the raw data, which contain $13{,}953$ genes and $134$ proteins across $90{,}261$ cells. After normalization, filtering, and feature selection, we restricted the analysis to NK and ILC1 cells, yielding $5{,}434$ source cells and $552$ target cells. The final predictor set contains $248$ shared gene features, and all $134$ ADT markers are retained as responses. Additional preprocessing details are provided in Appendix~\ref{sec:details-real-world-experiment}.

We evaluate \textsc{SMART} on a Multimodal Single-Cell dataset of bone marrow mononuclear cells from 12 healthy donors\footnote{The dataset was released for the Multimodal Single-Cell Data Integration Challenge at NeurIPS~2021 and is publicly available at \url{https://www.ncbi.nlm.nih.gov/geo/query/acc.cgi?acc=GSE194122}.}. 
Each cell contains paired measurements of gene expression (GEX) and protein abundance (ADT). We treat each row of $\mathbf{Y}$ as a protein-expression vector and the corresponding row of $\mathbf{X}$ as the associated gene-expression profile, and estimate a linear map from GEX to ADT.

The above dataset includes cell-type annotations, with substantial imbalance in cell counts across types. Some populations are well sampled, whereas others are rare, making target-only estimation challenging for the latter. Figure~\ref{fig:cell-type-number} in Appendix~\ref{sec:details-real-world-experiment} depicts the distribution of cell counts across cell types.

We focus on transfer from NK cells to ILC1 cells. Specifically, we treat ILC1 (tissue-resident innate lymphoid cells) as the \emph{target} and NK (circulating natural killer) as the \emph{source}. Such choice is motivated by both sample size and biological relatedness: ILC1 has $n=552$ cells, whereas NK has $n=5434$ cells, and both belong to the group~1 innate lymphoid cells that share key transcription factors (e.g., T-bet) and cytokines (e.g., IFN-$\gamma$). At the same time, NK cells exhibit broader systemic activity, whereas ILC1s are locally specialized and functionally constrained~\citep{vivier2018innate}. Consequently, the regulatory programs of ILC1s are plausibly contained within those of NK cells, which aligns naturally with the spectral containment assumption and motivates the use of NK as the source for transfer learning to ILC1.

\textbf{Implementation.}
For each source estimator (Lasso, Ridge, OLS, RRR, SRRR, RSSVD, SOFAR), we first fit a source matrix $\widetilde{\RegM}^{(0)}$ on the NK data and then use $\widetilde{\RegM}^{(0)}$ as the source input to \textsc{SMART} to obtain a target estimator $\widehat{\RegM}$. Table~\ref{tab:smart-summary} names each \textsc{SMART} variant by its source estimator. We assess performance over 100 random splits of the ILC1 cells (70\% training, 30\% testing) and, for each split, report the normalized test prediction error
\[
\frac{1}{\sqrt{n_{\mathrm{test}} q}} \left\| \mathbf{X}_{\mathrm{test}} \widehat{\RegM} - \mathbf{Y}_{\mathrm{test}} \right\|_{\mathrm{F}}.
\]

\textbf{Baseline methods.}
We compare \textsc{SMART} to \emph{source-only} and \emph{target-only} versions of the same seven estimators. Source-only methods fit a regression matrix on NK and apply it directly to ILC1 without target adaptation. Target-only methods are trained only on the ILC1 training sample. All methods are evaluated on the same random splits using the same prediction metric.

\textbf{Results.}
Table~\ref{tab:smart-summary} reports prediction errors for all methods. SMART consistently attains the lowest errors, with little variability across source initializations. Target-only methods perform worse, reflecting the difficulty of modeling with limited ILC1 data. Source-only methods are stable but suboptimal, underscoring the need for study-specific adaptation. Overall, SMART not only outperforms both baselines in accuracy but also yields the narrowest interquartile ranges, demonstrating strong robustness and generalization.

\begin{table}[t]
\centering
\begin{tabular}{lccccc}
\toprule
\textbf{Method} & \textbf{Mean} & \textbf{SE} & \textbf{Median} & \textbf{Q1} & \textbf{Q3} \\
\midrule
\multicolumn{6}{l}{\textit{Target-Only Methods}} \\
\midrule
Target-Only (RRR)   & 0.410 & $2.4 \times 10^{-2}$ & 0.307 & 0.259 & 0.478 \\
Target-Only (SRRR)  & 0.449 & $3.4 \times 10^{-2}$ & 0.292 & 0.227 & 0.533 \\
Target-Only (RSSVD) & 0.464 & $3.4 \times 10^{-2}$ & 0.309 & 0.239 & 0.489 \\
Target-Only (SOFAR) & 0.360 & $1.8 \times 10^{-2}$ & 0.297 & 0.267 & 0.355 \\
Target-Only (Lasso) & 0.331 & $3.0 \times 10^{-4}$ & 0.331 & 0.330 & 0.333 \\
Target-Only (OLS)   & 0.432 & $2.0 \times 10^{-3}$ & 0.428 & 0.416 & 0.442 \\
Target-Only (Ridge) & 0.336 & $4.0 \times 10^{-4}$ & 0.336 & 0.332 & 0.338 \\
\midrule
\multicolumn{6}{l}{\textit{Source-Only Methods}} \\
\midrule
Source-Only (RRR)   & 0.331 & $4.0 \times 10^{-4}$ & 0.331 & 0.328 & 0.334 \\
Source-Only (SRRR)  & 0.332 & $4.0 \times 10^{-4}$ & 0.332 & 0.329 & 0.334 \\
Source-Only (RSSVD) & 0.358 & $6.0 \times 10^{-4}$ & 0.358 & 0.355 & 0.362 \\
Source-Only (SOFAR) & 0.382 & $4.0 \times 10^{-4}$ & 0.383 & 0.380 & 0.385 \\
Source-Only (Lasso) & 0.365 & $3.0 \times 10^{-4}$ & 0.365 & 0.363 & 0.367 \\
Source-Only (OLS)   & 0.358 & $2.0 \times 10^{-4}$ & 0.358 & 0.357 & 0.360 \\
Source-Only (Ridge) & 0.361 & $3.0 \times 10^{-4}$ & 0.361 & 0.360 & 0.363 \\
\midrule
\multicolumn{6}{l}{\textit{SMART Methods}} \\
\midrule
SMART (Source: RRR)   & 0.200 & $3.0 \times 10^{-4}$ & 0.200 & 0.198 & 0.202 \\
SMART (Source: SRRR)  & 0.200 & $3.0 \times 10^{-4}$ & 0.200 & 0.198 & 0.202 \\
SMART (Source: RSSVD) & 0.203 & $5.0 \times 10^{-4}$ & 0.202 & 0.200 & 0.204 \\
SMART (Source: SOFAR) & 0.211 & $2.1 \times 10^{-3}$ & 0.204 & 0.201 & 0.215 \\
SMART (Source: Ridge) & 0.203 & $4.0 \times 10^{-4}$ & 0.203 & 0.201 & 0.205 \\
SMART (Source: Lasso) & 0.201 & $5.0 \times 10^{-4}$ & 0.201 & 0.199 & 0.203 \\
SMART (Source: OLS)   & 0.200 & $4.0 \times 10^{-4}$ & 0.200 & 0.198 & 0.202 \\
\bottomrule
\end{tabular}
\caption{Summary statistics of the normalized test prediction error $(n_{\mathrm{test}} q)^{-1/2}\|\mathbf{X}_{\mathrm{test}}\widehat{\RegM} - \mathbf{Y}_{\mathrm{test}}\|_{\mathrm{F}}$ over 100 random 70\%/30\% splits of the ILC1 cells. For SMART, ``Source: METHOD'' indicates the source estimator used to construct the source regression matrix.}
\label{tab:smart-summary}
\end{table}

\newpage
\section{Discussions}
\label{sec:conclusion-smart}

We have introduced in this paper \textsc{SMART}, a spectral transfer framework for low-rank multi-task linear regression that leverages the spectral structure of a related source matrix to improve estimation when target sample sizes are limited. By incorporating source spectral information through structured regularization, \textsc{SMART} enables effective transfer using only a fitted source regression matrix. We have suggested an ADMM-based algorithm for practical optimization of the resulting nonconvex objective and established non-asymptotic error upper bounds together with a minimax lower bound in the noiseless-source regime. Under some mild regularity conditions, these results imply near-minimax Frobenius rates up to logarithmic factors in that regime. In the simulation examples and the multi-modal single-cell application, \textsc{SMART} outperforms the baseline methods and shows robustness to negative transfer. Future works include extensions to the generalized and nonlinear models, inference procedures, and theoretically grounded data-adaptive selection of structural and regularization parameters.

\clearpage

\newpage

\putbib[boxinz-papers]
\end{bibunit}

%\bibliography{boxinz-papers}

\newpage
\appendix

\pagebreak
\begin{bibunit}[plainnat]

\section{Comparison with~\texorpdfstring{\cite{park2025transfer}}{Park et al., 2025}}
\label{sec:example-spectral-vs-nuclear}

In this section, we present a simple example to highlight the difference between our spectral similarity assumptions and the nuclear-norm-based assumption of \citet{park2025transfer}. For clarity of exposition, we restrict attention here to the setting with a single source.

Denote by $\mathbf{C}^* \in \mathbb{R}^{p \times q}$ the target coefficient matrix, and $\mathbf{C}^{(0)} \in \mathbb{R}^{p \times q}$ a source coefficient matrix. We begin by defining the contrast
\[
\mathbf{\Delta}^{(0)} \coloneqq \mathbf{C}^{(0)} - \mathbf{C}^*.
\]
In \citet{park2025transfer}, the key similarity requirement for an informative source is a \emph{bounded nuclear-norm contrast condition} of the form
\begin{equation}
\label{eq:park-nuclear}
\|\mathbf{\Delta}^{(0)}\|_* \leq h_n,
\end{equation}
where $h_n$ is sufficiently small relative to the sample size and the problem dimensions. In effect, Condition~\eqref{eq:park-nuclear} above requires the total singular-value mass of the difference between $\mathbf{C}^{(0)}$ and $\mathbf{C}^*$ to remain uniformly bounded, thereby enforcing global similarity in both singular directions and effect magnitude across domains.

In contrast, our framework (see Assumptions~\ref{assump:relation-target-source}--\ref{assump:spectral-sparstiy} in Section~\ref{sec:problem-setup-smart}) postulates a structured spectral linkage between the target and source. Specifically, we assume that the dominant singular subspaces of the target coefficient matrix are nested within those of the source, and that the associated singular vectors align sparsely when represented in the source’s singular basis. Collectively, these requirements imply that the target signal is concentrated within a low-dimensional set of modules already extracted from the source, while still permitting the singular values—and thus the effect sizes—to vary freely along these shared directions.

\medskip
\noindent\textbf{A simple separating example.}
To make the distinction concrete, let us consider the case of $p = q = 2$ with
\[
\mathbf{C}^{(0)}
= \begin{pmatrix}
2 & 0\\
0 & 1
\end{pmatrix},
\qquad
\mathbf{C}^*
= \begin{pmatrix}
M & 0\\
0 & 0
\end{pmatrix}
= M \mathbf{e}_1 \mathbf{e}_1^\top,
\]
where $M > 0$ may depend on $n$, and $\{\mathbf{e}_1,\mathbf{e}_2\}$ denotes the standard basis of $\mathbb{R}^2$.

\emph{(i) Our spectral assumptions hold.}
First, note that the target matrix $\mathbf{C}^*$ has rank one with singular vectors $(\mathbf{U}^*, \mathbf{V}^*) = (\mathbf{e}_1, \mathbf{e}_1)$, whereas the source matrix $\mathbf{C}^{(0)}$ has singular vectors $(\mathbf{U}^{(0)}, \mathbf{V}^{(0)}) = ([\mathbf{e}_1,\mathbf{e}_2],[\mathbf{e}_1,\mathbf{e}_2])$. Then it follows that 
\[
\mathrm{col}(\mathbf{U}^*) = \mathrm{span}\{\mathbf{e}_1\} \subseteq \mathrm{span}\{\mathbf{e}_1,\mathbf{e}_2\} = \mathrm{col}(\mathbf{U}^{(0)}),
\]
\[
\mathrm{col}(\mathbf{V}^*) = \mathrm{span}\{\mathbf{e}_1\} \subseteq \mathrm{span}\{\mathbf{e}_1,\mathbf{e}_2\} = \mathrm{col}(\mathbf{V}^{(0)}),
\]
so Assumption~\ref{assump:relation-target-source} holds exactly. Moreover, the alignment is maximally sparse: the target singular directions coincide with individual source directions, which is precisely the sparse alignment pattern encoded by Assumption~\ref{assump:spectral-sparstiy}.

\emph{(ii) The nuclear-norm contrast condition fails.}
In contrast, the nuclear-norm condition breaks down in this same example. Indeed, the source-target contrast is given by 
\[
\mathbf{\Delta}^{(0)} = \mathbf{C}^{(0)} - \mathbf{C}^*
= \begin{pmatrix}
2-M & 0\\
0 & 1
\end{pmatrix},
\]
whose singular values are $\bigl|2-M\bigr|$ and $1$. Hence, we have that 
\[
\|\mathbf{\Delta}^{(0)}\|_* = \bigl|2-M\bigr| + 1.
\]
If we take $M$ to be large—for instance, $M \asymp n^\alpha$ for some $\alpha > 0$—then $\|\mathbf{\Delta}^{(0)}\|_*$ can grow without bound. As a result, Condition~\eqref{eq:park-nuclear} \textit{cannot} be satisfied with a ``small'' $h_n$. Hence, this choice violates the bounded nuclear-norm contrast condition in \citet{park2025transfer}, despite the fact that the target signal lies entirely within, and is perfectly aligned with, the singular subspace of the source.

\medskip
\noindent\textbf{Practical interpretation.}
More broadly, this example reflects a common pattern in biomedical applications. One may view $\mathbf{C}^{(0)}$ as encoding two functional modules, such as pathways or protein complexes, under a baseline condition, with both modules exhibiting moderate effect sizes. The target matrix $\mathbf{C}^*$ then represents a new condition, such as a particular tissue type or treatment, in which one module (corresponding to $\mathbf{e}_1$) is strongly up-regulated through a large value of $M$, while the other is suppressed. The underlying module structure---namely, the singular subspace spanned by $\mathbf{e}_1$---is perfectly shared between the source and target, so our spectral containment assumption is \textit{natural} in this setting. At the same time, the large change in effect magnitude along that shared direction inflates the nuclear-norm contrast and causes the bounded-difference condition~\eqref{eq:park-nuclear} to fail. In this sense, our assumptions are \textit{better suited} to settings with conserved low-dimensional structure but substantial, interpretable shifts in effect sizes, a regime in which nuclear-norm-based similarity can be unnecessarily restrictive.

\section{Optimization algorithm details}
\label{sec:opt-details}

In this section, we will provide additional details for solving the SMART objective~\eqref{eq:estimator-obj}. Throughout the section, we abbreviate $\widetilde{\mathbf{U}}^{\mathrm{s}}_{\perp}(r_u)$ and $\widetilde{\mathbf{V}}^{\mathrm{s}}_{\perp}(r_v)$ as $\widetilde{\mathbf{U}}^{\mathrm{s}}_{\perp}$ and $\widetilde{\mathbf{V}}^{\mathrm{s}}_{\perp}$, respectively. With such notation in place, the resulting ADMM-type procedure is summarized in Algorithm~\ref{alg:SMART}. In particular, the $\mathbf{U}$- and $\mathbf{V}$-blocks are updated numerically on the \textit{Stiefel manifold}, while the $\mathbf{D}$-, $\bm{\Omega}_u$-, and $\bm{\Omega}_v$-blocks admit closed-form updates.

\begin{algorithm}[htp]
\caption{SMART algorithm}
\label{alg:SMART}
\begin{algorithmic}[1]
\small
\STATE \textbf{Input:} target rank $r$; structural ranks $r_u$, $r_v$; regularization parameters $\lambda_u$, $\lambda_v$; complementary source bases $\widetilde{\mathbf{U}}^{\mathrm{s}}_{\perp}$, $\widetilde{\mathbf{V}}^{\mathrm{s}}_{\perp}$; penalty scaling factor $\gamma > 1$; maximum number of iterations $t_{\max}$.
\STATE \textbf{Initialize:} primal variables $\mathbf{U}^{(0)}, \mathbf{D}^{(0)}, \mathbf{V}^{(0)}$; auxiliary variables $\bm{\Omega}^{(0)}_u, \bm{\Omega}^{(0)}_v$; dual variables $\bm{\Gamma}^{(0)}_u, \bm{\Gamma}^{(0)}_v$; penalty parameter $\rho^{(0)}$; iteration counter $t \leftarrow 0$.
\WHILE{stopping criteria not met and $t < t_{\max}$}
    \STATE Update $\mathbf{U}$:
    \begin{align*}
     \mathbf{U}^{(t+1)} \leftarrow \arg \min_{ \mathbf{U} : \mathbf{U}^{\top} \mathbf{U} = \mathbf{I}_r } & \Bigg\{ \frac{1}{2 n} \left\Vert \mathbf{Y} - \mathbf{X} \mathbf{U} \mathbf{D}^{(t)} \mathbf{V}^{(t)\top} \right\Vert^2_{\mathrm{F}}  \\
    &  \quad + \left\langle \bm{\Gamma}^{(t)}_u, \widetilde{\mathbf{U}}^{\mathrm{s}\top}_{\perp} \mathbf{U} \mathbf{D}^{(t)} - \bm{\Omega}^{(t)}_u \right\rangle
    + \frac{\rho^{(t)}}{2} \left\Vert \widetilde{\mathbf{U}}^{\mathrm{s}\top}_{\perp} \mathbf{U} \mathbf{D}^{(t)} - \bm{\Omega}^{(t)}_u \right\Vert^2_{\mathrm{F}}\Bigg\}.
    \end{align*}

    \STATE Update $\mathbf{V}$:
    \begin{align*}
     \mathbf{V}^{(t+1)} \leftarrow \arg \min_{ \mathbf{V} : \mathbf{V}^{\top} \mathbf{V} = \mathbf{I}_r } & \Bigg\{ \frac{1}{2 n} \left\Vert \mathbf{Y} - \mathbf{X} \mathbf{U}^{(t+1)} \mathbf{D}^{(t)} \mathbf{V}^{\top} \right\Vert^2_{\mathrm{F}}  \\
    &  \quad + \left\langle \bm{\Gamma}^{(t)}_v, \widetilde{\mathbf{V}}^{\mathrm{s}\top}_{\perp} \mathbf{V} \mathbf{D}^{(t)} - \bm{\Omega}^{(t)}_v \right\rangle
    + \frac{\rho^{(t)}}{2} \left\Vert \widetilde{\mathbf{V}}^{\mathrm{s}\top}_{\perp} \mathbf{V} \mathbf{D}^{(t)} - \bm{\Omega}^{(t)}_v \right\Vert^2_{\mathrm{F}} \Bigg\}.
    \end{align*}

    \STATE Update $\mathbf{D}$:
     \begin{align*}
    \mathbf{D}^{(t+1)} \leftarrow \arg \min_{\mathbf{D}=\diag(d_1,\ldots,d_r),\ d_k \ge 0} & \Bigg\{ \frac{1}{2 n} \left\Vert \mathbf{Y} - \mathbf{X}  \mathbf{U}^{(t+1)} \mathbf{D} \mathbf{V}^{{(t+1)} \top}
    \right\Vert^2_{\mathrm{F}}   \\  &
    \quad + \left\langle \bm{\Gamma}^{(t)}_u, \widetilde{\mathbf{U}}^{\mathrm{s}  \top}_{\perp} \mathbf{U}^{(t+1)} \mathbf{D} - \bm{\Omega}^{(t)}_u \right\rangle
    + \left\langle \bm{\Gamma}^{(t)}_v, \widetilde{\mathbf{V}}^{\mathrm{s}  \top}_{\perp} \mathbf{V}^{(t+1)} \mathbf{D} - \bm{\Omega}^{(t)}_v  \right\rangle   \\
    &   \quad + \frac{\rho^{(t)}}{2} \left\Vert \widetilde{\mathbf{U}}^{\mathrm{s}  \top}_{\perp} \mathbf{U}^{(t+1)} \mathbf{D} - \bm{\Omega}^{(t)}_u \right\Vert^2_{\mathrm{F}}
    + \frac{\rho^{(t)}}{2} \left\Vert \widetilde{\mathbf{V}}^{\mathrm{s}  \top}_{\perp} \mathbf{V}^{(t+1)} \mathbf{D} - \bm{\Omega}^{(t)}_v \right\Vert^2_{\mathrm{F}} \Bigg\}.
    \end{align*}

    \STATE Update $\bm{\Omega}_u$:
    \[
    \bm{\Omega}^{(t+1)}_u \leftarrow \arg \min_{\bm{\Omega}_u} \left\{
        \frac{\rho^{(t)}}{2} \left\Vert \bm{\Omega}_u - \widetilde{\mathbf{U}}^{\mathrm{s}\top}_{\perp} \mathbf{U}^{(t+1)} \mathbf{D}^{(t+1)} \right\Vert^2_{\mathrm{F}}
        - \left\langle \bm{\Gamma}^{(t)}_u, \bm{\Omega}_u \right\rangle
        + \lambda_u \left\vert \bm{\Omega}_u \right\vert_1
    \right\}.
    \]

    \STATE Update $\bm{\Omega}_v$:
    \[
    \bm{\Omega}^{(t+1)}_v \leftarrow \arg \min_{\bm{\Omega}_v} \left\{
        \frac{\rho^{(t)}}{2} \left\Vert \bm{\Omega}_v - \widetilde{\mathbf{V}}^{\mathrm{s}\top}_{\perp} \mathbf{V}^{(t+1)} \mathbf{D}^{(t+1)} \right\Vert^2_{\mathrm{F}}
        - \left\langle \bm{\Gamma}^{(t)}_v, \bm{\Omega}_v \right\rangle
        + \lambda_v \left\vert \bm{\Omega}_v \right\vert_1
    \right\}.
    \]

    \STATE Update dual variables:
    \begin{align*}
    \bm{\Gamma}^{(t+1)}_u &\leftarrow \bm{\Gamma}^{(t)}_u + \rho^{(t)} \left( \widetilde{\mathbf{U}}^{\mathrm{s}\top}_{\perp} \mathbf{U}^{(t+1)} \mathbf{D}^{(t+1)} - \bm{\Omega}^{(t+1)}_u \right), \\
    \bm{\Gamma}^{(t+1)}_v &\leftarrow \bm{\Gamma}^{(t)}_v + \rho^{(t)} \left( \widetilde{\mathbf{V}}^{\mathrm{s}\top}_{\perp} \mathbf{V}^{(t+1)} \mathbf{D}^{(t+1)} - \bm{\Omega}^{(t+1)}_v \right).
    \end{align*}

    \STATE Update penalty parameter: $\rho^{(t+1)} \leftarrow \gamma \rho^{(t)}$.
    \STATE Increment iteration counter: $t \leftarrow t+1$.
\ENDWHILE

\STATE \textbf{Output:}
\[
\widehat{\mathbf{U}} \leftarrow \mathbf{U}^{(t)}, \quad
\widehat{\mathbf{V}} \leftarrow \mathbf{V}^{(t)}, \quad
\widehat{\mathbf{D}} \leftarrow \mathbf{D}^{(t)}, \quad
\widehat{\bm{\Omega}}_u \leftarrow \bm{\Omega}^{(t)}_u, \quad
\widehat{\bm{\Omega}}_v \leftarrow \bm{\Omega}^{(t)}_v,
\widehat{\bm{\Gamma}}_u \leftarrow \bm{\Gamma}^{(t)}_u, \quad
\widehat{\bm{\Gamma}}_v \leftarrow \bm{\Gamma}^{(t)}_v.
\]
\end{algorithmic}
\end{algorithm}

When $r_u = p$ or $r_v = q$, the corresponding complementary basis has zero rows, so the associated $\bm{\Omega}$- and $\bm{\Gamma}$-blocks are omitted. With that convention fixed, we now describe each block update in turn.

We begin with the $\mathbf{U}$-update, which is equivalent to
\begin{align*}
\mathbf{U}^{(t+1)} \leftarrow & \arg \min_{\mathbf{U}} \left\{
\frac{1}{2n} \left\Vert \mathbf{X} \mathbf{U} \mathbf{D}^{(t)} \right\Vert^2_{\mathrm{F}}
+ \frac{\rho^{(t)}}{2} \left\Vert \widetilde{\mathbf{U}}^{\mathrm{s}\top}_{\perp} \mathbf{U} \mathbf{D}^{(t)} \right\Vert^2_{\mathrm{F}}
+ \mathrm{tr} \left( \mathbf{U}^{\top} \mathbf{B}^{(t)}_u \right)
\right\} \\
& \text{s.t. } \mathbf{U}^{\top} \mathbf{U} = \mathbf{I}_r,
\end{align*}
where
\[
\mathbf{B}^{(t)}_u = \widetilde{\mathbf{U}}^{\mathrm{s}}_{\perp} \bm{\Gamma}^{(t)}_u \mathbf{D}^{(t)}
- \frac{1}{n} \mathbf{X}^{\top} \mathbf{Y} \mathbf{V}^{(t)} \mathbf{D}^{(t)}
- \rho^{(t)} \widetilde{\mathbf{U}}^{\mathrm{s}}_{\perp} \bm{\Omega}^{(t)}_u \mathbf{D}^{(t)}.
\]
This is a smooth quadratic program on the Stiefel manifold, and we solve it approximately using the Riemannian conjugate gradient method implemented in \texttt{Pymanopt}~\citep{townsend2016pymanopt}.

Next, let us consider the $\mathbf{V}$-update. After expanding the loss and using $\mathbf{V}^{\top}\mathbf{V} = \mathbf{I}_r$, term $\| \mathbf{X} \mathbf{U}^{(t+1)} \mathbf{D}^{(t)} \mathbf{V}^{\top} \|_{\mathrm{F}}^2$ is constant with respect to $\mathbf{V}$. Accordingly, the subproblem reduces to
\[
\mathbf{V}^{(t+1)} \leftarrow \arg \min_{\mathbf{V}} \left\{
\frac{\rho^{(t)}}{2} \left\Vert \widetilde{\mathbf{V}}^{\mathrm{s}\top}_{\perp} \mathbf{V} \mathbf{D}^{(t)} \right\Vert^2_{\mathrm{F}}
+ \mathrm{tr} \left( \mathbf{V}^{\top} \mathbf{B}^{(t)}_v \right)
\right\}
\quad \text{s.t. } \mathbf{V}^{\top} \mathbf{V} = \mathbf{I}_r,
\]
where
\[
\mathbf{B}^{(t)}_v = \widetilde{\mathbf{V}}^{\mathrm{s}}_{\perp} \bm{\Gamma}^{(t)}_v \mathbf{D}^{(t)}
- \frac{1}{n} \mathbf{Y}^{\top} \mathbf{X} \mathbf{U}^{(t+1)} \mathbf{D}^{(t)}
- \rho^{(t)} \widetilde{\mathbf{V}}^{\mathrm{s}}_{\perp} \bm{\Omega}^{(t)}_v \mathbf{D}^{(t)}.
\]
As in the previous step, this problem can be solved numerically with \texttt{Pymanopt}.

We now turn to the $\mathbf{D}$-update. Denote by 
\[
\widetilde{\bm{\Omega}}_u \coloneqq \bm{\Omega}^{(t)}_u - \rho^{(t)^{-1}} \bm{\Gamma}^{(t)}_u,
\qquad
\widetilde{\bm{\Omega}}_v \coloneqq \bm{\Omega}^{(t)}_v - \rho^{(t)^{-1}} \bm{\Gamma}^{(t)}_v.
\]
Then the subproblem becomes
\begin{align*}
\mathbf{D}^{(t+1)} \leftarrow \arg \min_{\mathbf{D}=\diag(d_1,\ldots,d_r),\ d_k \ge 0} \Big\{
& \frac{1}{2n} \left\Vert \mathbf{Y} - \mathbf{X} \mathbf{U}^{(t+1)} \mathbf{D} \mathbf{V}^{(t+1)\top} \right\Vert^2_{\mathrm{F}} \\
& + \frac{\rho^{(t)}}{2} \left\Vert \widetilde{\mathbf{U}}^{\mathrm{s}\top}_{\perp} \mathbf{U}^{(t+1)} \mathbf{D} - \widetilde{\bm{\Omega}}_u \right\Vert^2_{\mathrm{F}} \\
& + \frac{\rho^{(t)}}{2} \left\Vert \widetilde{\mathbf{V}}^{\mathrm{s}\top}_{\perp} \mathbf{V}^{(t+1)} \mathbf{D} - \widetilde{\bm{\Omega}}_v \right\Vert^2_{\mathrm{F}}
\Big\}.
\end{align*}
To obtain a closed-form solution, let us write
\[
\mathbf{P} \coloneqq \mathbf{X} \mathbf{U}^{(t+1)} = [\mathbf{p}_1, \ldots, \mathbf{p}_r], \qquad
\mathbf{Q} \coloneqq \mathbf{V}^{(t+1)} = [\mathbf{q}_1, \ldots, \mathbf{q}_r],
\]
\[
\mathbf{A}_u \coloneqq \widetilde{\mathbf{U}}^{\mathrm{s}\top}_{\perp} \mathbf{U}^{(t+1)} = [\mathbf{a}_{u,1}, \ldots, \mathbf{a}_{u,r}], \qquad
\mathbf{A}_v \coloneqq \widetilde{\mathbf{V}}^{\mathrm{s}\top}_{\perp} \mathbf{V}^{(t+1)} = [\mathbf{a}_{v,1}, \ldots, \mathbf{a}_{v,r}],
\]
and denote by $\tilde{\bm{\omega}}_{u,k}$ and $\tilde{\bm{\omega}}_{v,k}$ the $k$th columns of $\widetilde{\bm{\Omega}}_u$ and $\widetilde{\bm{\Omega}}_v$, respectively. Since $\mathbf{D}$ is diagonal and $\mathbf{Q}$ has orthonormal columns, the objective decouples across $d_1,\ldots,d_r$. We define
\[
\alpha_k \coloneqq \frac{1}{n} \| \mathbf{p}_k \|_2^2, \qquad
\beta_k \coloneqq \frac{1}{n} \mathbf{p}_k^{\top} \mathbf{Y} \mathbf{q}_k,
\]
\[
\gamma_{u,k} \coloneqq \| \mathbf{a}_{u,k} \|_2^2, \qquad
\delta_{u,k} \coloneqq \mathbf{a}_{u,k}^{\top} \tilde{\bm{\omega}}_{u,k},
\]
\[
\gamma_{v,k} \coloneqq \| \mathbf{a}_{v,k} \|_2^2, \qquad
\delta_{v,k} \coloneqq \mathbf{a}_{v,k}^{\top} \tilde{\bm{\omega}}_{v,k}.
\]
Then the $k$th coordinate has objective $(1/2)c_k d_k^2 - b_k d_k$, where
\[
c_k \coloneqq \alpha_k + \rho^{(t)} (\gamma_{u,k} + \gamma_{v,k}), \qquad
b_k \coloneqq \beta_k + \rho^{(t)} (\delta_{u,k} + \delta_{v,k}).
\]
Hence, we have that 
\[
d_k^{(t+1)} =
\begin{cases}
\max \{ 0, b_k / c_k \}, & c_k > 0, \\
0, & c_k = 0,
\end{cases}
\]
for $k = 1, \ldots, r$.

The $\bm{\Omega}_u$- and $\bm{\Omega}_v$-updates are proximal steps. More precisely, we have that 
\[
\bm{\Omega}^{(t+1)}_u \leftarrow \arg \min_{\bm{\Omega}_u} \left\{
\frac{\rho^{(t)}}{2} \left\Vert \bm{\Omega}_u - \widetilde{\mathbf{U}}^{\mathrm{s}\top}_{\perp} \mathbf{U}^{(t+1)} \mathbf{D}^{(t+1)} - \rho^{(t)^{-1}} \bm{\Gamma}^{(t)}_u \right\Vert^2_{\mathrm{F}}
+ \lambda_u |\bm{\Omega}_u|_1
\right\},
\]
whose solution is obtained entrywise by soft-thresholding
\[
\bm{\Omega}^{(t+1)}_{u,ij} =
\mathcal{ST}_{\lambda_u / \rho^{(t)}} \left(
\left[
\widetilde{\mathbf{U}}^{\mathrm{s}\top}_{\perp} \mathbf{U}^{(t+1)} \mathbf{D}^{(t+1)}
+ \rho^{(t)^{-1}} \bm{\Gamma}^{(t)}_u
\right]_{ij}
\right).
\]
Similarly, we have that 
\[
\bm{\Omega}^{(t+1)}_{v,ij} =
\mathcal{ST}_{\lambda_v / \rho^{(t)}} \left(
\left[
\widetilde{\mathbf{V}}^{\mathrm{s}\top}_{\perp} \mathbf{V}^{(t+1)} \mathbf{D}^{(t+1)}
+ \rho^{(t)^{-1}} \bm{\Gamma}^{(t)}_v
\right]_{ij}
\right).
\]
Here, the soft-thresholding operator is defined as 
\[
\mathcal{ST}_{\kappa}(a) =
\begin{cases}
a - \kappa, & a > \kappa, \\
0, & |a| \leq \kappa, \\
a + \kappa, & a < -\kappa.
\end{cases}
\]

\textbf{Stopping criterion.}
Finally, we adopt a practical convergence criterion based on the primal feasibility and blockwise stationarity. The residuals are evaluated before updating the penalty parameter for the next iteration.

To begin, let us define the primal residuals
\[
\mathbf{R}^{(t+1)}_{\mathrm{p,u}} \coloneqq \widetilde{\mathbf{U}}^{\mathrm{s}\top}_{\perp} \mathbf{U}^{(t+1)} \mathbf{D}^{(t+1)} - \bm{\Omega}^{(t+1)}_u,
\qquad
\mathbf{R}^{(t+1)}_{\mathrm{p,v}} \coloneqq \widetilde{\mathbf{V}}^{\mathrm{s}\top}_{\perp} \mathbf{V}^{(t+1)} \mathbf{D}^{(t+1)} - \bm{\Omega}^{(t+1)}_v.
\]
Denote by 
\[
\eta_u^{(t+1)} \coloneqq
\begin{cases}
\displaystyle \frac{\| \mathbf{R}^{(t+1)}_{\mathrm{p,u}} \|_{\mathrm{F}}}{\sqrt{(p-r_u)r}}, & r_u < p, \\
0, & r_u = p,
\end{cases}
\qquad
\eta_v^{(t+1)} \coloneqq
\begin{cases}
\displaystyle \frac{\| \mathbf{R}^{(t+1)}_{\mathrm{p,v}} \|_{\mathrm{F}}}{\sqrt{(q-r_v)r}}, & r_v < q, \\
0, & r_v = q.
\end{cases}
\]
We then aggregate these two quantities through
\[
r^{(t+1)}_{\mathrm{primal}} = \frac{1}{2} \eta_u^{(t+1)} + \frac{1}{2} \eta_v^{(t+1)}.
\]

Since $\mathbf{D}$, $\bm{\Omega}_u$, and $\bm{\Omega}_v$ are updated in closed form, we monitor first-order stationarity only for the manifold-constrained $\mathbf{U}$- and $\mathbf{V}$-blocks. To this end, we define the current block objectives
\begin{align*}
f^{(t+1)}(\mathbf{U}) & =
\frac{1}{2n} \left\Vert \mathbf{Y} - \mathbf{X} \mathbf{U} \mathbf{D}^{(t+1)} \mathbf{V}^{(t+1)\top} \right\Vert^2_{\mathrm{F}} \\
& \qquad + \left\langle \bm{\Gamma}^{(t+1)}_u, \widetilde{\mathbf{U}}^{\mathrm{s}\top}_{\perp} \mathbf{U} \mathbf{D}^{(t+1)} - \bm{\Omega}^{(t+1)}_u \right\rangle
+ \frac{\rho^{(t)}}{2} \left\Vert \widetilde{\mathbf{U}}^{\mathrm{s}\top}_{\perp} \mathbf{U} \mathbf{D}^{(t+1)} - \bm{\Omega}^{(t+1)}_u \right\Vert^2_{\mathrm{F}}, \\
g^{(t+1)}(\mathbf{V}) & =
\frac{1}{2n} \left\Vert \mathbf{Y} - \mathbf{X} \mathbf{U}^{(t+1)} \mathbf{D}^{(t+1)} \mathbf{V}^{\top} \right\Vert^2_{\mathrm{F}} \\
& \qquad + \left\langle \bm{\Gamma}^{(t+1)}_v, \widetilde{\mathbf{V}}^{\mathrm{s}\top}_{\perp} \mathbf{V} \mathbf{D}^{(t+1)} - \bm{\Omega}^{(t+1)}_v \right\rangle
+ \frac{\rho^{(t)}}{2} \left\Vert \widetilde{\mathbf{V}}^{\mathrm{s}\top}_{\perp} \mathbf{V} \mathbf{D}^{(t+1)} - \bm{\Omega}^{(t+1)}_v \right\Vert^2_{\mathrm{F}}.
\end{align*}
Let $\mathcal{M}_{p,r} = \{ \mathbf{U} \in \mathbb{R}^{p \times r} : \mathbf{U}^{\top} \mathbf{U} = \mathbf{I}_r \}$ and $\mathcal{M}_{q,r} = \{ \mathbf{V} \in \mathbb{R}^{q \times r} : \mathbf{V}^{\top} \mathbf{V} = \mathbf{I}_r \}$. We define the stationarity residual as 
\[
r^{(t+1)}_{\mathrm{stat}} =
\frac{1}{2} \frac{ \left\Vert \nabla_{\mathcal{M}_{p,r}} f^{(t+1)}(\mathbf{U}^{(t+1)}) \right\Vert_{\mathrm{F}} }{ \sqrt{pr} }
+
\frac{1}{2} \frac{ \left\Vert \nabla_{\mathcal{M}_{q,r}} g^{(t+1)}(\mathbf{V}^{(t+1)}) \right\Vert_{\mathrm{F}} }{ \sqrt{qr} },
\]
where $\nabla_{\mathcal{M}_{p,r}} f(\mathbf{U})$ denotes the Riemannian gradient of $f$ at $\mathbf{U}$.

In summary, for a user-specified tolerance $\varepsilon > 0$, we will terminate the algorithm once
\[
r^{(t+1)}_{\mathrm{primal}} \leq \varepsilon
\qquad \text{and} \qquad
r^{(t+1)}_{\mathrm{stat}} \leq \varepsilon.
\]

\section{Details of the multi-modal single-cell experiment}
\label{sec:details-real-world-experiment}

This section provides additional details for the multi-modal single-cell analysis reported in Section~\ref{sec:real-world-data-smart}. Figure~\ref{fig:cell-type-number} displays the cell counts across annotated cell types. We next summarize the preprocessing pipeline, the implementation details, and the baseline comparisons used in this experiment.

\textbf{Data preprocessing.}
We extracted the gene expression (GEX) and protein expression (ADT) matrices from the raw dataset, which contains $13{,}953$ genes and $134$ proteins measured across $90{,}261$ cells. For GEX, we applied the library-size normalization, log transformation, selection of the top $3{,}000$ highly variable genes, and standardization. For ADT, we applied the centered log-ratio normalization. We then restricted attention to NK and ILC1 cells, which yielded $5{,}434$ source cells and $552$ target cells. After removing genes with zero variance in either group, $2{,}998$ GEX features and all $134$ ADT markers remained. Finally, we used marginal screening to define a common set of 248 gene predictors, and this predictor set was used throughout the regression analysis.

\textbf{Implementation.}
To evaluate SMART, we first constructed the noisy source matrix $\widetilde{\RegM}^{(0)}$ by applying several baseline estimators to the NK (source) data. Specifically, we considered seven methods: the Lasso, Ridge, OLS, RRR, SRRR, RSSVD, and SOFAR. For each method, the resulting estimate $\widetilde{\RegM}^{(0)}$ was passed to SMART to produce a target regression estimate $\widehat{\RegM}$. The model was trained on a random $70\%$ subset of the ILC1 (target) cells and evaluated on the remaining $30\%$ held-out observations. We measured prediction accuracy using the Frobenius norm of the test error
\[
\frac{1}{\sqrt{n_{\mathrm{test}} q}} \left\| \mathbf{X}_{\mathrm{test}} \widehat{\RegM} - \mathbf{Y}_{\mathrm{test}} \right\|_{\mathrm{F}},
\]
where $n_{\mathrm{test}}$ denotes the number of test observations and $q$ denotes the number of output variables. To improve robustness, we repeated this procedure $100$ times using different train/test splits generated from a fixed list of random seeds. For each method, we report the mean, standard deviation, median, and interquartile range (Q1 and Q3) of the prediction errors across repetitions.

\begin{figure}[t]
\centering
\includegraphics[width=\textwidth]{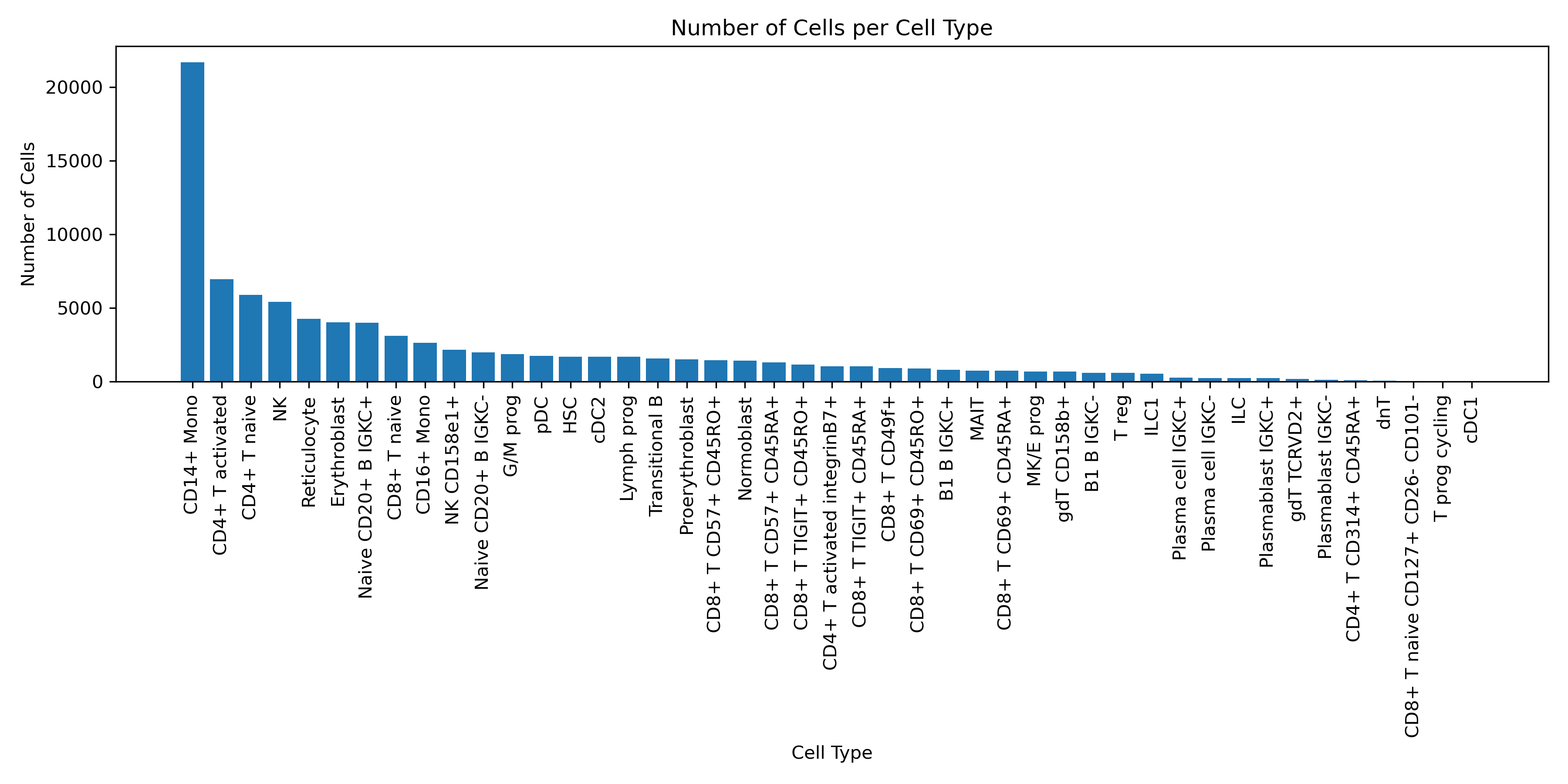}
\caption{Cell counts for each annotated cell type in the multi-modal single-cell data.}
\label{fig:cell-type-number}
\end{figure}

\textbf{Baseline methods.}
To place the performance of SMART in context, we compared it to a collection of baseline methods divided into two groups: \emph{source-only} and \emph{target-only}. The source-only methods applied the regression matrix estimated from the NK data directly to the ILC1 test set, without any retraining. In contrast, the target-only methods applied the same family of estimators---the Lasso, Ridge, OLS, RRR, SRRR, RSSVD, and SOFAR---only to the ILC1 data, without using the source information. For comparability, each method was evaluated on the same 70\%/30\% train/test splits and with the same error metric used for SMART. The source-only methods were evaluated directly on the test set, whereas the target-only methods were fit on the target training split before evaluation. As a final step, each experiment was repeated over $100$ random splits, and the reported summaries were computed across those repetitions.

\section{Proof of Proposition~\ref{proposition:lasso}}
\label{sec:proof-proposition:lasso}

For any matrix $\mathbf{A}$, let $\mathbf{A}_{:\,j}$ be its $j$th column. We consider the matrix Lasso estimator
\begin{equation*}
\widehat{\RegM}^{\mathrm{lasso}} \in \arg \min_{\RegM \in \mathbb{R}^{p \times q}}
\left\{
\frac{1}{2n} \left\Vert \mathbf{Y} - \mathbf{X}\RegM \right\Vert_{\mathrm{F}}^2
+ \lambda_0 \left\vert \RegM \right\vert_1
\right\}.
\end{equation*}
Since the objective is separable across columns, the $j$th column of $\widehat{\RegM}^{\mathrm{lasso}}$ can be computed as
\begin{equation*}
\widehat{\RegM}^{\mathrm{lasso}}_{:\,j} \in \arg \min_{\bm{\beta} \in \mathbb{R}^p}
\left\{
\frac{1}{2n} \left\Vert \mathbf{Y}_{:\,j} - \mathbf{X}\bm{\beta} \right\Vert_2^2
+ \lambda_0 \left\Vert \bm{\beta} \right\Vert_1
\right\},
\qquad 1 \le j \le q.
\end{equation*}

For each $1 \le j \le q$, recall that
\[
s_j \coloneqq \bigl\vert \RegM^*_{:\,j} \bigr\vert_0,
\qquad
s_{\max} \coloneqq \max_{1 \le j \le q} s_j,
\qquad
s_{*,1,2} \coloneqq \left( \sum_{j=1}^q s_j^2 \right)^{1/2}.
\]
Let us also define
\[
\delta_n^2 \coloneqq \frac{c_0^2 - 2}{4} \cdot \frac{\log p}{n}
+ \frac{c_0^2}{4} \cdot \frac{\log q}{n},
\qquad
\lambda_0 \coloneqq 4 \sqrt{\nu \alpha_{\max}}
\left( \sqrt{\frac{\log p}{n}} + \delta_n \right),
\]
where $\nu$ is the constant introduced in Assumption~\ref{assump:constrained-norm-design}(i).

Since the rows of $\mathbf{E}$ are independent and identically distributed (i.i.d.)~Gaussian with covariance matrix $\bm{\Sigma}$, the noise vector in the $j$th column model satisfies that 
\[
\mathbf{E}_{:\,j} \sim \mathcal{N}(0, \Sigma_{jj}\mathbf{I}_n),
\qquad
\Sigma_{jj} \le \|\bm{\Sigma}\|_2 \le \alpha_{\max}.
\]
Accordingly, the common choice $\lambda_0$ dominates the column-specific penalty level required by~\citet[Corollary 9.26]{wainwright2019high}. Moreover, it follows from the elementary bound $(a+b)^2 \le 2a^2 + 2b^2$ that 
\begin{equation*}
\lambda_0^2
=
16 \nu \alpha_{\max}
\left( \sqrt{\frac{\log p}{n}} + \delta_n \right)^2
\le
32 \nu \alpha_{\max}
\left( \frac{\log p}{n} + \delta_n^2 \right).
\end{equation*}
Then it holds that 
\begin{equation*}
\lambda_0^2
\le
32 \nu \alpha_{\max}
\left\{
\frac{\log p}{n}
+
\left( \frac{c_0^2 - 2}{4} \right)\frac{\log p}{n}
+
\left( \frac{c_0^2}{4} \right)\frac{\log q}{n}
\right\}.
\end{equation*}
Using condition~\eqref{eq:n-requir-lasso} and the bound $s_j \le s_{\max}$ for all $1 \le j \le q$, it follows that
\begin{equation*}
s_j \left( \lambda_0^2 + \frac{\log p}{n} \right)
<
\frac{4\rho_2^2}{9} \wedge \frac{\rho_2}{64\rho_1}
\qquad \text{for all } 1 \le j \le q.
\end{equation*}
Hence, an application of \citet[Corollary 9.26]{wainwright2019high} yields that for each $j \in [q]$,
\begin{equation}
\label{eq:lasso-proof-each-column}
\left\Vert \widehat{\RegM}^{\mathrm{lasso}}_{:\,j} - \RegM^*_{:\,j} \right\Vert_2^2
\le
\frac{9 s_j \lambda_0^2}{4 \rho_2^2},
\qquad
\left\Vert \widehat{\RegM}^{\mathrm{lasso}}_{:\,j} - \RegM^*_{:\,j} \right\Vert_1
\le
\frac{6 s_j \lambda_0}{\rho_2},
\end{equation}
with probability at least $1 - 2 e^{-2n\delta_n^2}$.

A union bound now shows that \eqref{eq:lasso-proof-each-column} holds simultaneously for all $j \in [q]$ with probability at least
\[
1 - 2q e^{-2n\delta_n^2}
=
1 - 2(pq)^{1 - c_0^2/2}.
\]
On this event, summing the squared $\ell_2$-errors across columns leads to 
\begin{equation*}
\left\Vert \widehat{\RegM}^{\mathrm{lasso}} - \RegM^* \right\Vert_{\mathrm{F}}^2
=
\sum_{j=1}^q
\left\Vert \widehat{\RegM}^{\mathrm{lasso}}_{:\,j} - \RegM^*_{:\,j} \right\Vert_2^2
\le
\frac{9\lambda_0^2}{4\rho_2^2}
\sum_{j=1}^q s_j
=
\frac{9\lambda_0^2 s_0}{4\rho_2^2}
\lesssim
\frac{\alpha_{\max} s_0 \log(pq)}{n}.
\end{equation*}
Taking square roots, we can deduce that 
\begin{equation*}
\left\Vert \widehat{\RegM}^{\mathrm{lasso}} - \RegM^* \right\Vert_{\mathrm{F}}
\lesssim
\sqrt{\frac{\alpha_{\max} s_0 \log(pq)}{n}}.
\end{equation*}

For the mixed norm $\|\cdot\|_{1,2}$, we can instead aggregate the columnwise $\ell_1$-bounds in $\ell_2$
\begin{equation*}
\left\Vert \widehat{\RegM}^{\mathrm{lasso}} - \RegM^* \right\Vert_{1,2}
=
\left(
\sum_{j=1}^q
\left\Vert \widehat{\RegM}^{\mathrm{lasso}}_{:\,j} - \RegM^*_{:\,j} \right\Vert_1^2
\right)^{1/2}
\le
\frac{6\lambda_0}{\rho_2}
\left( \sum_{j=1}^q s_j^2 \right)^{1/2}
=
\frac{6\lambda_0 s_{*,1,2}}{\rho_2}.
\end{equation*}
Therefore, in view of  $\lambda_0 \asymp \sqrt{\alpha_{\max}\log(pq)/n}$, we can obtain that
\begin{equation*}
\left\Vert \widehat{\RegM}^{\mathrm{lasso}} - \RegM^* \right\Vert_{1,2}
\lesssim
s_{*,1,2}\sqrt{\frac{\alpha_{\max}\log(pq)}{n}}.
\end{equation*}
This completes the proof of Proposition~\ref{proposition:lasso}.

\section{Proof of Theorem~\ref{thm:spectral-trans}}
\label{sec:tech-details-smart}

This section contains the complete proof of Theorem~\ref{thm:spectral-trans}. To streamline the arguments, let us first introduce some notation that will be used throughout. We denote the structured factor matrices and their estimators as
\[
\mathbf{A}^* = \mathbf{U}^* \mathbf{D}^*,
\quad
\mathbf{B}^* = \mathbf{V}^* \mathbf{D}^*,
\quad
\widehat{\mathbf{A}} = \widehat{\mathbf{U}} \widehat{\mathbf{D}},
\quad
\text{and}
\quad
\widehat{\mathbf{B}} = \widehat{\mathbf{V}} \widehat{\mathbf{D}},
\]
respectively. Define the estimation error of the regression matrix and its structured components as 
\begin{equation}
\label{eq:AB-def}
\widehat{\bm{\Delta}} = \widehat{\RegM} - \RegM^*,
\quad
\widehat{\bm{\Delta}}^a = \widehat{\mathbf{A}} - \mathbf{A}^*,
\quad
\widehat{\bm{\Delta}}^b = \widehat{\mathbf{B}} - \mathbf{B}^*,
\quad
\widehat{\bm{\Delta}}^d = \widehat{\mathbf{D}} - \mathbf{D}^*,
\end{equation}
and write $\widehat{\bm{\Delta}}^{\mathrm{init}} = \widehat{\RegM}^{\mathrm{init}} - \RegM^*$ for the error of the initial estimator. We further define the generalized inverse diagonal matrix
\[
\widehat{\mathbf{D}}^- = \operatorname{diag}(\widehat{d}_1^-, \ldots, \widehat{d}_r^-)
\quad \text{with} \quad
\widehat{d}_j^- = \widehat{d}_j^{-1} \cdot \mathds{1} \{ \widehat{d}_j > 0 \} 
\]
for all $1 \leq j \leq r$.

\subsection{Some technical lemmas}

In this subsection, we will collect several technical lemmas that are used repeatedly in the main arguments. We start with a perturbation result that guarantees the sign consistency between the estimated and true singular vectors.
\begin{lemma}[Sign consistency]
\label{lemma:sign-consistency}
Let
\[
\RegM = \mathbf{U} \mathbf{D} \mathbf{V}^{\top}
= \sum_{k=1}^{r} d_k \mathbf{u}_k \mathbf{v}_k^{\top},
\qquad
\RegM^{*} = \mathbf{U}^{*} \mathbf{D}^{*} \mathbf{V}^{* \top}
= \sum_{k=1}^{r} d_k^* \mathbf{u}_k^* \mathbf{v}_k^{*\top}
\]
be rank-$r$ matrices, with singular values ordered as
\(d_1 \ge \cdots \ge d_r > 0\) and \(d_1^* \ge \cdots \ge d_r^* > 0\).
Assume that the squared singular values of $\RegM^{*}$ satisfy the separation condition
\begin{equation*}
d^{*2}_{j-1} - d^{*2}_j \geq \tau d^{*2}_{j-1}
\quad \text{for all } 2 \leq j \leq r,
\end{equation*}
for some universal constant $0 < \tau < 1$. Fix any index $1 \leq j \leq r$. If the perturbation satisfies that 
\begin{equation}
\label{eq:sign-consistency-condition}
\left\Vert \RegM - \RegM^{*} \right\Vert_2 \leq \frac{\tau}{6 \sqrt{5}} \cdot \left( \frac{d^{*}_j}{\sqrt{r}} \right) \cdot \left( \frac{d^{*}_j}{d^{*}_1} \right)^{3/2},
\end{equation}
the $j$th pair of singular vectors has aligned sign
\begin{equation}
\label{eq:sign-consistent-equation}
\left( \mathbf{u}_j^{\top} \mathbf{u}^{*}_j \right) \cdot \left( \mathbf{v}_j^{\top} \mathbf{v}^{*}_j \right) > 0.
\end{equation}

Consequently, if condition~\eqref{eq:sign-consistency-condition} holds for $j=r$, the sign alignment in~\eqref{eq:sign-consistent-equation} holds uniformly for all $j = 1, \dots, r$.
\end{lemma}

%\begin
\textit{Proof}. Let us set $\varepsilon = \|\RegM - \RegM^*\|_2$. We argue by contradiction. Assume that
\[
\left( \mathbf{u}_j^{\top} \mathbf{u}_j^{*} \right) \cdot \left( \mathbf{v}_j^{\top} \mathbf{v}_j^{*} \right) \leq 0.
\]
Then it holds that 
\begin{align*}
\mathbf{u}_j^{* \top} \left( \RegM^{*} - \RegM \right) \mathbf{v}_j^{*}
&= d_j^{*} - d_j \left( \mathbf{u}_j^{\top} \mathbf{u}_j^{*} \right) \left( \mathbf{v}_j^{\top} \mathbf{v}_j^{*} \right)
- \sum_{k \neq j} d_k \left( \mathbf{u}_k^{\top} \mathbf{u}_j^{*} \right) \left( \mathbf{v}_k^{\top} \mathbf{v}_j^{*} \right) \\
&\geq d_j^{*} - \sum_{k \neq j} d_k \left| \mathbf{u}_k^{\top} \mathbf{u}_j^{*} \right| \left| \mathbf{v}_k^{\top} \mathbf{v}_j^{*} \right|. \numberthis \label{eq:sign-consistency-1}
\end{align*}
We will apply Theorem~3 of~\cite{yu2015useful} to $\RegM \RegM^\top$ and $\RegM^* \RegM^{*\top}$. In light of 
\[
\|\RegM \RegM^\top - \RegM^* \RegM^{*\top}\|_2
\le (\|\RegM\|_2 + \|\RegM^*\|_2)\,\varepsilon
\le (2d_1^* + \varepsilon)\,\varepsilon,
\]
there exists a sign $s_j^u \in \{-1,+1\}$ such that
\[
\left\Vert s_j^u \mathbf{u}_j - \mathbf{u}_j^{*} \right\Vert_2
\leq
\frac{ 2^{3/2} ( 2 d_1^{*} + \varepsilon ) \varepsilon }
{ \min \left( d_{j-1}^{*2} - d_j^{*2},  d_j^{*2} - d_{j+1}^{*2} \right) },
\]
where $d_0^{*2} = \infty$ and $d_{r+1}^{*2} = 0$. Since \eqref{eq:sign-consistency-condition} entails that 
\(\varepsilon \le \tfrac{1}{6\sqrt{5}} d_j^* \le d_1^*\),
the spectral-gap assumption yields that 
\begin{equation}
\label{eq:sign-consistency-uj}
\left\Vert s_j^u \mathbf{u}_j - \mathbf{u}_j^{*} \right\Vert_2
\leq \frac{3 \cdot 2^{3/2} d_1^{*} \varepsilon}{\tau d_j^{*2}}.
\end{equation}

Similarly, an application of Theorem~3 of~\cite{yu2015useful} %in the same way 
to $\RegM^\top \RegM$ and $\RegM^{*\top} \RegM^*$ gives a sign $s_j^v \in \{-1,+1\}$ such that
\begin{equation}
\label{eq:sign-consistency-vj}
\left\Vert s_j^v \mathbf{v}_j - \mathbf{v}_j^{*} \right\Vert_2
\leq \frac{3 \cdot 2^{3/2} d_1^{*} \varepsilon}{\tau d_j^{*2}}.
\end{equation}
Let us fix $k \neq j$. It follows from the orthogonality that 
\[
\mathbf{u}_k^\top \mathbf{u}_j = 0,
\qquad
\mathbf{v}_k^\top \mathbf{v}_j = 0,
\]
and thus,
\[
\left| \mathbf{u}_k^\top \mathbf{u}_j^* \right|
= \left| \mathbf{u}_k^\top ( \mathbf{u}_j^* - s_j^u \mathbf{u}_j ) \right|,
\qquad
\left| \mathbf{v}_k^\top \mathbf{v}_j^* \right|
= \left| \mathbf{v}_k^\top ( \mathbf{v}_j^* - s_j^v \mathbf{v}_j ) \right|.
\]
Consequently, we have that 
\begin{align*}
\sum_{k \neq j} d_k \left| \mathbf{u}_k^{\top} \mathbf{u}_j^{*} \right| \left| \mathbf{v}_k^{\top} \mathbf{v}_j^{*} \right|
&\leq \left\Vert \mathbf{u}_j^{*} - s_j^u \mathbf{u}_j \right\Vert_2
\cdot \left\Vert \mathbf{v}_j^{*} - s_j^v \mathbf{v}_j \right\Vert_2
\cdot \sum_{k=1}^{r} d_k. \numberthis \label{eq:sign-consistency-2}
\end{align*}

Combining \eqref{eq:sign-consistency-uj}, \eqref{eq:sign-consistency-vj}, and \eqref{eq:sign-consistency-2}, we can deduce that 
\begin{equation}
\label{eq:sign-consistency-3}
\sum_{k \neq j} d_k \left| \mathbf{u}_k^{\top} \mathbf{u}_j^{*} \right| \left| \mathbf{v}_k^{\top} \mathbf{v}_j^{*} \right|
\leq \frac{72\, d_1^{*2} \varepsilon^2}{\tau^2 d_j^{*4}} \cdot \sum_{k=1}^{r} d_k
\leq \frac{72\, r\, d_1^{*2} d_1 \varepsilon^2}{\tau^2 d_j^{*4}}.
\end{equation}
Moreover, it follows from \eqref{eq:sign-consistency-condition} that 
\begin{equation}
\label{eq:sign-consistency-4}
\varepsilon
\leq \frac{ \tau }{ 6 \sqrt{5} } \cdot \left( \frac{d_j^{*}}{\sqrt{r}} \right) \cdot \left( \frac{d_j^{*}}{d_1^{*}} \right)^{3/2}
\leq \frac{1}{6 \sqrt{5}} d_j^{*}
\leq \frac{1}{4} d_j^{*}.
\end{equation}
Hence, an application of Weyl's inequality leads to 
\begin{equation}
\label{eq:sign-consistency-5}
d_1 = \left\Vert \RegM \right\Vert_2
\leq \left\Vert \RegM^{*} \right\Vert_2 + \left\Vert \RegM - \RegM^{*} \right\Vert_2
\leq d_1^{*} + \frac{1}{4} d_j^{*}
\leq \frac{5}{4} d_1^{*}.
\end{equation}
Then substituting \eqref{eq:sign-consistency-4} and \eqref{eq:sign-consistency-5} into \eqref{eq:sign-consistency-3} yields that 
\begin{equation}
\label{eq:sign-consistency-6}
\sum_{k \neq j} d_k \left| \mathbf{u}_k^{\top} \mathbf{u}_j^{*} \right| \left| \mathbf{v}_k^{\top} \mathbf{v}_j^{*} \right|
\leq \frac{90\, r\, d_1^{*3} \varepsilon^2}{\tau^2 d_j^{*4}}.
\end{equation}

Further, in view of \eqref{eq:sign-consistency-condition}, it holds that 
\[
\varepsilon^2 \leq \frac{\tau^2 d_j^{*5}}{180 r\, d_1^{*3}}.
\]
Substituting this bound into \eqref{eq:sign-consistency-6}, we can obtain that 
\[
\sum_{k \neq j} d_k \left| \mathbf{u}_k^{\top} \mathbf{u}_j^{*} \right| \left| \mathbf{v}_k^{\top} \mathbf{v}_j^{*} \right|
\leq \frac{1}{2} d_j^{*}.
\]
Returning to \eqref{eq:sign-consistency-1}, this implies that 
\[
\mathbf{u}_j^{* \top} \left( \RegM^{*} - \RegM \right) \mathbf{v}_j^{*}
\geq \frac{1}{2} d_j^{*}.
\]
Then it follows that 
\begin{equation}
\label{eq:sign-consistency-final}
\left\Vert \RegM - \RegM^{*} \right\Vert_2
\geq \mathbf{u}_j^{* \top} \left( \RegM^{*} - \RegM \right) \mathbf{v}_j^{*}
\geq \frac{1}{2} d_j^{*},
\end{equation}
which contradicts \eqref{eq:sign-consistency-4}. Therefore, we have that 
\[
\left( \mathbf{u}_j^{\top} \mathbf{u}_j^{*} \right) \cdot \left( \mathbf{v}_j^{\top} \mathbf{v}_j^{*} \right) > 0.
\]

The final claim follows since the right-hand side of \eqref{eq:sign-consistency-condition} is minimized at $j=r$, as $d_j^* \ge d_r^*$ for each $1 \le j \le r$. This concludes the proof of Lemma \ref{lemma:sign-consistency}.
%\end{proof}

The lemma below complements the previous result by providing Frobenius-norm bounds for the perturbations of singular values and singular vectors under an eigen-gap condition and columnwise sign alignment.

\begin{lemma}[Perturbation bound under singular value separation]
\label{lemma:perturb-bound}
Let
\[
\RegM = \mathbf{U} \mathbf{D} \mathbf{V}^{\top}
= \sum_{k=1}^{r} d_k \mathbf{u}_k \mathbf{v}_k^{\top},
\qquad
\RegM^{*} = \mathbf{U}^{*} \mathbf{D}^{*} \mathbf{V}^{* \top}
= \sum_{k=1}^{r} d_k^* \mathbf{u}_k^* \mathbf{v}_k^{*\top}
\]
be rank-$r$ matrices, with singular values ordered as
\(d_1 \ge \cdots \ge d_r > 0\) and \(d_1^* \ge \cdots \ge d_r^* > 0\).
Assume that 
\begin{equation}
\label{eq:lemma-perturb-bound-cond}
\left\Vert \RegM - \RegM^{*} \right\Vert_2 \leq d^{*}_1,
\qquad
\left( \mathbf{u}_j^{\top} \mathbf{u}^{*}_j \right) \cdot \left( \mathbf{v}_j^{\top} \mathbf{v}^{*}_j \right) \geq 0
\quad \text{for all } 1 \leq j \leq r,
\end{equation}
and the squared singular values of $\RegM^*$ satisfy that 
\begin{equation*}
d^{*2}_{j-1} - d^{*2}_j \geq \tau d^{*2}_{j-1}
\quad \text{for all } 2 \leq j \leq r,
\end{equation*}
for some universal constant $0<\tau<1$. Then after possibly replacing each pair
$(\mathbf{u}_j^*, \mathbf{v}_j^*)$ with $(-\mathbf{u}_j^*, -\mathbf{v}_j^*)$ and relabeling the resulting singular vectors again by $(\mathbf{u}_j^*, \mathbf{v}_j^*)$, it holds that 
\begin{align}
\left\Vert \mathbf{D} - \mathbf{D}^{*} \right\Vert_{\mathrm{F}}
&\leq \left\Vert \RegM - \RegM^{*} \right\Vert_{\mathrm{F}}, \label{eq:regM-frob-bd-1} \\
\left\Vert \mathbf{U} \mathbf{D} - \mathbf{U}^{*} \mathbf{D}^{*} \right\Vert_{\mathrm{F}}
&\leq c \eta_r \left\Vert \RegM - \RegM^{*} \right\Vert_{\mathrm{F}}, \label{eq:regM-frob-bd-2} \\
\left\Vert \mathbf{V} \mathbf{D} - \mathbf{V}^{*} \mathbf{D}^{*} \right\Vert_{\mathrm{F}}
&\leq c \eta_r \left\Vert \RegM - \RegM^{*} \right\Vert_{\mathrm{F}}, \label{eq:regM-frob-bd-3}
\end{align}
where
\[
\eta_r = 1 + \frac{1}{\tau} \left( \sum_{j=1}^r \left( \frac{d^{*}_1}{d^{*}_j} \right)^2 \right)^{1/2}
\]
and $c > 0$ is a universal constant.
\end{lemma}

%\begin{proof}
\textit{Proof}. We begin by making explicit the sign convention that is implicit in the proof of Lemma~3 in the Supplementary Material of~\cite{uematsu2019sofar}. For each $1 \leq j \leq r$, condition~\eqref{eq:lemma-perturb-bound-cond} entails that
$\mathbf{u}_j^{\top} \mathbf{u}_j^*$ and $\mathbf{v}_j^{\top} \mathbf{v}_j^*$ have the same sign, unless one of them is zero. Then there exists some $s_j \in \{-1,1\}$ such that
\[
\mathbf{u}_j^{\top} (s_j \mathbf{u}_j^*) \geq 0
\ \text{ and } \ 
\mathbf{v}_j^{\top} (s_j \mathbf{v}_j^*) \geq 0.
\]
If we replace $(\mathbf{u}_j^*, \mathbf{v}_j^*)$ with $(s_j \mathbf{u}_j^*, s_j \mathbf{v}_j^*)$, it holds that 
\[
d_j^* (s_j \mathbf{u}_j^*) (s_j \mathbf{v}_j^*)^{\top}
= d_j^* \mathbf{u}_j^* \mathbf{v}_j^{*\top},
\]
so matrix $\RegM^*$ remains unchanged. After such replacement and relabeling the resulting singular vectors again by $(\mathbf{u}_j^*, \mathbf{v}_j^*)$, we can assume that
\[
\mathbf{u}_j^{\top} \mathbf{u}_j^* \geq 0
\text{ and }
\mathbf{v}_j^{\top} \mathbf{v}_j^* \geq 0
\ \text{ for all } 1 \leq j \leq r.
\]
With the sign convention fixed, bound~\eqref{eq:regM-frob-bd-1} can be established directly from the Wielandt--Hoffman inequality.

Under such sign alignment, the proof of Lemma~3 in the Supplementary Material of~\cite{uematsu2019sofar} applies %verbatim 
and yields that 
\[
\left\Vert \mathbf{U} \mathbf{D} - \mathbf{U}^{*} \mathbf{D}^{*} \right\Vert_{\mathrm{F}}
+
\left\Vert \mathbf{V} \mathbf{D} - \mathbf{V}^{*} \mathbf{D}^{*} \right\Vert_{\mathrm{F}}
\leq c \eta_r \left\Vert \RegM - \RegM^{*} \right\Vert_{\mathrm{F}}.
\]
Each term on the left-hand side of the expression above is therefore bounded by the same right-hand side, which establishes \eqref{eq:regM-frob-bd-2} and \eqref{eq:regM-frob-bd-3}. This completes the proof of Lemma \ref{lemma:perturb-bound}.
%\end{proof}

The following lemma quantifies the source approximation error induced by additive perturbation in the source matrix. Recall that we observe a noisy version of the source matrix
\begin{equation*}
\widetilde{\RegM}^{(0)} = \RegM^{(0)} + \mathbf{E}^{(0)},
\end{equation*}
where $\mathbf{E}^{(0)}$ denotes the perturbation matrix. Let the singular value decomposition (SVD) of $\widetilde{\RegM}^{(0)}$ be
\begin{equation*}
\widetilde{\RegM}^{(0)} = \widetilde{\mathbf{U}}^{(0)} \widetilde{\mathbf{D}}^{(0)} \widetilde{\mathbf{V}}^{(0)\top} = \sum_{j=1}^{\widetilde{r}_0} \widetilde{d}_j^{(0)} \widetilde{\mathbf{u}}_j^{(0)} \widetilde{\mathbf{v}}_j^{(0)\top}.
\end{equation*}
We then extend $\widetilde{\mathbf{U}}^{(0)}$ and $\widetilde{\mathbf{V}}^{(0)}$ to full orthonormal bases of $\mathbb{R}^p$ and $\mathbb{R}^q$, respectively, by appending the vectors $\widetilde{\mathbf{u}}^{(0)}_{\widetilde{r}_0+1}, \ldots, \widetilde{\mathbf{u}}^{(0)}_p$ and $\widetilde{\mathbf{v}}^{(0)}_{\widetilde{r}_0+1}, \ldots, \widetilde{\mathbf{v}}^{(0)}_q$.

For any integers $0 \leq r_u \leq p$ and $0 \leq r_v \leq q$, let us define the truncated singular subspaces and their orthogonal complements for the left and right singular vectors of $\widetilde{\RegM}^{(0)}$ as
\begin{equation*}
\widetilde{\mathbf{U}}^{\mathrm{s}}(r_u) \coloneqq [\widetilde{\mathbf{u}}^{(0)}_1, \ldots, \widetilde{\mathbf{u}}^{(0)}_{r_u}], \quad
\widetilde{\mathbf{U}}^{\mathrm{s}}_\perp(r_u) \coloneqq [\widetilde{\mathbf{u}}^{(0)}_{r_u+1}, \ldots, \widetilde{\mathbf{u}}^{(0)}_p],
\end{equation*}
respectively, and similarly, 
\begin{equation*}
\widetilde{\mathbf{V}}^{\mathrm{s}}(r_v) \coloneqq [\widetilde{\mathbf{v}}^{(0)}_1, \ldots, \widetilde{\mathbf{v}}^{(0)}_{r_v}], \quad
\widetilde{\mathbf{V}}^{\mathrm{s}}_\perp(r_v) \coloneqq [\widetilde{\mathbf{v}}^{(0)}_{r_v+1}, \ldots, \widetilde{\mathbf{v}}^{(0)}_q].
\end{equation*}
For later use, we also introduce the full source-augmented orthogonal basis matrices
\begin{equation*}
\widetilde{\mathbf{U}}_{\mathrm{full}} \coloneqq \begin{bmatrix} \widetilde{\mathbf{U}}^{\mathrm{s}} & \widetilde{\mathbf{U}}^{\mathrm{s}}_{\perp} \end{bmatrix} \in \mathbb{R}^{p \times p}, \qquad
\widetilde{\mathbf{V}}_{\mathrm{full}} \coloneqq \begin{bmatrix} \widetilde{\mathbf{V}}^{\mathrm{s}} & \widetilde{\mathbf{V}}^{\mathrm{s}}_{\perp} \end{bmatrix} \in \mathbb{R}^{q \times q}.
\end{equation*}

Let $\mathcal{S}_u$, $\mathcal{S}_u^c$, $\mathcal{S}_v$, and $\mathcal{S}_v^c$ be as defined in~\eqref{eq:def-support-u} and~\eqref{eq:def-support-v}, and recall  quantities $s_u(r_u)$ and $s_v(r_v)$ introduced in~\eqref{eq:similar-quant-def}. In addition, denote by
\begin{equation*}
\varepsilon^{(0)} \coloneqq \left\Vert \mathbf{E}^{(0)} \right\Vert_2
\end{equation*}
the spectral error of the source matrix. Throughout, we will also use the standard shorthand notation $x_{+} \coloneqq \max(x, 0)$ for any $x \in \mathbb{R}$.

We are now ready to state the lemma below, which quantifies the approximation error incurred when the noisy singular vectors of the source matrix are used in place of their population counterparts.

\begin{lemma}[Approximation error from noisy source singular vectors]
\label{lemma:approx-error}
Assume that Assumption~\ref{assump:relation-target-source} and Condition~\eqref{eq:eigen-gap-source} in Assumption~\ref{assump:eigen-gap} are satisfied. Then for all $0 \leq r_u \leq p$, we have that 
\begin{equation*}
\begin{aligned}
\left\vert \left[ \widetilde{\mathbf{U}}^{\mathrm{s} \top}_{\perp}(r_u)\, \mathbf{U}^* \right]_{\mathcal{S}_u^c} \right\vert_1
& \leq \frac{2^{3/2} \left[ r\, (r^*_u - r_u)_+ - s_u(r_u) \right]}{\tau \left( d_{r^*_u}^{(0)} \right)^2} \left( 2 d_1^{(0)} + \varepsilon^{(0)} \right) \varepsilon^{(0)} \\
& \quad + \frac{2^{5/4} \, r\, (r_u^*)^{1/4} \sqrt{p - \max(r^*_u, r_u)} }{\sqrt{\tau}\, d_{r^*_u}^{(0)}} \left( \varepsilon^{(0)} \right)^{1/2} \left( \varepsilon^{(0)} + 2 d_1^{(0)} \right)^{1/2}.
\end{aligned}
\end{equation*}
Similarly, for all $0 \leq r_v \leq q$, we have that 
\begin{equation*}
\begin{aligned}
\left\vert \left[ \widetilde{\mathbf{V}}^{\mathrm{s} \top}_{\perp}(r_v)\, \mathbf{V}^* \right]_{\mathcal{S}_v^c} \right\vert_1
& \leq \frac{2^{3/2} \left[ r\, (r^*_v - r_v)_+ - s_v(r_v) \right]}{\tau \left( d_{r^*_v}^{(0)} \right)^2} \left( 2 d_1^{(0)} + \varepsilon^{(0)} \right) \varepsilon^{(0)} \\
& \quad + \frac{2^{5/4} \, r\, (r_v^*)^{1/4} \sqrt{q - \max(r^*_v, r_v)} }{\sqrt{\tau}\, d_{r^*_v}^{(0)}} \left( \varepsilon^{(0)} \right)^{1/2} \left( \varepsilon^{(0)} + 2 d_1^{(0)} \right)^{1/2}.
\end{aligned}
\end{equation*}
\end{lemma}

%\begin{proof}
\textit{Proof}. We will prove the bound for the left singular vectors; the arguments for the right singular vectors are analogous. Let us write $\varepsilon_0 = \varepsilon^{(0)}$ and $s_u = s_u(r_u)$, abbreviate $\mathbf{U}^{\mathrm{s}}(r_u)$ and $\mathbf{U}^{\mathrm{s}}_{\perp}(r_u)$ as $\mathbf{U}^{\mathrm{s}}$ and $\mathbf{U}^{\mathrm{s}}_{\perp}$, and define $\widetilde{\mathbf{U}}^{\mathrm{s}}$ and $\widetilde{\mathbf{U}}^{\mathrm{s}}_{\perp}$ similarly. We also adopt the conventions of $d_0^{(0)}=\infty$ and $d_{r_0+1}^{(0)}=0$.

We first consider the case of $r_u \leq r_u^*$. Let us define
\begin{equation*}
\mathbf{U}_{\mathrm{gap}} = [\mathbf{u}^{(0)}_{r_u+1}, \ldots, \mathbf{u}^{(0)}_{r_u^*}], \qquad
\mathbf{U}_{\mathrm{tail}} = [\mathbf{u}^{(0)}_{r_u^*+1}, \ldots, \mathbf{u}^{(0)}_{p}],
\end{equation*}
and
\begin{equation*}
\widetilde{\mathbf{U}}_{\mathrm{gap}} = [\widetilde{\mathbf{u}}^{(0)}_{r_u+1}, \ldots, \widetilde{\mathbf{u}}^{(0)}_{r_u^*}], \qquad
\widetilde{\mathbf{U}}_{\mathrm{tail}} = [\widetilde{\mathbf{u}}^{(0)}_{r_u^*+1}, \ldots, \widetilde{\mathbf{u}}^{(0)}_{p}].
\end{equation*}
When $r_u=r_u^*$, the gap matrices are understood to be empty. Then it holds that 
\begin{equation*}
\mathbf{U}_{\perp}^{\mathrm{s}} = [\mathbf{U}_{\mathrm{gap}} \ \mathbf{U}_{\mathrm{tail}}],
\qquad
\widetilde{\mathbf{U}}_{\perp}^{\mathrm{s}} = [\widetilde{\mathbf{U}}_{\mathrm{gap}} \ \widetilde{\mathbf{U}}_{\mathrm{tail}}].
\end{equation*}

We further define
\begin{align*}
\mathcal{M}_u &\coloneqq \left\{ (j,i) \in [r_u^*-r_u] \times [r] : \left[ \mathbf{U}_{\mathrm{gap}}^{\top}\mathbf{U}^* \right]_{ji} \neq 0 \right\}, \\
\mathcal{M}_u^c &\coloneqq \left\{ (j,i) \in [r_u^*-r_u] \times [r] : \left[ \mathbf{U}_{\mathrm{gap}}^{\top}\mathbf{U}^* \right]_{ji} = 0 \right\}.
\end{align*}
Since Assumption~\ref{assump:relation-target-source} implies that $\mathbf{U}_{\mathrm{tail}}^{\top}\mathbf{U}^* = 0$, the support of $\mathbf{U}_{\perp}^{\mathrm{s}\top}\mathbf{U}^*$ coincides with $\mathcal{M}_u$ in the first $r_u^*-r_u$ rows and is empty in the remaining rows. Hence, it follows that 
\begin{equation*}
\left\vert \left[ \widetilde{\mathbf{U}}_{\perp}^{\mathrm{s}\top}\mathbf{U}^* \right]_{\mathcal{S}_u^c} \right\vert_1
=
\left\vert \left[ \widetilde{\mathbf{U}}_{\mathrm{gap}}^{\top}\mathbf{U}^* \right]_{\mathcal{M}_u^c} \right\vert_1
+
\left\vert \widetilde{\mathbf{U}}_{\mathrm{tail}}^{\top}\mathbf{U}^* \right\vert_1.
\end{equation*}

Let us first control the gap term. For any $(j,i)\in \mathcal{M}_u^c$, it holds that 
\[
\langle \mathbf{u}^{(0)}_{r_u+j}, \mathbf{u}^*_i \rangle = 0.
\]
Then for any sign $s_{r_u+j}^u \in \{-1,+1\}$, we have that 
\begin{equation*}
\begin{aligned}
\left\vert \left[ \widetilde{\mathbf{U}}_{\mathrm{gap}}^{\top}\mathbf{U}^* \right]_{ji} \right\vert
&= \left\vert \left\langle \widetilde{\mathbf{u}}^{(0)}_{r_u+j}, \mathbf{u}^*_i \right\rangle \right\vert \\
&= \left\vert \left\langle s_{r_u+j}^u \widetilde{\mathbf{u}}^{(0)}_{r_u+j} - \mathbf{u}^{(0)}_{r_u+j}, \mathbf{u}^*_i \right\rangle \right\vert \\
&\leq \left\Vert s_{r_u+j}^u \widetilde{\mathbf{u}}^{(0)}_{r_u+j} - \mathbf{u}^{(0)}_{r_u+j} \right\Vert_2.
\end{aligned}
\end{equation*}

We next observe that
\begin{equation*}
\begin{aligned}
\left\Vert \widetilde{\RegM}^{(0)}\widetilde{\RegM}^{(0)\top} - \RegM^{(0)}\RegM^{(0)\top} \right\Vert_2
&\leq \left\Vert \RegM^{(0)}\mathbf{E}^{(0)\top} \right\Vert_2
+ \left\Vert \mathbf{E}^{(0)}\RegM^{(0)\top} \right\Vert_2
+ \left\Vert \mathbf{E}^{(0)}\mathbf{E}^{(0)\top} \right\Vert_2 \\
&\leq \left( 2 d_1^{(0)} + \varepsilon_0 \right)\varepsilon_0.
\end{aligned}
\end{equation*}
Applying Theorem~3 of~\cite{yu2015useful} to the Gram matrices $\widetilde{\RegM}^{(0)}\widetilde{\RegM}^{(0)\top}$ and $\RegM^{(0)}\RegM^{(0)\top}$, we see that for each $k \in [r_u+1,r_u^*]$, there exists a sign $\tilde{s}_k^u \in \{-1,+1\}$ such that
\begin{equation*}
\left\Vert \tilde{s}_k^u \widetilde{\mathbf{u}}^{(0)}_k - \mathbf{u}^{(0)}_k \right\Vert_2
\leq
\frac{2^{3/2}\left( 2 d_1^{(0)} + \varepsilon_0 \right)\varepsilon_0}
{\min\left\{ (d_{k-1}^{(0)})^2 - (d_k^{(0)})^2,\ (d_k^{(0)})^2 - (d_{k+1}^{(0)})^2 \right\}}.
\end{equation*}
With the aid of Condition~\eqref{eq:eigen-gap-source}, we have that 
\begin{equation*}
\min\left\{ (d_{k-1}^{(0)})^2 - (d_k^{(0)})^2,\ (d_k^{(0)})^2 - (d_{k+1}^{(0)})^2 \right\}
\geq \tau (d_k^{(0)})^2
\geq \tau \left( d_{r_u^*}^{(0)} \right)^2.
\end{equation*}
Consequently, we can deduce that 
\begin{equation*}
\max_{(j,i)\in\mathcal{M}_u^c}
\left\vert \left[ \widetilde{\mathbf{U}}_{\mathrm{gap}}^{\top}\mathbf{U}^* \right]_{ji} \right\vert
\leq
\frac{2^{3/2}}{\tau \left( d_{r_u^*}^{(0)} \right)^2}
\left( 2 d_1^{(0)} + \varepsilon_0 \right)\varepsilon_0.
\end{equation*}
Since $|\mathcal{M}_u^c| = r(r_u^*-r_u) - s_u$, it follows that
\begin{equation*}
\left\vert \left[ \widetilde{\mathbf{U}}_{\mathrm{gap}}^{\top}\mathbf{U}^* \right]_{\mathcal{M}_u^c} \right\vert_1
\leq
\frac{2^{3/2}\left( r(r_u^*-r_u) - s_u \right)}{\tau \left( d_{r_u^*}^{(0)} \right)^2}
\left( 2 d_1^{(0)} + \varepsilon_0 \right)\varepsilon_0.
\end{equation*}

We now bound the tail term. Denote by 
\begin{equation*}
\widetilde{\mathbf{U}}_{\mathrm{oracle}} = [\widetilde{\mathbf{u}}^{(0)}_1,\ldots,\widetilde{\mathbf{u}}^{(0)}_{r_u^*}],
\qquad
\mathbf{U}_{\mathrm{oracle}} = [\mathbf{u}^{(0)}_1,\ldots,\mathbf{u}^{(0)}_{r_u^*}].
\end{equation*}
Since $\left\vert \widetilde{\mathbf{U}}_{\mathrm{tail}}^{\top}\mathbf{U}^* \right\vert_1 = \left\vert \mathbf{U}^{*\top}\widetilde{\mathbf{U}}_{\mathrm{tail}} \right\vert_1$, it follows from the orthogonality that 
\begin{equation*}
\left\Vert \mathbf{U}^{*\top}\widetilde{\mathbf{U}}_{\mathrm{tail}} \right\Vert_{\mathrm{F}}^2
+
\left\Vert \mathbf{U}^{*\top}\widetilde{\mathbf{U}}_{\mathrm{oracle}} \right\Vert_{\mathrm{F}}^2
=
\left\Vert \mathbf{U}^* \right\Vert_{\mathrm{F}}^2
=
r.
\end{equation*}
Hence, we can deduce that 
\begin{equation*}
\begin{aligned}
\left\vert \mathbf{U}^{*\top}\widetilde{\mathbf{U}}_{\mathrm{tail}} \right\vert_1^2
&\leq r(p-r_u^*) \left\Vert \mathbf{U}^{*\top}\widetilde{\mathbf{U}}_{\mathrm{tail}} \right\Vert_{\mathrm{F}}^2 \\
&= r(p-r_u^*) \left( r - \left\Vert \mathbf{U}^{*\top}\widetilde{\mathbf{U}}_{\mathrm{oracle}} \right\Vert_{\mathrm{F}}^2 \right) \\
&\leq 2 r^{3/2}(p-r_u^*)
\left( \sqrt{r} - \left\Vert \mathbf{U}^{*\top}\widetilde{\mathbf{U}}_{\mathrm{oracle}} \right\Vert_{\mathrm{F}} \right).
\end{aligned}
\end{equation*}

For any orthonormal matrix $\mathbf{O}_{\mathrm{oracle}} \in \mathbb{R}^{r_u^* \times r_u^*}$, it holds that 
\begin{equation*}
\begin{aligned}
\left\Vert \mathbf{U}^{*\top}\widetilde{\mathbf{U}}_{\mathrm{oracle}} \right\Vert_{\mathrm{F}}
&=
\left\Vert \mathbf{U}^{*\top}\widetilde{\mathbf{U}}_{\mathrm{oracle}}\mathbf{O}_{\mathrm{oracle}} \right\Vert_{\mathrm{F}} \\
&\geq
\left\Vert \mathbf{U}^{*\top}\mathbf{U}_{\mathrm{oracle}} \right\Vert_{\mathrm{F}}
-
\left\Vert \mathbf{U}^{*\top}\left( \widetilde{\mathbf{U}}_{\mathrm{oracle}}\mathbf{O}_{\mathrm{oracle}} - \mathbf{U}_{\mathrm{oracle}} \right) \right\Vert_{\mathrm{F}}.
\end{aligned}
\end{equation*}
Since Assumption~\ref{assump:relation-target-source} entails that  $\mathbf{U}^{*\top}\mathbf{U}_{\mathrm{tail}}=0$, we have that $\left\Vert \mathbf{U}^{*\top}\mathbf{U}_{\mathrm{oracle}} \right\Vert_{\mathrm{F}} = \sqrt{r}$. Consequently, it follows that 
\begin{equation*}
\sqrt{r} - \left\Vert \mathbf{U}^{*\top}\widetilde{\mathbf{U}}_{\mathrm{oracle}} \right\Vert_{\mathrm{F}}
\leq
\sqrt{r}\,
\left\Vert \widetilde{\mathbf{U}}_{\mathrm{oracle}}\mathbf{O}_{\mathrm{oracle}} - \mathbf{U}_{\mathrm{oracle}} \right\Vert_{\mathrm{F}}.
\end{equation*}
Substituting this into the previous expression, we can obtain that 
\begin{equation*}
\left\vert \mathbf{U}^{*\top}\widetilde{\mathbf{U}}_{\mathrm{tail}} \right\vert_1^2
\leq
2 r^2 (p-r_u^*)
\left\Vert \widetilde{\mathbf{U}}_{\mathrm{oracle}}\mathbf{O}_{\mathrm{oracle}} - \mathbf{U}_{\mathrm{oracle}} \right\Vert_{\mathrm{F}}.
\end{equation*}

Let us define
\begin{equation*}
\Delta_u^* \coloneqq \left( d_{r_u^*}^{(0)} \right)^2 - \left( d_{r_u^*+1}^{(0)} \right)^2.
\end{equation*}
Applying Theorem~3 of~\cite{yu2015useful} once more to the leading $r_u^*$-dimensional eigenspaces of $\widetilde{\RegM}^{(0)}\widetilde{\RegM}^{(0)\top}$ and $\RegM^{(0)}\RegM^{(0)\top}$, there exists an orthonormal matrix $\widetilde{\mathbf{O}}_{\mathrm{oracle}}$ such that
\begin{equation*}
\left\Vert \widetilde{\mathbf{U}}_{\mathrm{oracle}}\widetilde{\mathbf{O}}_{\mathrm{oracle}} - \mathbf{U}_{\mathrm{oracle}} \right\Vert_{\mathrm{F}}
\leq
\frac{2\sqrt{2r_u^*}}{\Delta_u^*}
\left( 2 d_1^{(0)} + \varepsilon_0 \right)\varepsilon_0.
\end{equation*}
It follows from Condition~\eqref{eq:eigen-gap-source} that $\Delta_u^* \geq \tau \left( d_{r_u^*}^{(0)} \right)^2$. Then it holds that 
\begin{equation*}
\left\vert \widetilde{\mathbf{U}}_{\mathrm{tail}}^{\top}\mathbf{U}^* \right\vert_1
=
\left\vert \mathbf{U}^{*\top}\widetilde{\mathbf{U}}_{\mathrm{tail}} \right\vert_1
\leq
\frac{2^{5/4} r (r_u^*)^{1/4} \sqrt{p-r_u^*}}{\sqrt{\tau}\, d_{r_u^*}^{(0)}}
\varepsilon_0^{1/2}
\left( 2 d_1^{(0)} + \varepsilon_0 \right)^{1/2}.
\end{equation*}
Thus, combining the bounds for the gap and tail terms yields that 
\begin{equation*}
\begin{aligned}
\left\vert \left[ \widetilde{\mathbf{U}}_{\perp}^{\mathrm{s}\top}\mathbf{U}^* \right]_{\mathcal{S}_u^c} \right\vert_1
&\leq
\frac{2^{3/2}\left( r(r_u^*-r_u)-s_u \right)}{\tau \left( d_{r_u^*}^{(0)} \right)^2}
\left( 2 d_1^{(0)} + \varepsilon_0 \right)\varepsilon_0 \\
&\quad +
\frac{2^{5/4} r (r_u^*)^{1/4} \sqrt{p-r_u^*}}{\sqrt{\tau}\, d_{r_u^*}^{(0)}}
\varepsilon_0^{1/2}
\left( 2 d_1^{(0)} + \varepsilon_0 \right)^{1/2}.
\end{aligned}
\end{equation*}

It remains to consider the case of $r_u > r_u^*$. Then we have that  $\mathcal{S}_u=\varnothing$ and $s_u(r_u)=0$. Denote by 
\begin{equation*}
\widetilde{\mathbf{U}}_{\mathrm{tail}} = [\widetilde{\mathbf{u}}^{(0)}_{r_u+1}, \ldots, \widetilde{\mathbf{u}}^{(0)}_p].
\end{equation*}
Since $\widetilde{\mathbf{U}}_{\perp}^{\mathrm{s}} = \widetilde{\mathbf{U}}_{\mathrm{tail}}$, we can deduce that 
\begin{equation*}
\left\vert \left[ \widetilde{\mathbf{U}}_{\perp}^{\mathrm{s}\top}\mathbf{U}^* \right]_{\mathcal{S}_u^c} \right\vert_1
=
\left\vert \widetilde{\mathbf{U}}_{\mathrm{tail}}^{\top}\mathbf{U}^* \right\vert_1
=
\left\vert \mathbf{U}^{*\top}\widetilde{\mathbf{U}}_{\mathrm{tail}} \right\vert_1
\leq
\sqrt{r(p-r_u)}\,
\left\Vert \mathbf{U}^{*\top}\widetilde{\mathbf{U}}_{\mathrm{tail}} \right\Vert_{\mathrm{F}}.
\end{equation*}
Moreover, it holds that 
\begin{equation*}
\left\Vert \mathbf{U}^{*\top}\widetilde{\mathbf{U}}_{\mathrm{tail}} \right\Vert_{\mathrm{F}}^2
\leq
r - \left\Vert \mathbf{U}^{*\top}\widetilde{\mathbf{U}}_{\mathrm{oracle}} \right\Vert_{\mathrm{F}}^2.
\end{equation*}
Using the same argument as above for $\sqrt{r} - \left\Vert \mathbf{U}^{*\top}\widetilde{\mathbf{U}}_{\mathrm{oracle}} \right\Vert_{\mathrm{F}}$, we can show that
\begin{equation*}
\left\vert \left[ \widetilde{\mathbf{U}}_{\perp}^{\mathrm{s}\top}\mathbf{U}^* \right]_{\mathcal{S}_u^c} \right\vert_1
\leq
\frac{2^{5/4} r (r_u^*)^{1/4} \sqrt{p-r_u}}{\sqrt{\tau}\, d_{r_u^*}^{(0)}}
\varepsilon_0^{1/2}
\left( 2 d_1^{(0)} + \varepsilon_0 \right)^{1/2}.
\end{equation*}

Since $s_u(r_u)=0$ whenever $r_u \geq r_u^*$, combining the above two cases gives that 
\begin{equation*}
\begin{aligned}
\left\vert \left[ \widetilde{\mathbf{U}}_{\perp}^{\mathrm{s}\top}\mathbf{U}^* \right]_{\mathcal{S}_u^c} \right\vert_1
&\leq
\frac{2^{3/2}\left[ r(r_u^*-r_u)_+ - s_u(r_u) \right]}{\tau \left( d_{r_u^*}^{(0)} \right)^2}
\left( 2 d_1^{(0)} + \varepsilon_0 \right)\varepsilon_0 \\
&\quad +
\frac{2^{5/4} r (r_u^*)^{1/4} \sqrt{p-\max(r_u,r_u^*)}}{\sqrt{\tau}\, d_{r_u^*}^{(0)}}
\varepsilon_0^{1/2}
\left( 2 d_1^{(0)} + \varepsilon_0 \right)^{1/2},
\end{aligned}
\end{equation*}
which is exactly the desired bound for the left singular vectors. The right singular vector bound can be established using the same argument after replacing $(p,r_u,r_u^*,s_u,\mathbf{U})$ with $(q,r_v,r_v^*,s_v,\mathbf{V})$. This concludes the proof of Lemma \ref{lemma:approx-error}.
%\end{proof}

The following lemma specifies a set of desirable events that occur with high probability. For the rest of the analysis, we work under the assumption that these events hold in order to establish the central deterministic bounds.

\begin{lemma}[High-probability %good 
events]
\label{lemma:good-event}
Assume that Assumptions~\ref{assump:design-RSC}--\ref{assump:error-distribution} are satisfied. Let $\widehat{\RegM} = \widehat{\mathbf{U}} \widehat{\mathbf{D}} \widehat{\mathbf{V}}^{\top}$ be the localized Stage 2 estimator satisfying \eqref{eq:opt-constraints}, and assume that \eqref{eq:require-initial}, \eqref{eq:cond-lasso-l2-err}, and \eqref{eq:cond-O2} are satisfied. Fix any $0 < \delta < 1$ and define the regularization parameters $\lambda_u$ and $\lambda_v$ as in~\eqref{eq:lambdas-in-good-events}. Then with probability at least $1 - \delta$, it holds simultaneously that 
\begin{align}
\left\vert \frac{1}{n} \widetilde{\mathbf{U}}_{\mathrm{full}}^{\top} \mathbf{X}^{\top} \mathbf{E} \widehat{\mathbf{B}} \widehat{\mathbf{D}}^{-} \right\vert_{\infty} & \leq \lambda_u / 2, \label{eq:lambdau-cond-lemma-new} \\
\left\vert \frac{1}{n} \left( \mathbf{X} \mathbf{U}^* \right)^{\top} \mathbf{E} \widehat{\mathbf{B}} \widehat{\mathbf{D}}^{-} \right\vert_{\infty} & \leq \lambda_u / 2, \label{eq:lambdad-cond-lemma-new} \\
\left\vert \frac{1}{n} \left( \mathbf{X} \mathbf{U}^* \right)^{\top} \mathbf{E} \widetilde{\mathbf{V}}_{\mathrm{full}} \right\vert_{\infty} & \leq \lambda_v / 2. \label{eq:lambdav-cond-lemma-new}
\end{align}
\end{lemma}

%\begin{proof}
\textit{Proof}. Denote by 
\[
\widetilde{\mathbf{X}} = \mathbf{X}\widetilde{\mathbf{U}}_{\mathrm{full}},
\qquad
\widetilde{\mathbf{E}} = \mathbf{E}\widetilde{\mathbf{V}}_{\mathrm{full}},
\qquad
\mathbf{X}_* = \mathbf{X}\mathbf{U}^*,
\qquad
\mathbf{E}_* = \mathbf{E}\mathbf{V}^*.
\]
Let $\widetilde{\mathbf{x}}^{(j)}$ and $\widetilde{\mathbf{e}}^{(j)}$ be the $j$th columns of $\widetilde{\mathbf{X}}$ and $\widetilde{\mathbf{E}}$, respectively, and  $\mathbf{x}_*^{(j)}$ and $\mathbf{e}_*^{(j)}$ the corresponding columns of $\mathbf{X}_*$ and $\mathbf{E}_*$. Since $\widehat{\RegM}$ satisfies \eqref{eq:opt-constraints}, it follows from Condition~\eqref{eq:require-initial} that 
\begin{equation}
\label{eq:Delta-bd-3R}
\left\Vert \widehat{\bm{\Delta}} \right\Vert_{\mathrm{F}} \leq 3 R_{2,n}.
\end{equation}
Observe that the Stage 2 objective and constraints are invariant under simultaneous permutations of the columns of $(\widehat{\mathbf{U}}, \widehat{\mathbf{V}})$ and of the diagonal entries of $\widehat{\mathbf{D}}$. We can thus choose a representative of the optimizer for which
\[
\widehat{d}_1 \geq \cdots \geq \widehat{d}_r \geq 0.
\]

In view of \eqref{eq:cond-lasso-l2-err}, it holds that 
\begin{equation}
\label{eq:delta-bd-good}
\left\Vert \widehat{\bm{\Delta}} \right\Vert_2
\leq
\left\Vert \widehat{\bm{\Delta}} \right\Vert_{\mathrm{F}}
\leq
3 R_{2,n}
\leq
\frac{\tau}{6 \sqrt{5}} \cdot \frac{d_r^*}{\sqrt{r}} \cdot \left( \frac{d_r^*}{d_1^*} \right)^{3/2}.
\end{equation}
In particular, the right-hand side of the expression above is at most $d_r^* / 4$. An application of Weyl's inequality gives that for each $1 \leq j \leq r$,
\[
\left| \widehat{d}_j - d_j^* \right|
\leq
\left\Vert \widehat{\bm{\Delta}} \right\Vert_2
\leq
\frac{d_r^*}{4}
\leq
\frac{d_j^*}{4}.
\]
Consequently, we have that $\widehat{d}_j \geq 3 d_j^* / 4 > 0$, and thus
\begin{equation}
\label{eq:D-inverse-bd}
\left| \widehat{\mathbf{D}}^- \right|_{\infty} \leq \frac{2}{d_{\min}},
\qquad
\left| \mathbf{D}^* \widehat{\mathbf{D}}^- \right|_{\infty} \leq 2.
\end{equation}

At this point, the assumptions of Lemma~\ref{lemma:sign-consistency} (with $j=r$) and Lemma~\ref{lemma:perturb-bound} are satisfied. Then it follows that 
\begin{equation}
\label{eq:abd-bd}
\left\Vert \widehat{\bm{\Delta}}^a \right\Vert_{\mathrm{F}} \leq c \eta_r \left\Vert \widehat{\bm{\Delta}} \right\Vert_{\mathrm{F}},
\qquad
\left\Vert \widehat{\bm{\Delta}}^b \right\Vert_{\mathrm{F}} \leq c \eta_r \left\Vert \widehat{\bm{\Delta}} \right\Vert_{\mathrm{F}},
\qquad
\left\Vert \widehat{\bm{\Delta}}^d \right\Vert_{\mathrm{F}} \leq \left\Vert \widehat{\bm{\Delta}} \right\Vert_{\mathrm{F}},
\end{equation}
where $c$ is the universal constant from Lemma~\ref{lemma:perturb-bound}, and hence the same constant as in \eqref{eq:cond-O2}. By a union bound, it suffices to show that each of \eqref{eq:lambdau-cond-lemma-new}--\eqref{eq:lambdav-cond-lemma-new} holds with probability at least $1 - \delta / 3$. We will verify these bounds one by one.

We first prove \eqref{eq:lambdau-cond-lemma-new}. Let us write
\[
\left| \frac{1}{n} \widetilde{\mathbf{U}}_{\mathrm{full}}^{\top} \mathbf{X}^{\top} \mathbf{E} \widehat{\mathbf{B}} \widehat{\mathbf{D}}^- \right|_{\infty}
\leq
T_{u,1} + T_{u,2},
\]
where
\[
T_{u,1}
=
\left| \frac{1}{n} \widetilde{\mathbf{X}}^{\top} \mathbf{E} \mathbf{B}^* \widehat{\mathbf{D}}^- \right|_{\infty},
\qquad
T_{u,2}
=
\left| \frac{1}{n} \widetilde{\mathbf{X}}^{\top} \mathbf{E} \widehat{\bm{\Delta}}^b \widehat{\mathbf{D}}^- \right|_{\infty}.
\]
We begin with $T_{u,1}$. In light of \eqref{eq:D-inverse-bd}, it holds that 
\[
T_{u,1}
=
\left| \frac{1}{n} \widetilde{\mathbf{X}}^{\top} \mathbf{E}_* \mathbf{D}^* \widehat{\mathbf{D}}^- \right|_{\infty}
\leq
\frac{2}{n} \left| \widetilde{\mathbf{X}}^{\top} \mathbf{E}_* \right|_{\infty}.
\]
For each $1 \leq i \leq p$ and $1 \leq j \leq r$, scalar $\widetilde{\mathbf{x}}^{(i)\top}\mathbf{e}_*^{(j)}$ is Gaussian with mean zero and variance
\[
\mathbb{E} \left[ \left( \widetilde{\mathbf{x}}^{(i)\top}\mathbf{e}_*^{(j)} \right)^2 \right]
=
\left( \mathbf{v}_j^{*\top}\bm{\Sigma}\mathbf{v}_j^* \right) \left\Vert \widetilde{\mathbf{x}}^{(i)} \right\Vert_2^2
\leq
\alpha_{\max} \left\Vert \mathbf{X}\widetilde{\mathbf{u}}_i^{(0)} \right\Vert_2^2
\leq
\nu \alpha_{\max} n.
\]
As a result, we have that 
\[
\mathbb{P} \left( T_{u,1} > \frac{\lambda_u}{4} \right)
\leq
2pr \exp \left( - \frac{n \lambda_u^2}{128 \nu \alpha_{\max}} \right)
\leq
\frac{\delta}{6}.
\]

We next analyze term $T_{u,2}$. By resorting to Lemma~\ref{lemma:smart-prod-max}, \eqref{eq:D-inverse-bd}, and \eqref{eq:abd-bd}, we can deduce that 
\begin{multline*}
T_{u,2}
\leq
\left(
\max_{1 \leq j \leq p}
\left\Vert \frac{1}{n}\mathbf{E}^{\top}\widetilde{\mathbf{x}}^{(j)} \right\Vert_2
\right)
\left(
\max_{1 \leq k \leq r}
\left\Vert \left( \widehat{\bm{\Delta}}^b \widehat{\mathbf{D}}^- \right)_{:k} \right\Vert_2
\right)
\\
\leq
\frac{2}{d_{\min}}
\left(
\max_{1 \leq j \leq p}
\left\Vert \frac{1}{n}\mathbf{E}^{\top}\widetilde{\mathbf{x}}^{(j)} \right\Vert_2
\right)
\left(
\max_{1 \leq k \leq r}
\left\Vert \widehat{\bm{\Delta}}^b_{:k} \right\Vert_2
\right)
\leq
\frac{6 c \eta_r R_{2,n}}{d_{\min}} M_u,
\end{multline*}
where
\[
M_u
=
\max_{1 \leq j \leq p}
\left\Vert \frac{1}{n}\mathbf{E}^{\top}\widetilde{\mathbf{x}}^{(j)} \right\Vert_2.
\]
For each $j$, if $\beta_j = \left\Vert \widetilde{\mathbf{x}}^{(j)} \right\Vert_2^2 / n$, we have that 
\[
\frac{1}{n}\mathbf{E}^{\top}\widetilde{\mathbf{x}}^{(j)}
\sim
\mathcal{N} \left( 0, \frac{\beta_j}{n}\bm{\Sigma} \right),
\]
and Assumption~\ref{assump:constrained-norm-design}(i) entails that $\beta_j \leq \nu$. Hence, it follows from Lemma~\ref{lemma:gaussian-vector-norm} and a union bound that with probability at least $1 - \delta / 6$,
\[
M_u^2
\leq
\frac{\nu}{n}
\left[
\tr(\bm{\Sigma})
+
c_1 \max\left\{
\sqrt{\frac{\tr(\bm{\Sigma}^2)\log(6p/\delta)}{c_2}},
\frac{\alpha_{\max}\log(6p/\delta)}{c_2}
\right\}
\right].
\]
Consequently, on the same event, it holds that 
\[
T_{u,2}^2
\leq
\frac{36 c^2 \eta_r^2 R_{2,n}^2 \nu}{d_{\min}^2 n}
\left[
\tr(\bm{\Sigma})
+
c_1 \max\left\{
\sqrt{\frac{\tr(\bm{\Sigma}^2)\log(6p/\delta)}{c_2}},
\frac{\alpha_{\max}\log(6p/\delta)}{c_2}
\right\}
\right].
\]
Note that Condition~\eqref{eq:cond-O2} bounds the right-hand side of the expression above by
\[
\frac{8 \nu \alpha_{\max}\log(12pr/\delta)}{n}
=
\frac{\lambda_u^2}{16}.
\]
Thus, we have that $T_{u,2} \leq \lambda_u / 4$ with probability at least $1 - \delta / 6$, which establishes \eqref{eq:lambdau-cond-lemma-new} with probability at least $1 - \delta / 3$.

We next aim to show \eqref{eq:lambdad-cond-lemma-new}. Observe that 
\[
\left| \frac{1}{n} \left( \mathbf{X}\mathbf{U}^* \right)^{\top} \mathbf{E} \widehat{\mathbf{B}} \widehat{\mathbf{D}}^- \right|_{\infty}
\leq
T_{d,1} + T_{d,2},
\]
where
\[
T_{d,1}
=
\left| \frac{1}{n} \mathbf{X}_*^{\top} \mathbf{E} \mathbf{B}^* \widehat{\mathbf{D}}^- \right|_{\infty},
\qquad
T_{d,2}
=
\left| \frac{1}{n} \mathbf{X}_*^{\top} \mathbf{E} \widehat{\bm{\Delta}}^b \widehat{\mathbf{D}}^- \right|_{\infty}.
\]
Similar as before, we can show that 
\[
T_{d,1} \leq \frac{2}{n} \left| \mathbf{X}_*^{\top}\mathbf{E}_* \right|_{\infty},
\]
and each entry of $\mathbf{X}_*^{\top}\mathbf{E}_*$ is Gaussian with variance at most $\nu \alpha_{\max} n$. Consequently, it holds that 
\[
\mathbb{P} \left( T_{d,1} > \frac{\lambda_u}{4} \right)
\leq
2r^2 \exp \left( - \frac{n \lambda_u^2}{128 \nu \alpha_{\max}} \right)
\leq
\frac{\delta}{6},
\]
since $r \leq p$.

For the second term above, it holds that 
\[
T_{d,2}
\leq
\frac{2}{d_{\min}}
\left(
\max_{1 \leq j \leq r}
\left\Vert \frac{1}{n}\mathbf{E}^{\top}\mathbf{x}_*^{(j)} \right\Vert_2
\right)
\left(
\max_{1 \leq k \leq r}
\left\Vert \widehat{\bm{\Delta}}^b_{:k} \right\Vert_2
\right)
\leq
\frac{6 c \eta_r R_{2,n}}{d_{\min}} M_{u,*},
\]
where
\[
M_{u,*}
=
\max_{1 \leq j \leq r}
\left\Vert \frac{1}{n}\mathbf{E}^{\top}\mathbf{x}_*^{(j)} \right\Vert_2.
\]
Using the same argument as above, with $r$ in place of $p$, we have that with probability at least $1 - \delta / 6$,
\[
M_{u,*}^2
\leq
\frac{\nu}{n}
\left[
\tr(\bm{\Sigma})
+
c_1 \max\left\{
\sqrt{\frac{\tr(\bm{\Sigma}^2)\log(6r/\delta)}{c_2}},
\frac{\alpha_{\max}\log(6r/\delta)}{c_2}
\right\}
\right].
\]
Since $r \leq p$, this bound is no larger than the corresponding bound for $M_u$, and Condition~\eqref{eq:cond-O2} again yields that $T_{d,2} \leq \lambda_u / 4$ with probability at least $1 - \delta / 6$. Hence, \eqref{eq:lambdad-cond-lemma-new} holds with probability at least $1 - \delta / 3$.

Finally, we establish \eqref{eq:lambdav-cond-lemma-new}. Let $\widetilde{\mathbf{v}}_j$ be the $j$th column of $\widetilde{\mathbf{V}}_{\mathrm{full}}$. For each $1 \leq i \leq r$ and $1 \leq j \leq q$, scalar $\mathbf{x}_*^{(i)\top}\widetilde{\mathbf{e}}^{(j)}$ is Gaussian with mean zero and variance
\[
\mathbb{E} \left[ \left( \mathbf{x}_*^{(i)\top}\widetilde{\mathbf{e}}^{(j)} \right)^2 \right]
=
\left( \widetilde{\mathbf{v}}_j^{\top}\bm{\Sigma}\widetilde{\mathbf{v}}_j \right) \left\Vert \mathbf{x}_*^{(i)} \right\Vert_2^2
\leq
\alpha_{\max} \left\Vert \mathbf{X}\mathbf{u}_i^* \right\Vert_2^2
\leq
\nu \alpha_{\max} n.
\]
Therefore, we can obtain that 
\[
\mathbb{P}
\left(
\left| \frac{1}{n} \left( \mathbf{X}\mathbf{U}^* \right)^{\top} \mathbf{E} \widetilde{\mathbf{V}}_{\mathrm{full}} \right|_{\infty}
>
\frac{\lambda_v}{2}
\right)
\leq
2qr \exp \left( - \frac{n \lambda_v^2}{8 \nu \alpha_{\max}} \right)
\leq
\frac{\delta}{3}.
\]
Combining the three bounds above completes the proof of Lemma \ref{lemma:good-event}.
%\end{proof}

The following two auxiliary lemmas are self-contained and will simplify several steps in the theoretical analysis that follows.

\begin{lemma}
\label{lemma:smart-prod-max}
Let $\mathbf{A} \in \mathbb{R}^{m \times n}$ and $\mathbf{B} \in \mathbb{R}^{n \times q}$ be arbitrary matrices. Denote the $i$th row of $\mathbf{A}$ as $\mathbf{A}_{i:}$ and the $j$th column of $\mathbf{B}$ as $\mathbf{B}_{:j}$. Then it holds that 
\begin{equation*}
\left\vert \mathbf{A} \mathbf{B} \right\vert_{\infty} \leq \left( \max_{1 \leq i \leq m} \left\Vert \mathbf{A}_{i:} \right\Vert_2 \right) \left( \max_{1 \leq j \leq q} \left\Vert \mathbf{B}_{:j} \right\Vert_2 \right).
\end{equation*}
\end{lemma}

%\begin{proof}
\textit{Proof}. 
For any $1 \leq i \leq m$ and $1 \leq j \leq q$, the $(i,j)$th entry of $\mathbf{A}\mathbf{B}$ satisfies that 
\begin{equation*}
\left| \left[ \mathbf{A} \mathbf{B} \right]_{ij} \right|
= \left| \sum_{k=1}^n \mathbf{A}_{ik} \mathbf{B}_{kj} \right|
= \left| \langle \mathbf{A}_{i:}, \mathbf{B}_{:j} \rangle \right|
\leq \left\Vert \mathbf{A}_{i:} \right\Vert_2 \left\Vert \mathbf{B}_{:j} \right\Vert_2,
\end{equation*}
where the inequality is an immediate consequence of the Cauchy--Schwarz inequality. Taking the maximum over all $i$ and $j$ establishes the desired claim. This concludes the proof of Lemma \ref{lemma:smart-prod-max}.
%\end{proof}

\begin{lemma}[Gaussian vector norm concentration]
\label{lemma:gaussian-vector-norm}
Let $\mathbf{g} \sim \mathcal{N}(0,\bm{\Sigma})$ be a $d$-dimensional Gaussian random vector. Then we have that for each $t \geq 0$,
\begin{equation*}
\mathbb{P} \left\{ \left\Vert \mathbf{g} \right\Vert^2_2 - \operatorname{tr} \left( \bm{\Sigma} \right) \geq t \right\}
\leq \exp \left\{ - c_2 \min \left( \frac{t^2}{c_1^2\, \operatorname{tr} \left( \bm{\Sigma}^2 \right)} , \frac{t}{c_1\, \lambda_{\max} \left( \bm{\Sigma} \right)} \right) \right\},
\end{equation*}
where \(c_1, c_2 > 0\) are universal constants.
\end{lemma}

%\begin{proof}
\textit{Proof}. 
Let $\bm{\Sigma} = \mathbf{U} \bm{\Lambda} \mathbf{U}^{\top}$ be the eigen-decomposition of $\bm{\Sigma}$, where $\mathbf{U} \in \mathbb{R}^{d \times d}$ is orthonormal and $\bm{\Lambda} = \operatorname{diag}(\lambda_1, \ldots, \lambda_d)$ with $\lambda_j \geq 0$. Let $\mathbf{z} = (Z_1,\ldots,Z_d)^\top \sim \mathcal{N}(0,\mathbf{I}_d)$. Then it holds that 
\[
\mathbf{g} \overset{d}{=} \mathbf{U} \bm{\Lambda}^{1/2} \mathbf{z}.
\]
Consequently, we have that
\begin{equation*}
\left\Vert \mathbf{g} \right\Vert_2^2
\overset{d}{=} \left\Vert \bm{\Lambda}^{1/2} \mathbf{z} \right\Vert_2^2
= \sum_{j=1}^d \lambda_j Z_j^2,
\end{equation*}
and thus
\begin{equation}
\label{eq:proof-gaussian-vector-norm-1}
\left\Vert \mathbf{g} \right\Vert_2^2 - \operatorname{tr}(\bm{\Sigma})
\overset{d}{=} \sum_{j=1}^d \lambda_j (Z_j^2 - 1).
\end{equation}

By invoking Lemma 2.7.4 of~\cite{vershynin2018high}, random variables $Z_j^2 - 1$ are independent, mean-zero, and sub-exponential, with $\|Z_j^2 - 1\|_{\psi_1} \leq K$ for some universal constant $K>0$. Then an application of Theorem 2.8.3 from the same reference to the weighted sum in~\eqref{eq:proof-gaussian-vector-norm-1} gives that for all $t \geq 0$,
\begin{equation}
\label{eq:proof-gaussian-vector-norm-2}
\mathbb{P} \left\{ \sum_{j=1}^d \lambda_j (Z_j^2 - 1) \geq t \right\}
\leq \exp \left\{ - c_2 \min \left( \frac{t^2}{K^2 \sum_{j=1}^d \lambda_j^2}, \frac{t}{K \max_j \lambda_j} \right) \right\},
\end{equation}
where $c_2 > 0$ is a universal constant. Therefore, since $\operatorname{tr}(\bm{\Sigma}^2) = \sum_{j=1}^d \lambda_j^2$ and $\lambda_{\max}(\bm{\Sigma}) = \max_j \lambda_j$, the claimed bound follows from \eqref{eq:proof-gaussian-vector-norm-1} and \eqref{eq:proof-gaussian-vector-norm-2} after absorbing $K$ into the universal constant $c_1$. This completes the proof of Lemma \ref{lemma:gaussian-vector-norm}.
%\end{proof}

\subsection{Proof of Theorem~\ref{thm:spectral-trans}}
\label{sec:proof-upp-bd-smart}

To simplify the notation, let us write $\mathbf{U}^{\mathrm{s}}(r_u)$ and $\mathbf{U}^{\mathrm{s}}_{\perp}(r_u)$ as $\mathbf{U}^{\mathrm{s}}$ and $\mathbf{U}^{\mathrm{s}}_{\perp}$, respectively. Similarly, we write $\mathbf{V}^{\mathrm{s}}(r_v)$ and $\mathbf{V}^{\mathrm{s}}_{\perp}(r_v)$ as $\mathbf{V}^{\mathrm{s}}$ and $\mathbf{V}^{\mathrm{s}}_{\perp}$. We also define $\widetilde{\mathbf{U}}^{\mathrm{s}}$, $\widetilde{\mathbf{U}}^{\mathrm{s}}_{\perp}$, $\widetilde{\mathbf{V}}^{\mathrm{s}}$, and $\widetilde{\mathbf{V}}^{\mathrm{s}}_{\perp}$ analogously. Throughout the remainder of the proof, we will abbreviate $s_u(r_u)$ and $s_v(r_v)$ as $s_u$ and $s_v$, respectively.

We start with defining the supports associated with the source-truncated orthogonal complements
\begin{equation}
\label{eq:def-support-u}
\begin{aligned}
\mathcal{S}_u & \coloneqq \left\{ (i, j) : \left[ \mathbf{U}^{\mathrm{s} \top}_{\perp} \mathbf{U}^* \right]_{ij} \neq 0 \right\} \subseteq [p - r_u] \times [r], \\
\mathcal{S}_u^c & \coloneqq \left\{ (i, j) : \left[ \mathbf{U}^{\mathrm{s} \top}_{\perp} \mathbf{U}^* \right]_{ij} = 0 \right\} \subseteq [p - r_u] \times [r].
\end{aligned}
\end{equation}
From the definition of $s_u$, we see that $\vert \mathcal{S}_u \vert = s_u$ and $\vert \mathcal{S}_u^c \vert = r(p - r_u) - s_u$.
Similarly, in the $v$-direction denote by 
\begin{equation}
\label{eq:def-support-v}
\begin{aligned}
\mathcal{S}_v & \coloneqq \left\{ (i, j) : \left[ \mathbf{V}^{\mathrm{s} \top}_{\perp} \mathbf{V}^* \right]_{ij} \neq 0 \right\} \subseteq [q - r_v] \times [r], \\
\mathcal{S}_v^c & \coloneqq \left\{ (i, j) : \left[ \mathbf{V}^{\mathrm{s} \top}_{\perp} \mathbf{V}^* \right]_{ij} = 0 \right\} \subseteq [q - r_v] \times [r],
\end{aligned}
\end{equation}
so that $\vert \mathcal{S}_v \vert = s_v$ and $\vert \mathcal{S}_v^c \vert = r(q - r_v) - s_v$.
For convenience, we also introduce the full source-augmented orthogonal basis matrices
\begin{equation*}
\widetilde{\mathbf{U}}_{\mathrm{full}} \coloneqq \begin{bmatrix} \widetilde{\mathbf{U}}^{\mathrm{s}} & \widetilde{\mathbf{U}}^{\mathrm{s}}_{\perp} \end{bmatrix} \in \mathbb{R}^{p \times p}, \qquad
\widetilde{\mathbf{V}}_{\mathrm{full}} \coloneqq \begin{bmatrix} \widetilde{\mathbf{V}}^{\mathrm{s}} & \widetilde{\mathbf{V}}^{\mathrm{s}}_{\perp} \end{bmatrix} \in \mathbb{R}^{q \times q}.
\end{equation*}

Under Condition~\eqref{eq:require-initial} and constraint~\eqref{eq:opt-constraints}, it follows from the triangle inequality that 
\begin{equation}
\label{eq:cond-lasso-l1-err}
\left\Vert \widehat{\bm{\Delta}} \right\Vert_{\mathrm{F}} \leq 3 R_{2,n} \qquad \text{and} \qquad \left\Vert \widehat{\bm{\Delta}} \right\Vert_{1,2} \leq 3 R_{1,n}.
\end{equation}
Further, \eqref{eq:cond-lasso-l2-err} implies that $3R_{2,n} < d_r^*/4$. Hence, with the aid of Weyl's inequality, we can deduce that 
\[
\widehat{d}_j \geq d_j^* - \left\Vert \widehat{\bm{\Delta}} \right\Vert_2 \geq d_j^* - \left\Vert \widehat{\bm{\Delta}} \right\Vert_{\mathrm{F}} \geq \frac{3}{4} d_j^* > 0,
\qquad 1 \leq j \leq r.
\]
In particular, it holds that $\widehat{\mathbf{D}}^- = \widehat{\mathbf{D}}^{-1}$ and $\widehat{\mathbf{B}}\widehat{\mathbf{D}}^- = \widehat{\mathbf{V}}$.

Given the regularization parameters
\begin{equation*}
\lambda_u = 8 \sqrt{2} \cdot \sqrt{ \frac{ \nu \alpha_{\max} \log \left( 12pr / \delta \right) }{ n } }, \qquad
\lambda_v = 2 \sqrt{2} \cdot \sqrt{ \frac{ \nu \alpha_{\max} \log \left( 6qr / \delta \right) }{ n } },
\end{equation*}
it follows from Lemma~\ref{lemma:good-event}, Conditions~\eqref{eq:require-initial}--\eqref{eq:cond-lasso-l2-err}, and Condition~\eqref{eq:cond-O2} that the following concentration inequalities hold with probability at least $1 - \delta$,
\begin{align}
\left\vert \frac{1}{n} \widetilde{\mathbf{U}}^{\top}_{\mathrm{full}} \mathbf{X}^{\top} \mathbf{E} \widehat{\mathbf{B}} \widehat{\mathbf{D}}^- \right\vert_{\infty} & \leq \frac{\lambda_u}{2}, \label{eq:lambdau-cond} \\
\left\vert \frac{1}{n} \left( \mathbf{X} \mathbf{U}^* \right)^{\top} \mathbf{E} \widehat{\mathbf{B}} \widehat{\mathbf{D}}^- \right\vert_{\infty} & \leq \frac{\lambda_u}{2}, \label{eq:lambdad-cond} \\
\left\vert \frac{1}{n} \left( \mathbf{X} \mathbf{U}^* \right)^{\top} \mathbf{E} \widetilde{\mathbf{V}}_{\mathrm{full}} \right\vert_{\infty} & \leq \frac{\lambda_v}{2}. \label{eq:lambdav-cond}
\end{align}
Accordingly, it remains only to establish the desired error bound deterministically on the event when \eqref{eq:lambdau-cond}--\eqref{eq:lambdav-cond} hold.

From Conditions~\eqref{eq:require-initial}--\eqref{eq:opt-constraints} and the optimality of $(\widehat{\mathbf{U}},\widehat{\mathbf{D}},\widehat{\mathbf{V}})$, we can deduce that 
\begin{align*}
& \quad \frac{1}{2 n} \left\Vert \mathbf{Y} - \mathbf{X}  \widehat{\mathbf{U}} \widehat{\mathbf{D}} \widehat{\mathbf{V}}^{\top}
\right\Vert^2_{\mathrm{F}} + \lambda_u \vert \widetilde{\mathbf{U}}^{\mathrm{s} \top}_{\perp} \widehat{\mathbf{U}} \widehat{\mathbf{D}} \vert_1 + \lambda_v \vert \widetilde{\mathbf{V}}^{\mathrm{s} \top}_{\perp} \widehat{\mathbf{V}} \widehat{\mathbf{D}} \vert_1 \\
& \leq \frac{1}{2 n} \left\Vert \mathbf{Y} - \mathbf{X}  \mathbf{U}^{*} \mathbf{D}^{*} \mathbf{V}^{* \top}
\right\Vert^2_{\mathrm{F}} + \lambda_u \vert \widetilde{\mathbf{U}}^{\mathrm{s} \top}_{\perp} \mathbf{U}^{*} \mathbf{D}^{*} \vert_1 + \lambda_v \vert \widetilde{\mathbf{V}}^{\mathrm{s} \top}_{\perp} \mathbf{V}^{*} \mathbf{D}^{*} \vert_1.
\end{align*}
Combining this inequality with~\eqref{eq:model-assump} and the definitions in~\eqref{eq:AB-def}, we can show that 
\begin{equation}
\label{eq:pred-upp-bd}
\frac{1}{2n} \Vert \mathbf{X} \widehat{\mathbf{\Delta}} \Vert^2_{\mathrm{F}} \leq  \left\langle \frac{1}{n} \mathbf{X}^{\top}\mathbf{E},  \widehat{\mathbf{\Delta}} \right\rangle + \lambda_u \left\{ \left\vert \widetilde{\mathbf{U}}^{\mathrm{s} \top}_{\perp} \mathbf{A}^{*} \right\vert_1 - \left\vert \widetilde{\mathbf{U}}^{\mathrm{s} \top}_{\perp} \widehat{\mathbf{A}} \right\vert_1 \right\} + \lambda_v \left\{ \left\vert \widetilde{\mathbf{V}}^{\mathrm{s} \top}_{\perp} \mathbf{B}^{*} \right\vert_1 - \left\vert \widetilde{\mathbf{V}}^{\mathrm{s} \top}_{\perp} \widehat{\mathbf{B}} \right\vert_1 \right\}.
\end{equation}
To deal with the stochastic inner product, note that
\begin{align*}
\widehat{\mathbf{\Delta}} &= \widehat{\mathbf{U}} \widehat{\mathbf{D}} \widehat{\mathbf{V}}^{\top} - \mathbf{U}^* \mathbf{D}^* \mathbf{V}^{* \top} = \widehat{\mathbf{A}} \widehat{\mathbf{D}}^- \widehat{\mathbf{B}}^{\top} - \mathbf{U}^* \mathbf{B}^{* \top} \\
&= \widehat{\mathbf{\Delta}}^a \left( \widehat{\mathbf{B}} \widehat{\mathbf{D}}^- \right)^{\top} + \mathbf{A}^* \widehat{\mathbf{D}}^- \widehat{\mathbf{B}}^{\top} - \mathbf{U}^* \mathbf{B}^{* \top} \\
& = \widehat{\mathbf{\Delta}}^a \left( \widehat{\mathbf{B}} \widehat{\mathbf{D}}^- \right)^{\top} - \mathbf{U}^* \widehat{\mathbf{\Delta}}^d \left( \widehat{\mathbf{B}} \widehat{\mathbf{D}}^- \right)^{\top} + \mathbf{U}^* \left( \widehat{\mathbf{\Delta}}^{b} \right)^{\top}.
\end{align*}
Then it holds that 
\begin{equation}
\label{eq:inner-prod-decompose}
\left\langle \frac{1}{n} \mathbf{X}^{\top}\mathbf{E},  \widehat{\mathbf{\Delta}} \right\rangle = \left\langle \frac{1}{n} \mathbf{X}^{\top} \mathbf{E} \widehat{\mathbf{B}} \widehat{\mathbf{D}}^- , \widehat{\mathbf{\Delta}}^a \right\rangle - \left\langle \frac{1}{n} \mathbf{U}^{* \top} \mathbf{X}^{\top} \mathbf{E} \widehat{\mathbf{B}} \widehat{\mathbf{D}}^-, \widehat{\mathbf{\Delta}}^d  \right\rangle + \left\langle \frac{1}{n} \mathbf{E}^{\top} \mathbf{X} \mathbf{U}^{*}, \widehat{\mathbf{\Delta}}^b \right\rangle.
\end{equation}

We next decompose
\begin{equation*}
\widehat{\mathbf{\Delta}}^a = \widetilde{\mathbf{U}}^{\mathrm{s}} \widetilde{\mathbf{U}}^{\mathrm{s} \top} \widehat{\mathbf{\Delta}}^a + \widetilde{\mathbf{U}}^{\mathrm{s}}_{\perp} \widetilde{\mathbf{U}}^{\mathrm{s} \top}_{\perp} \widehat{\mathbf{\Delta}}^a.
\end{equation*}
It follows that
\begin{align*}
\left\langle \frac{1}{n} \mathbf{X}^{\top} \mathbf{E} \widehat{\mathbf{B}} \widehat{\mathbf{D}}^- , \widehat{\mathbf{\Delta}}^a \right\rangle &= \left\langle \frac{1}{n} \widetilde{\mathbf{U}}^{\mathrm{s} \top} \mathbf{X}^{\top} \mathbf{E} \widehat{\mathbf{B}} \widehat{\mathbf{D}}^- , \widetilde{\mathbf{U}}^{\mathrm{s} \top} \widehat{\mathbf{\Delta}}^a \right\rangle + \left\langle \frac{1}{n} \widetilde{\mathbf{U}}^{\mathrm{s} \top}_{\perp} \mathbf{X}^{\top} \mathbf{E} \widehat{\mathbf{B}} \widehat{\mathbf{D}}^- , \widetilde{\mathbf{U}}^{\mathrm{s} \top}_{\perp} \widehat{\mathbf{\Delta}}^a \right\rangle \\
& \leq \left\vert \frac{1}{n} \widetilde{\mathbf{U}}^{\mathrm{s} \top} \mathbf{X}^{\top} \mathbf{E} \widehat{\mathbf{B}} \widehat{\mathbf{D}}^- \right\vert_{\infty} \left\vert \widetilde{\mathbf{U}}^{\mathrm{s} \top} \widehat{\mathbf{\Delta}}^a \right\vert_1 + \left\vert \frac{1}{n} \widetilde{\mathbf{U}}^{\mathrm{s} \top}_{\perp} \mathbf{X}^{\top} \mathbf{E} \widehat{\mathbf{B}} \widehat{\mathbf{D}}^- \right\vert_{\infty} \left\vert \widetilde{\mathbf{U}}^{\mathrm{s} \top}_{\perp} \widehat{\mathbf{\Delta}}^a \right\vert_1.
\end{align*}
In light of 
\begin{align*}
& \quad \max \left\{ \left\vert \frac{1}{n} \widetilde{\mathbf{U}}^{\mathrm{s} \top} \mathbf{X}^{\top} \mathbf{E} \widehat{\mathbf{B}} \widehat{\mathbf{D}}^- \right\vert_{\infty},  \left\vert \frac{1}{n} \widetilde{\mathbf{U}}^{\mathrm{s} \top}_{\perp} \mathbf{X}^{\top} \mathbf{E} \widehat{\mathbf{B}} \widehat{\mathbf{D}}^- \right\vert_{\infty} \right\} = \left\vert \frac{1}{n} \widetilde{\mathbf{U}}^{\top}_{\mathrm{full}} \mathbf{X}^{\top} \mathbf{E} \widehat{\mathbf{B}} \widehat{\mathbf{D}}^- \right\vert_{\infty},
\end{align*}
from \eqref{eq:lambdau-cond} we have that 
\begin{align*}
\left\langle \frac{1}{n} \mathbf{X}^{\top} \mathbf{E} \widehat{\mathbf{B}} \widehat{\mathbf{D}}^- , \widehat{\mathbf{\Delta}}^a \right\rangle & \leq \frac{\lambda_u}{2} \left\vert \widetilde{\mathbf{U}}^{\mathrm{s} \top} \widehat{\mathbf{\Delta}}^a \right\vert_1 + \frac{\lambda_u}{2} \left\vert \widetilde{\mathbf{U}}^{\mathrm{s} \top}_{\perp} \widehat{\mathbf{\Delta}}^a \right\vert_1 \\
& \leq \frac{\lambda_u}{2} \left\vert \widetilde{\mathbf{U}}^{\mathrm{s} \top} \widehat{\mathbf{\Delta}}^a \right\vert_1 + \frac{\lambda_u}{2} \left\vert \left[ \widetilde{\mathbf{U}}^{\mathrm{s} \top}_{\perp} \widehat{\mathbf{\Delta}}^a \right]_{\mathcal{S}_u} \right\vert_1 + \frac{\lambda_u}{2} \left\vert \left[ \widetilde{\mathbf{U}}^{\mathrm{s} \top}_{\perp} \widehat{\mathbf{\Delta}}^a \right]_{\mathcal{S}_u^c} \right\vert_1.
\end{align*}
With the aid of 
\begin{equation*}
\left\vert \left[ \widetilde{\mathbf{U}}^{\mathrm{s} \top}_{\perp} \widehat{\mathbf{\Delta}}^a \right]_{\mathcal{S}_u^c} \right\vert_1 \leq \left\vert \left[ \widetilde{\mathbf{U}}^{\mathrm{s} \top}_{\perp} \widehat{\mathbf{A}} \right]_{\mathcal{S}_u^c} \right\vert_1 + \left\vert \left[ \widetilde{\mathbf{U}}^{\mathrm{s} \top}_{\perp} \mathbf{A}^* \right]_{\mathcal{S}_u^c} \right\vert_1,
\end{equation*}
we can obtain that 
\begin{equation}
\label{eq:inner-prod-term1-bd}
\left\langle \frac{1}{n} \mathbf{X}^{\top} \mathbf{E} \widehat{\mathbf{B}} \widehat{\mathbf{D}}^- , \widehat{\mathbf{\Delta}}^a \right\rangle \leq \frac{\lambda_u}{2} \left\vert \widetilde{\mathbf{U}}^{\mathrm{s} \top} \widehat{\mathbf{\Delta}}^a \right\vert_1 + \frac{\lambda_u}{2} \left\vert \left[ \widetilde{\mathbf{U}}^{\mathrm{s} \top}_{\perp} \widehat{\mathbf{\Delta}}^a \right]_{\mathcal{S}_u} \right\vert_1
+ \frac{\lambda_u}{2} \left\vert \left[ \widetilde{\mathbf{U}}^{\mathrm{s} \top}_{\perp} \widehat{\mathbf{A}} \right]_{\mathcal{S}_u^c} \right\vert_1 + \frac{\lambda_u}{2} \left\vert \left[ \widetilde{\mathbf{U}}^{\mathrm{s} \top}_{\perp} \mathbf{A}^* \right]_{\mathcal{S}_u^c} \right\vert_1.
\end{equation}

On the other hand, it holds that 
\begin{align*}
\left\vert \widetilde{\mathbf{U}}^{\mathrm{s} \top}_{\perp} \mathbf{A}^{*} \right\vert_1 - \left\vert \widetilde{\mathbf{U}}^{\mathrm{s} \top}_{\perp} \widehat{\mathbf{A}} \right\vert_1 &= \left\vert \left[ \widetilde{\mathbf{U}}^{\mathrm{s} \top}_{\perp} \mathbf{A}^{*} \right]_{\mathcal{S}_u} \right\vert_1 - \left\vert \left[ \widetilde{\mathbf{U}}^{\mathrm{s} \top}_{\perp} \widehat{\mathbf{A}} \right]_{\mathcal{S}_u} \right\vert_1 + \left\vert \left[ \widetilde{\mathbf{U}}^{\mathrm{s} \top}_{\perp} \mathbf{A}^{*} \right]_{\mathcal{S}_u^c} \right\vert_1 - \left\vert \left[ \widetilde{\mathbf{U}}^{\mathrm{s} \top}_{\perp} \widehat{\mathbf{A}} \right]_{\mathcal{S}_u^c} \right\vert_1 \\
& \leq \left\vert \left[ \widetilde{\mathbf{U}}^{\mathrm{s} \top}_{\perp} \mathbf{A}^{*} \right]_{\mathcal{S}_u} -  \left[ \widetilde{\mathbf{U}}^{\mathrm{s} \top}_{\perp} \widehat{\mathbf{A}} \right]_{\mathcal{S}_u} \right\vert_1 + \left\vert \left[ \widetilde{\mathbf{U}}^{\mathrm{s} \top}_{\perp} \mathbf{A}^{*} \right]_{\mathcal{S}_u^c} \right\vert_1 - \left\vert \left[ \widetilde{\mathbf{U}}^{\mathrm{s} \top}_{\perp} \widehat{\mathbf{A}} \right]_{\mathcal{S}_u^c} \right\vert_1 \\
& = \left\vert  \left[ \widetilde{\mathbf{U}}^{\mathrm{s} \top}_{\perp} \widehat{\mathbf{\Delta}}^a \right]_{\mathcal{S}_u} \right\vert_1 + \left\vert \left[ \widetilde{\mathbf{U}}^{\mathrm{s} \top}_{\perp} \mathbf{A}^{*} \right]_{\mathcal{S}_u^c} \right\vert_1 - \left\vert \left[ \widetilde{\mathbf{U}}^{\mathrm{s} \top}_{\perp} \widehat{\mathbf{A}} \right]_{\mathcal{S}_u^c} \right\vert_1. \label{eq:UTA-bound} \numberthis
\end{align*}
Hence, combining \eqref{eq:inner-prod-term1-bd} and \eqref{eq:UTA-bound} leads to 
\begin{align*}
& \quad \left\langle \frac{1}{n} \mathbf{X}^{\top} \mathbf{E} \widehat{\mathbf{B}} \widehat{\mathbf{D}}^- , \widehat{\mathbf{\Delta}}^a \right\rangle + \lambda_u \left\{ \left\vert \widetilde{\mathbf{U}}^{\mathrm{s} \top}_{\perp} \mathbf{A}^{*} \right\vert_1 - \left\vert \widetilde{\mathbf{U}}^{\mathrm{s} \top}_{\perp} \widehat{\mathbf{A}} \right\vert_1 \right\} \\
& \leq \frac{\lambda_u}{2} \left\vert \widetilde{\mathbf{U}}^{\mathrm{s} \top} \widehat{\mathbf{\Delta}}^a \right\vert_1 + \frac{3\lambda_u}{2} \left\vert \left[ \widetilde{\mathbf{U}}^{\mathrm{s} \top}_{\perp} \widehat{\mathbf{\Delta}}^a \right]_{\mathcal{S}_u} \right\vert_1 + \frac{3\lambda_u}{2}  \left\vert \left[ \widetilde{\mathbf{U}}^{\mathrm{s} \top}_{\perp} \mathbf{A}^* \right]_{\mathcal{S}_u^c} \right\vert_1 - \frac{\lambda_u}{2} \left\vert \left[ \widetilde{\mathbf{U}}^{\mathrm{s} \top}_{\perp} \widehat{\mathbf{A}} \right]_{\mathcal{S}_u^c} \right\vert_1 \\
& \leq \frac{3\lambda_u}{2} \left\vert \widetilde{\mathbf{U}}^{\mathrm{s} \top} \widehat{\mathbf{\Delta}}^a \right\vert_1 + \frac{3\lambda_u}{2} \left\vert \left[ \widetilde{\mathbf{U}}^{\mathrm{s} \top}_{\perp} \widehat{\mathbf{\Delta}}^a \right]_{\mathcal{S}_u} \right\vert_1 + \frac{3\lambda_u}{2}  \left\vert \left[ \widetilde{\mathbf{U}}^{\mathrm{s} \top}_{\perp} \mathbf{A}^* \right]_{\mathcal{S}_u^c} \right\vert_1 . \label{eq:product-term1-bd} \numberthis
\end{align*}

Observe that $\left\Vert \widehat{\mathbf{\Delta}}^{a} \right\Vert^2_{\mathrm{F}} = \left\Vert \widetilde{\mathbf{U}}^{\mathrm{s} \top} \widehat{\mathbf{\Delta}}^{a} \right\Vert^2_{\mathrm{F}} + \left\Vert \widetilde{\mathbf{U}}^{\mathrm{s} \top}_{\perp} \widehat{\mathbf{\Delta}}^{a} \right\Vert^2_{\mathrm{F}}$. Then it follows that 
\begin{equation*}
\left\vert \widetilde{\mathbf{U}}^{\mathrm{s} \top} \widehat{\mathbf{\Delta}}^a \right\vert_1 \leq \sqrt{r_u r} \left\Vert \widetilde{\mathbf{U}}^{\mathrm{s} \top} \widehat{\mathbf{\Delta}}^a \right\Vert_{\mathrm{F}} \leq \sqrt{r_u r} \left\Vert \widehat{\mathbf{\Delta}}^a \right\Vert_{\mathrm{F}}
\end{equation*}
and
\begin{equation*}
\left\vert \left[ \widetilde{\mathbf{U}}^{\mathrm{s} \top}_{\perp} \widehat{\mathbf{\Delta}}^a \right]_{\mathcal{S}_u} \right\vert_1 \leq \sqrt{s_u} \left\Vert \left[ \widetilde{\mathbf{U}}^{\mathrm{s} \top}_{\perp} \widehat{\mathbf{\Delta}}^a \right]_{\mathcal{S}_u} \right\Vert_{\mathrm{F}} \leq \sqrt{s_u} \left\Vert \widetilde{\mathbf{U}}^{\mathrm{s} \top}_{\perp} \widehat{\mathbf{\Delta}}^a  \right\Vert_{\mathrm{F}} \leq \sqrt{s_u} \left\Vert \widehat{\mathbf{\Delta}}^a  \right\Vert_{\mathrm{F}}.
\end{equation*}
Consequently, we can deduce that 
\begin{equation*}
\left\vert \widetilde{\mathbf{U}}^{\mathrm{s} \top} \widehat{\mathbf{\Delta}}^a \right\vert_1 + \left\vert \left[ \widetilde{\mathbf{U}}^{\mathrm{s} \top}_{\perp} \widehat{\mathbf{\Delta}}^a \right]_{\mathcal{S}_u} \right\vert_1 \leq \left( \sqrt{r_u r} + \sqrt{s_u} \right) \left\Vert \widehat{\mathbf{\Delta}}^a  \right\Vert_{\mathrm{F}} \leq \sqrt{2 \left( r r_u + s_u  \right)} \left\Vert \widehat{\mathbf{\Delta}}^a  \right\Vert_{\mathrm{F}},
\end{equation*}
where the final inequality above has used $\sqrt{a} + \sqrt{b} \leq \sqrt{2(a+b)}$ for $a,b \geq 0$. Combining this bound with \eqref{eq:product-term1-bd} gives that 
\begin{equation}
\label{eq:inner-prod-1-bound}
\begin{aligned}
& \quad \left\langle \frac{1}{n} \mathbf{X}^{\top} \mathbf{E} \widehat{\mathbf{B}} \widehat{\mathbf{D}}^- , \widehat{\mathbf{\Delta}}^a \right\rangle + \lambda_u \left\{ \left\vert \widetilde{\mathbf{U}}^{\mathrm{s} \top}_{\perp} \mathbf{A}^{*} \right\vert_1 - \left\vert \widetilde{\mathbf{U}}^{\mathrm{s} \top}_{\perp} \widehat{\mathbf{A}} \right\vert_1 \right\} \\
& \leq \frac{3\lambda_u}{2} \sqrt{2 \left( r r_u + s_u  \right)} \left\Vert \widehat{\mathbf{\Delta}}^a  \right\Vert_{\mathrm{F}} + \frac{3\lambda_u}{2}  \left\vert \left[ \widetilde{\mathbf{U}}^{\mathrm{s} \top}_{\perp} \mathbf{A}^* \right]_{\mathcal{S}_u^c} \right\vert_1.
\end{aligned}
\end{equation}

By the same argument and using \eqref{eq:lambdav-cond}, we can show that 
\begin{equation}
\label{eq:inner-prod-2-bound}
\begin{aligned}
& \quad \left\langle \frac{1}{n} \mathbf{E}^{\top} \mathbf{X} \mathbf{U}^{*}, \widehat{\mathbf{\Delta}}^b \right\rangle + \lambda_v \left\{ \left\vert \widetilde{\mathbf{V}}^{\mathrm{s} \top}_{\perp} \mathbf{B}^{*} \right\vert_1 - \left\vert \widetilde{\mathbf{V}}^{\mathrm{s} \top}_{\perp} \widehat{\mathbf{B}} \right\vert_1 \right\} \\
& \leq \frac{3\lambda_v}{2} \sqrt{2\left( r r_v + s_v \right)} \left\Vert \widehat{\mathbf{\Delta}}^b   \right\Vert_{\mathrm{F}} + \frac{3\lambda_v}{2} \left\vert \left[ \widetilde{\mathbf{V}}^{\mathrm{s} \top}_{\perp} \mathbf{B}^* \right]_{\mathcal{S}_v^c} \right\vert_1.
\end{aligned}
\end{equation}
In addition, \eqref{eq:lambdad-cond} entails that 
\begin{equation}
\label{eq:inner-prod-3-bound}
- \left\langle \frac{1}{n} \mathbf{U}^{* \top} \mathbf{X}^{\top} \mathbf{E} \widehat{\mathbf{B}} \widehat{\mathbf{D}}^-, \widehat{\bm{\Delta}}^d \right\rangle
\leq \frac{\lambda_u}{2} \left\vert \widehat{\bm{\Delta}}^d \right\vert_1
\leq \frac{\lambda_u}{2} \sqrt{r} \left\Vert \widehat{\bm{\Delta}}^d \right\Vert_{\mathrm{F}}.
\end{equation}
Let us define 
\begin{equation*}
\lambda_{\max} \coloneqq \max\left\{ \lambda_u, \lambda_v \right\}.
\end{equation*}
In view of the definitions of $\lambda_u$ and $\lambda_v$, this quantity satisfies the upper bound
\begin{equation*}
\lambda_{\max} \leq c_{\lambda} \cdot \sqrt{ \frac{ \log \left( 12 \max(p, q) \cdot r / \delta \right) }{ n } },
\end{equation*}
where $c_{\lambda} = 8 \sqrt{2 \nu \alpha_{\max}}$ is a universal constant. Thus, combining \eqref{eq:pred-upp-bd}, \eqref{eq:inner-prod-decompose}, and \eqref{eq:inner-prod-1-bound}--\eqref{eq:inner-prod-3-bound}, we can obtain the deterministic inequality
\begin{equation}
\label{eq:pred-upp-bd-2}
\begin{aligned}
\frac{1}{2n} \left\Vert \mathbf{X} \widehat{\bm{\Delta}} \right\Vert^2_{\mathrm{F}}
& \leq \lambda_{\max} \left\{
    \frac{3}{2} \sqrt{2 \left( r r_u + s_u \right)} \left\Vert \widehat{\bm{\Delta}}^a \right\Vert_{\mathrm{F}}
  + \frac{3}{2} \sqrt{2 \left( r r_v + s_v \right)} \left\Vert \widehat{\bm{\Delta}}^b \right\Vert_{\mathrm{F}}
  + \frac{1}{2} \sqrt{r} \left\Vert \widehat{\bm{\Delta}}^d \right\Vert_{\mathrm{F}}
\right\} \\
& \quad + \frac{3 \lambda_{\max}}{2} \left\vert \left[ \widetilde{\mathbf{U}}^{\mathrm{s} \top}_{\perp} \mathbf{A}^* \right]_{\mathcal{S}_u^c} \right\vert_1
      + \frac{3 \lambda_{\max}}{2} \left\vert \left[ \widetilde{\mathbf{V}}^{\mathrm{s} \top}_{\perp} \mathbf{B}^* \right]_{\mathcal{S}_v^c} \right\vert_1.
\end{aligned}
\end{equation}

Further, it follows from \eqref{eq:cond-lasso-l1-err} and Condition~\eqref{eq:cond-lasso-l2-err} that 
\[
\left\Vert \widehat{\mathbf{\Delta}} \right\Vert_2 \leq \left\Vert \widehat{\mathbf{\Delta}} \right\Vert_{\mathrm{F}} \leq \frac{\tau}{6 \sqrt{5}} \cdot \left( \frac{d^*_r}{\sqrt{r}} \right) \cdot \left( \frac{d^*_r}{d^*_1} \right)^{3/2}.
\]
By resorting to Lemma~\ref{lemma:sign-consistency}, Lemma~\ref{lemma:perturb-bound}, and Assumption~\ref{assump:eigen-gap}, we have that 
\begin{equation*}
\left\Vert \widehat{\mathbf{\Delta}}^a  \right\Vert_{\mathrm{F}} \leq c \eta_r \left\Vert \widehat{\mathbf{\Delta}}  \right\Vert_{\mathrm{F}}, \quad \left\Vert \widehat{\mathbf{\Delta}}^b  \right\Vert_{\mathrm{F}} \leq c \eta_r \left\Vert \widehat{\mathbf{\Delta}}  \right\Vert_{\mathrm{F}}, \quad \text{and} \quad \left\Vert \widehat{\mathbf{\Delta}}^d  \right\Vert_{\mathrm{F}} \leq \left\Vert \widehat{\mathbf{\Delta}}  \right\Vert_{\mathrm{F}},
\end{equation*}
where $\eta_r = 1 + (1 / \tau ) \cdot \left( \sum_{j=1}^r \left( d^{*}_1 / d^{*}_j \right)^2 \right)^{1/2}$, $0<\tau<1$, and $c > 0$ are universal constants. Substituting these bounds into \eqref{eq:pred-upp-bd-2} yields that 
\begin{align*}
\frac{1}{2n} \left\Vert \mathbf{X} \widehat{\bm{\Delta}} \right\Vert^2_{\mathrm{F}}
& \leq \lambda_{\max} \left\{
    \frac{3 \sqrt{2}}{2} c \eta_r \sqrt{r r_u + s_u}
  + \frac{3 \sqrt{2}}{2} c \eta_r \sqrt{r r_v + s_v}
  + \frac{1}{2} \sqrt{r}
\right\} \cdot \left\Vert \widehat{\bm{\Delta}} \right\Vert_{\mathrm{F}} \\
& \quad + \frac{3 \lambda_{\max}}{2} \left\vert \left[ \widetilde{\mathbf{U}}^{\mathrm{s} \top}_{\perp} \mathbf{A}^* \right]_{\mathcal{S}_u^c} \right\vert_1
      + \frac{3 \lambda_{\max}}{2} \left\vert \left[ \widetilde{\mathbf{V}}^{\mathrm{s} \top}_{\perp} \mathbf{B}^* \right]_{\mathcal{S}_v^c} \right\vert_1 \\
& \leq c' \lambda_{\max} \eta_r \sqrt{ r \left( r_u + r_v + 1 \right) + s_u + s_v } \cdot \left\Vert \widehat{\bm{\Delta}} \right\Vert_{\mathrm{F}} \\
& \quad + \frac{3 \lambda_{\max}}{2} \left\vert \left[ \widetilde{\mathbf{U}}^{\mathrm{s} \top}_{\perp} \mathbf{A}^* \right]_{\mathcal{S}_u^c} \right\vert_1
      + \frac{3 \lambda_{\max}}{2} \left\vert \left[ \widetilde{\mathbf{V}}^{\mathrm{s} \top}_{\perp} \mathbf{B}^* \right]_{\mathcal{S}_v^c} \right\vert_1 \numberthis \label{eq:pred-upp-bd-3}
\end{align*}
for a universal constant $c' > 0$.

We now invoke the restricted strong convexity. Applying Assumption~\ref{assump:design-RSC} columnwise leads to 
\begin{multline*}
\frac{1}{2n} \left\Vert \mathbf{X} \widehat{\bm{\Delta}} \right\Vert^2_{\mathrm{F}} = \sum^q_{j=1} \frac{1}{2n} \left\Vert \mathbf{X} \widehat{\bm{\Delta}}_{:j} \right\Vert^2_{2}
\geq \frac{\rho_2}{2} \sum^q_{j=1} \left\Vert \widehat{\bm{\Delta}}_{:j} \right\Vert^2_{2}
- \rho_1 \frac{\log p}{n} \sum^q_{j=1} \left\Vert \widehat{\bm{\Delta}}_{:j} \right\Vert^2_1
\\
= \frac{\rho_2}{2} \left\Vert \widehat{\bm{\Delta}} \right\Vert^2_{\mathrm{F}} - \rho_1 \frac{\log p}{n} \left\Vert \widehat{\bm{\Delta}} \right\Vert^2_{1,2}
\geq \frac{\rho_2}{2} \left\Vert \widehat{\bm{\Delta}} \right\Vert^2_{\mathrm{F}}
- 9 \rho_1 \frac{\log p}{n} R^2_{1,n},
\end{multline*}
where the last step above has utilized \eqref{eq:cond-lasso-l1-err}. Combining this lower bound with \eqref{eq:pred-upp-bd-3}, we can deduce that 
\begin{equation}
\label{eq:pred-upp-bd-4}
\begin{aligned}
\left\Vert \widehat{\bm{\Delta}} \right\Vert^2_{\mathrm{F}}
& \leq \frac{2 c'}{\rho_2} \lambda_{\max} \eta_r \sqrt{ r(r_u + r_v + 1) + s_u + s_v } \cdot \left\Vert \widehat{\bm{\Delta}} \right\Vert_{\mathrm{F}} \\
& \quad + \frac{3 \lambda_{\max}}{\rho_2} \left\vert \left[ \widetilde{\mathbf{U}}^{\mathrm{s} \top}_{\perp} \mathbf{A}^* \right]_{\mathcal{S}_u^c} \right\vert_1
      + \frac{3 \lambda_{\max}}{\rho_2} \left\vert \left[ \widetilde{\mathbf{V}}^{\mathrm{s} \top}_{\perp} \mathbf{B}^* \right]_{\mathcal{S}_v^c} \right\vert_1
      + \frac{18 \rho_1}{\rho_2} \cdot \frac{\log p}{n} R^2_{1,n} \\
& = \frac{2 c' c_{\lambda}}{\rho_2} \eta_r
    \sqrt{ \frac{ \left[ r(r_u + r_v + 1) + s_u + s_v \right] \cdot \log \left( 12 r \max(p, q) / \delta \right) }{ n } }
    \cdot \left\Vert \widehat{\bm{\Delta}} \right\Vert_{\mathrm{F}} \\
& \quad + \frac{3 c_{\lambda}}{\rho_2} \left\{
    \left\vert \left[ \widetilde{\mathbf{U}}^{\mathrm{s} \top}_{\perp} \mathbf{A}^* \right]_{\mathcal{S}_u^c} \right\vert_1
  + \left\vert \left[ \widetilde{\mathbf{V}}^{\mathrm{s} \top}_{\perp} \mathbf{B}^* \right]_{\mathcal{S}_v^c} \right\vert_1
\right\} \cdot \sqrt{ \frac{ \log \left( 12 r \max(p, q) / \delta \right) }{ n } } \\
& \quad + \frac{18 \rho_1}{\rho_2} \cdot \frac{\log p}{n} R^2_{1,n}.
\end{aligned}
\end{equation}

For any $t \in \mathbb{R}$ satisfying $t^2 - b t - c \leq 0$ with $b, c > 0$, we have that 
$$t \leq b + \sqrt{c}.$$ Applying this elementary fact to \eqref{eq:pred-upp-bd-4}, we can obtain that
\begin{equation}
\label{eq:conclusion-no-approx}
\begin{aligned}
\left\Vert \widehat{\bm{\Delta}} \right\Vert_{\mathrm{F}}
& \leq \frac{2 c' c_{\lambda}}{\rho_2} \eta_r
\sqrt{ \frac{ \left[ r(r_u + r_v + 1) + s_u + s_v \right] \cdot \log \left( 12 r \max(p, q) / \delta \right) }{ n } } \\
& \quad + \sqrt{ \frac{3 c_{\lambda}}{\rho_2} }
\left\{
    \left\vert \left[ \widetilde{\mathbf{U}}^{\mathrm{s} \top}_{\perp} \mathbf{A}^* \right]_{\mathcal{S}_u^c} \right\vert_1
  + \left\vert \left[ \widetilde{\mathbf{V}}^{\mathrm{s} \top}_{\perp} \mathbf{B}^* \right]_{\mathcal{S}_v^c} \right\vert_1
\right\}^{1/2}
\left( \frac{ \log \left( 12 r \max(p, q) / \delta \right) }{ n } \right)^{1/4} \\
& \quad + \sqrt{ \frac{18 \rho_1}{\rho_2} } \sqrt{ \frac{\log p}{n} }  R_{1,n} .
\end{aligned}
\end{equation}

We now bound the approximation terms. It follows from Lemma~\ref{lemma:approx-error} and the inequality $|\mathbf{M}\mathbf{D}^*|_1 \leq d_1^* |\mathbf{M}|_1$ for any matrix $\mathbf{M}$ with $r$ columns that 
\begin{equation*}
\begin{aligned}
& \left\vert \left[ \widetilde{\mathbf{U}}^{\mathrm{s} \top}_{\perp} \mathbf{A}^* \right]_{\mathcal{S}_u^c} \right\vert_1 + \left\vert \left[ \widetilde{\mathbf{V}}^{\mathrm{s} \top}_{\perp} \mathbf{B}^* \right]_{\mathcal{S}_v^c} \right\vert_1 \\
& \leq d_1^* \Bigg[
2^{3/2} \left[ r \left\{ (r_u^* - r_u)_+ + (r_v^* - r_v)_+ \right\} - \left( s_u(r_u) + s_v(r_v) \right) \right]
\frac{\left( 2 d_1^{(0)} + \varepsilon^{(0)} \right) \varepsilon^{(0)}}{\tau \left( \min \left( d_{r_u^*}^{(0)}, d_{r_v^*}^{(0)} \right) \right)^2} \\
& \quad + 2^{5/4} r \left[ \left( r_u^* \right)^{1/4} \sqrt{ p - \max(r_u^*, r_u) } + \left( r_v^* \right)^{1/4} \sqrt{ q - \max(r_v^*, r_v) } \right]
\frac{ \left( \varepsilon^{(0)} \right)^{1/2} \left( \varepsilon^{(0)} + 2 d_1^{(0)} \right)^{1/2} }{\min \left( d_{r_u^*}^{(0)}, d_{r_v^*}^{(0)} \right)}
\Bigg].
\end{aligned}
\end{equation*}
Substituting this bound into \eqref{eq:conclusion-no-approx} gives that 
\begin{equation}
\begin{aligned}
\left\Vert \widehat{\bm{\Delta}} \right\Vert_{\mathrm{F}}
& \leq \frac{2 c^{\prime} c_{\lambda}}{\rho_2} \eta_r \sqrt{ \frac{ \left( r \left( r_u + r_v + 1 \right) + s_u + s_v \right) \log \left( 12 r \max (p,q) / \delta \right) }{ n } } \\
& \quad + \frac{2 \sqrt{d_1^*}}{\tau} \sqrt{\frac{6 c_{\lambda}}{\rho_2}} \widetilde{\varepsilon}_{\mathrm{source}} \left( \frac{ \log \left( 12 r \max (p,q) / \delta \right) }{ n } \right)^{1/4}
+ \sqrt{ \frac{18 \rho_1}{\rho_2} } \sqrt{ \frac{\log p}{n} } R_{1,n},
\end{aligned}
\end{equation}
where
\begin{equation*}
\begin{aligned}
\widetilde{\varepsilon}_{\mathrm{source}}
=   & \left[
    r \left\{ \left( r_u^* - r_u \right)_+ + \left( r_v^* - r_v \right)_+ \right\}
    - \left( s_u(r_u) + s_v(r_v) \right)
\right]^{1/2} \\
& \quad \times
\frac{ \left( \varepsilon^{(0)} \right)^{1/2}
       \left( 2 d_1^{(0)} + \varepsilon^{(0)} \right)^{1/2} }
     { \min \left( d_{r_u^*}^{(0)}, d_{r_v^*}^{(0)} \right) } \\[0.5em]
& + \sqrt{r} \cdot \left[
    \left( r_u^* \right)^{1/4} \cdot \sqrt{ p - \max(r_u^*, r_u) }
    + \left( r_v^* \right)^{1/4} \cdot \sqrt{ q - \max(r_v^*, r_v) }
\right]^{1/2} \\
& \quad \times
\frac{ \left( \varepsilon^{(0)} \right)^{1/4}
       \left( 2 d_1^{(0)} + \varepsilon^{(0)} \right)^{1/4} }
     { \sqrt{ \min \left( d_{r_u^*}^{(0)}, d_{r_v^*}^{(0)} \right) } }.
\end{aligned}
\end{equation*}

In view of Assumption~\ref{assump:constrained-norm-design}, we have that 
\[
\min \left( d_{r_u^*}^{(0)}, d_{r_v^*}^{(0)} \right) \geq d^{(0)}_{\min} = \Omega(1).
\]
Since $\tau$ is universal, it holds that 
\[
\widetilde{\varepsilon}_{\mathrm{source}} \lesssim \varepsilon_{\mathrm{source}},
\]
where $\varepsilon_{\mathrm{source}}$ is defined in~\eqref{eq:eps-source}. Recalling that $c_{\lambda}=8\sqrt{2 \nu \alpha_{\max}}$, we can deduce that 
\begin{multline*}
\left\Vert \widehat{\bm{\Delta}} \right\Vert_{\mathrm{F}} \lesssim \sqrt{\alpha_{\max}} \, \eta_r \sqrt{ \frac{ \left( r \left( r_u + r_v + 1 \right) + s_u + s_v \right) \log \left( 12 r \max (p,q) / \delta \right) }{ n } } \\
+ \alpha_{\max}^{1/4} \left( d_1^* \right)^{1/2} \varepsilon_{\mathrm{source}} \left( \frac{ \log \left( 12 r \max (p,q) / \delta \right) }{ n } \right)^{1/4}
+ R_{1,n} \sqrt{\frac{\log p}{n}}.
\end{multline*}
Therefore, combining this result with Condition~\eqref{eq:cond-O1} yields that 
\begin{align*}
\left\Vert \widehat{\bm{\Delta}} \right\Vert_{\mathrm{F}}
& \lesssim \sqrt{\alpha_{\max}} \, \eta_r \sqrt{ \frac{ \left( r \left( r_u + r_v + 1 \right) + s_u + s_v \right) \log \left( 12 r \max (p,q) / \delta \right) }{ n } } \\
& \quad + \alpha_{\max}^{1/4} \left( d_1^* \right)^{1/2} \varepsilon_{\mathrm{source}} \left( \frac{ \log \left( 12 r \max (p,q) / \delta \right) }{ n } \right)^{1/4}.
\end{align*}
This concludes the proof of Theorem~\ref{thm:spectral-trans}.

\section{Proof of Theorem~\ref{thm:lower-bound}}
\label{sec:proof-lwd-bd}

In this section, we present the complete proof of Theorem~\ref{thm:lower-bound}. We first establish several technical lemmas in Section~\ref{sec:proof-lwd-bd-lemmas}, and then turn to the main arguments in Section~\ref{sec:proof-lwd-bd-lemmas-main-proof}.

\subsection{Technical lemmas}
\label{sec:proof-lwd-bd-lemmas}

Let us begin with the definitions of covering and packing numbers that will be used repeatedly in the lower-bound argument.

\begin{definition}[Covering number]
Let $(\Theta, \rho)$ be a metric space. A set $\{\theta^1, \ldots, \theta^N\} \subseteq \Theta$ is called a $\delta$-cover of $\Theta$ if for each $\theta \in \Theta$, there exists some $i \in [N]$ such that $\rho(\theta,\theta^i) \leq \delta$. The $\delta$-covering number $\mathcal{N}(\delta; \Theta, \rho)$ is the smallest cardinality of such a cover.
\end{definition}

\begin{definition}[Packing number]
Let $(\Theta, \rho)$ be a metric space. A set $\{\theta^1, \ldots, \theta^M\} \subseteq \Theta$ is called a $\delta$-packing of $\Theta$ if $\rho(\theta^i,\theta^j) > \delta$ for all distinct $i,j \in \{1,\ldots,M\}$. The $\delta$-packing number $\mathcal{M}(\delta; \Theta, \rho)$ is the largest cardinality of such a packing.
\end{definition}

Let
\[
\mathrm{St}(p,r) \coloneqq \left\{ \bm{\Phi} \in \mathbb{R}^{p \times r} : \bm{\Phi}^{\top}\bm{\Phi} = \mathbf{I}_r \right\}
\]
be the Stiefel manifold. The lemma below provides matching upper and lower bounds for the covering number of the Stiefel manifold under the Frobenius norm.

\begin{lemma}
\label{lemma:covering-number-Stiefel-Frob}
For integers $1 \leq r \leq p$ with $(p,r) \neq (1,1)$, there exist some universal constants $C > c > 0$ such that for any $0 < \delta < 1$,
\begin{equation}
\label{eq:bd-covering-number}
\left( \frac{c}{\delta} \right)^{pr - \tfrac{r(r+1)}{2}} r^{\tfrac{r(r-1)}{4}}
\leq
\mathcal{N}\left(\delta; \mathrm{St}(p,r), \|\cdot\|_{\mathrm{F}}\right)
\leq
\left( \frac{C}{\delta} \right)^{pr - \tfrac{r(r+1)}{2}} r^{\tfrac{r(r-1)}{4}}.
\end{equation}
\end{lemma}

%\begin{proof}
\textit{Proof}. 
When $p=r$, the desired claim reduces to the case of orthogonal group $O(r)=\mathrm{St}(r,r)$ and follows directly from Proposition~7 of~\cite{szarek1982nets}. Accordingly, we will assume that $p>r$. Let $\mathrm{Gr}(p,r)$ be the Grassmann manifold of $r$-dimensional subspaces of $\mathbb{R}^p$. For any $\mathbf{U}, \mathbf{V} \in \mathrm{St}(p,r)$, denote by 
\[
\rho_1(\mathbf{U}, \mathbf{V}) \coloneqq \left\Vert \mathbf{U} - \mathbf{V} \right\Vert_{\mathrm{F}}.
\]
For any $E,F \in \mathrm{Gr}(p,r)$, we identify them with their orthogonal projection matrices $\mathbf{P}_E$ and $\mathbf{P}_F$, and define
\[
\rho_2(E,F) \coloneqq \frac{1}{\sqrt{2}} \left\Vert \mathbf{P}_E - \mathbf{P}_F \right\Vert_{\mathrm{F}}.
\]
By invoking Lemma~1 of~\cite{cai2013sparse}, after absorbing factor $\sqrt{2}$ into the constants, there exist some absolute constants $c_1 > c_0 > 0$ such that for any $0 < \varepsilon < \sqrt{r \wedge (p-r)}$,
\begin{equation}
\label{eq:bd-cai2013}
\left( \frac{c_0}{\varepsilon} \right)^{r(p-r)}
\leq
\mathcal{N}\left(\varepsilon; \mathrm{Gr}(p,r), \rho_2\right)
\leq
\left( \frac{c_1}{\varepsilon} \right)^{r(p-r)}.
\end{equation}

We first establish the upper bound in \eqref{eq:bd-covering-number}. We set
\[
\varepsilon_1 = \frac{\delta}{2\sqrt{2}},
\qquad
\varepsilon_2 = \frac{\delta}{2}.
\]
Let $\{ E^1, \ldots, E^{N_{\mathrm{base}}} \}$ be an $\varepsilon_1$-cover of $\mathrm{Gr}(p,r)$ with respect to $\rho_2$, and choose $\mathbf{V}^j \in \mathrm{St}(p,r)$ such that $\mathrm{span}(\mathbf{V}^j)=E^j$ for each $j$. Let $\{ \mathbf{O}^1, \ldots, \mathbf{O}^{N_{\mathrm{fib}}} \}$ be an $\varepsilon_2$-cover of $\mathrm{St}(r,r)$ with respect to the Frobenius norm.
We claim that
\[
\left\{ \mathbf{V}^j \mathbf{O}^k : 1 \leq j \leq N_{\mathrm{base}},\ 1 \leq k \leq N_{\mathrm{fib}} \right\}
\]
is a $\delta$-cover of $\mathrm{St}(p,r)$ under $\rho_1$. To verify this, let us fix any $\mathbf{U} \in \mathrm{St}(p,r)$ and choose $j$ such that
\[
\rho_2(\mathrm{span}(\mathbf{U}), E^j) \leq \varepsilon_1.
\]

We next define
\[
\mathbf{O}^* \in \arg \min_{\mathbf{O} \in \mathrm{St}(r,r)} \left\Vert \mathbf{U} - \mathbf{V}^j \mathbf{O} \right\Vert_{\mathrm{F}},
\]
and then choose $k$ so that $\| \mathbf{O}^* - \mathbf{O}^k \|_{\mathrm{F}} \leq \varepsilon_2$.
It follows from the standard principal-angle relations that 
\[
\rho_2(\mathrm{span}(\mathbf{U}), \mathrm{span}(\mathbf{V}))
\leq
\min_{\mathbf{O} \in \mathrm{St}(r,r)} \left\Vert \mathbf{U} - \mathbf{V}\mathbf{O} \right\Vert_{\mathrm{F}}
\leq
\sqrt{2}\,\rho_2(\mathrm{span}(\mathbf{U}), \mathrm{span}(\mathbf{V}))
\]
for all $\mathbf{U}, \mathbf{V} \in \mathrm{St}(p,r)$. Then it holds that 
\begin{align*}
\left\Vert \mathbf{U} - \mathbf{V}^j \mathbf{O}^k \right\Vert_{\mathrm{F}}
&\leq \left\Vert \mathbf{U} - \mathbf{V}^j \mathbf{O}^* \right\Vert_{\mathrm{F}}
+ \left\Vert \mathbf{V}^j \left( \mathbf{O}^* - \mathbf{O}^k \right) \right\Vert_{\mathrm{F}} \\
&\leq \sqrt{2}\,\rho_2(\mathrm{span}(\mathbf{U}), E^j)
+ \left\Vert \mathbf{O}^* - \mathbf{O}^k \right\Vert_{\mathrm{F}} \\
&\leq \sqrt{2}\,\varepsilon_1 + \varepsilon_2 \\
&= \delta.
\end{align*}
Consequently, we can obtain that 
\begin{equation}
\label{eq:covering-upp-bd-other-packing}
\mathcal{N}\left(\delta; \mathrm{St}(p,r), \|\cdot\|_{\mathrm{F}}\right)
\leq
\mathcal{N}\left(\frac{\delta}{2\sqrt{2}}; \mathrm{Gr}(p,r), \rho_2\right)
\mathcal{N}\left(\frac{\delta}{2}; \mathrm{St}(r,r), \|\cdot\|_{\mathrm{F}}\right).
\end{equation}

With the aid of Proposition~7 of~\cite{szarek1982nets}, there exists a universal constant $c_3 > 0$ such that
\begin{equation}
\label{eq:covering-upp-bd-Or}
\mathcal{N}\left(\frac{\delta}{2}; \mathrm{St}(r,r), \|\cdot\|_{\mathrm{F}}\right)
\leq
\left( \frac{2 c_3 \sqrt{r}}{\delta} \right)^{\tfrac{r(r-1)}{2}}.
\end{equation}
Hence, combining \eqref{eq:bd-cai2013}, \eqref{eq:covering-upp-bd-other-packing}, and \eqref{eq:covering-upp-bd-Or}, and setting $C = \max\{2\sqrt{2}c_1, 2c_3\}$ yield that 
\[
\mathcal{N}\left(\delta; \mathrm{St}(p,r), \|\cdot\|_{\mathrm{F}}\right)
\leq
\left( \frac{C}{\delta} \right)^{pr - \tfrac{r(r+1)}{2}} r^{\tfrac{r(r-1)}{4}}.
\]

We now turn to the lower bound. Let $\{E^1, \dots, E^{M_{\mathrm{base}}}\} \subseteq \mathrm{Gr}(p,r)$ be a $\delta$-packing of $\mathrm{Gr}(p,r)$ with respect to $\rho_2$, and $\{\mathbf{Q}^1, \dots, \mathbf{Q}^{M_{\mathrm{fib}}}\} \subseteq \mathrm{St}(r,r)$ a $\delta$-packing of $\mathrm{St}(r,r)$ with respect to the Frobenius norm. We choose $\mathbf{U}^j \in \mathrm{St}(p,r)$ so that $\mathrm{span}(\mathbf{U}^j)=E^j$ for each $j$.
We claim that
\[
\left\{ \mathbf{U}^j \mathbf{Q}^k : 1 \leq j \leq M_{\mathrm{base}},\ 1 \leq k \leq M_{\mathrm{fib}} \right\}
\]
forms a $\delta$-packing of $\mathrm{St}(p,r)$ under $\rho_1$.
Let us fix $j$ and consider $k \neq \ell$. Then it holds that 
\[
\left\Vert \mathbf{U}^j \mathbf{Q}^k - \mathbf{U}^j \mathbf{Q}^{\ell} \right\Vert_{\mathrm{F}}
=
\left\Vert \mathbf{Q}^k - \mathbf{Q}^{\ell} \right\Vert_{\mathrm{F}}
>
\delta.
\]

Now assume that $j \neq j'$. For any $k,\ell$, we have that 
\[
\left\Vert \mathbf{U}^j \mathbf{Q}^k - \mathbf{U}^{j'} \mathbf{Q}^{\ell} \right\Vert_{\mathrm{F}}
=
\left\Vert \mathbf{U}^j - \mathbf{U}^{j'} \widetilde{\mathbf{Q}} \right\Vert_{\mathrm{F}},
\qquad
\widetilde{\mathbf{Q}} \coloneqq \mathbf{Q}^{\ell} \left( \mathbf{Q}^k \right)^{-1} \in \mathrm{St}(r,r).
\]
Then it follows that 
\[
\left\Vert \mathbf{U}^j \mathbf{Q}^k - \mathbf{U}^{j'} \mathbf{Q}^{\ell} \right\Vert_{\mathrm{F}}
\geq
\min_{\mathbf{Q} \in \mathrm{St}(r,r)} \left\Vert \mathbf{U}^j - \mathbf{U}^{j'} \mathbf{Q} \right\Vert_{\mathrm{F}}
\geq
\rho_2(E^j, E^{j'})
>
\delta.
\]
Consequently, we can deduce that 
\begin{equation}
\label{eq:packing-lw-bd-other-packing}
\mathcal{M}\left(\delta; \mathrm{St}(p,r), \|\cdot\|_{\mathrm{F}}\right)
\geq
\mathcal{M}\left(\delta; \mathrm{Gr}(p,r), \rho_2\right)
\mathcal{M}\left(\delta; \mathrm{St}(r,r), \|\cdot\|_{\mathrm{F}}\right).
\end{equation}

It follows from Lemma~5.5 of~\cite{wainwright2019high}, \eqref{eq:bd-cai2013}, and Proposition~7 of~\cite{szarek1982nets} that there exists some universal constant $c_2 > 0$ such that
\[
\mathcal{M}\left(\delta; \mathrm{Gr}(p,r), \rho_2\right)
\geq
\mathcal{N}\left(\delta; \mathrm{Gr}(p,r), \rho_2\right)
\geq
\left( \frac{c_0}{\delta} \right)^{r(p-r)}
\]
and
\[
\mathcal{M}\left(\delta; \mathrm{St}(r,r), \|\cdot\|_{\mathrm{F}}\right)
\geq
\mathcal{N}\left(\delta; \mathrm{St}(r,r), \|\cdot\|_{\mathrm{F}}\right)
\geq
\left( \frac{c_2 \sqrt{r}}{\delta} \right)^{\tfrac{r(r-1)}{2}}
\]
for all $0 < \delta < 1$. Combining these inequalities with \eqref{eq:packing-lw-bd-other-packing} and setting $c' = \min(c_0, c_2)$ give that 
\[
\mathcal{M}\left(\delta; \mathrm{St}(p,r), \|\cdot\|_{\mathrm{F}}\right)
\geq
\left( \frac{c'}{\delta} \right)^{pr - \tfrac{r(r+1)}{2}} r^{\tfrac{r(r-1)}{4}}.
\]
Therefore, an application of Lemma~5.5 of~\cite{wainwright2019high} yields that 
\[
\mathcal{N}\left(\delta; \mathrm{St}(p,r), \|\cdot\|_{\mathrm{F}}\right)
\geq
\mathcal{M}\left(2\delta; \mathrm{St}(p,r), \|\cdot\|_{\mathrm{F}}\right)
\geq
\left( \frac{c}{\delta} \right)^{pr - \tfrac{r(r+1)}{2}} r^{\tfrac{r(r-1)}{4}}
\]
for $c = c'/2$. This completes the proof of Lemma \ref{lemma:covering-number-Stiefel-Frob}.
%\end{proof}

Denote by 
\[
\mathbb{S}_{p-1}(s)
\coloneqq
\left\{ \mathbf{v} \in \mathbb{R}^{p} : \|\mathbf{v}\|_2 = 1,\ \|\mathbf{v}\|_0 \leq s \right\}
\]
the set of $s$-sparse unit vectors in $\mathbb{R}^{p}$. The next lemma is a direct specialization of Lemma~3.1.2 in~\cite{vu2012minimax}.

\begin{lemma}[Lemma~3.1.2 of~\cite{vu2012minimax}]
\label{lemma:packing-set-sparse-unit}
Assume that $s \geq 2$ and $p \geq 5$. Then there exists a universal constant $c > 0$ such that for each $0 < \delta \leq 1$, there exists a finite subset
$\{ \mathbf{v}^1, \ldots, \mathbf{v}^M \} \subseteq \mathbb{S}_{p-1}(s)$
satisfying that for each $j \neq \ell$,
\[
\frac{\delta}{2}
<
\left\| \mathbf{v}^j - \mathbf{v}^{\ell} \right\|_2
\leq
2\delta
\]
and
\[
\log M \geq c (s-1) \log\left( \frac{p-1}{s-1} \right).
\]
\end{lemma}

%\begin{proof}
\textit{Proof}. 
Let us fix $0 < \delta \leq 1$. The conclusion follows directly from Lemma~3.1.2 of~\cite{vu2012minimax} with $q=0$ and $R_0=s$. The distance bounds stated here are weaker than those in the original lemma and therefore continue to hold. This concludes the proof of Lemma \ref{lemma:packing-set-sparse-unit}.
%\end{proof}

The following lemma provides the local metric-entropy lower bound that we will use in the minimax arguments below.

\begin{lemma}[Proposition~3 of~\cite{cai2013sparse}]
\label{lemma:local-metric-entropy}
Let $(\Theta, \rho)$ be a totally bounded metric space, and $\{ \mathbb{P}_{\theta} : \theta \in \Theta \}$ a collection of probability measures.
Assume that
\[
\sup_{\theta \neq \theta^{\prime}}
\frac{ D_{\mathrm{KL}}(\mathbb{P}_{\theta} \| \mathbb{P}_{\theta^{\prime}} ) }
     { \rho^2(\theta,\theta^{\prime}) }
\leq
A
\]
and further that there exist constants $0 < c < C < \infty$ and $d,d^{\prime} \geq 1$ such that
\[
d^{\prime} \left( \frac{c}{\delta} \right)^d
\leq
\mathcal{N}\left(\delta; \Theta, \rho \right)
\leq
d^{\prime} \left( \frac{C}{\delta} \right)^d
\]
for all $0 < \delta < \delta_0$. Then we have that 
\[
\inf_{\widehat{\theta}} \sup_{\theta \in \Theta}
\mathbb{E}\left[ \rho^2\left( \widehat{\theta}, \theta \right) \right]
\geq
\frac{c^2}{840 C^2}
\min\left\{ \frac{d}{A}, \delta_0^2 \right\}.
\]
\end{lemma}

%\begin{proof}
\textit{Proof}. 
Let us set $\widetilde{c} = d^{\prime 1/d} c$ and $\widetilde{C} = d^{\prime 1/d} C$. Then it holds that 
\[
\left( \frac{\widetilde{c}}{\delta} \right)^d
\leq
\mathcal{N}\left(\delta; \Theta, \rho \right)
\leq
\left( \frac{\widetilde{C}}{\delta} \right)^d
\]
for all $0 < \delta < \delta_0$. Therefore, an application of Proposition~3 of~\cite{cai2013sparse} with $\widetilde{c}$ and $\widetilde{C}$ yields that 
\[
\inf_{\widehat{\theta}} \sup_{\theta \in \Theta}
\mathbb{E}\left[ \rho^2\left( \widehat{\theta}, \theta \right) \right]
\geq
\frac{\widetilde{c}^2}{840 \widetilde{C}^2}
\min\left\{ \frac{d}{A}, \delta_0^2 \right\}
=
\frac{c^2}{840 C^2}
\min\left\{ \frac{d}{A}, \delta_0^2 \right\}.
\]
This completes the proof of Lemma \ref{lemma:local-metric-entropy}.
%\end{proof}

\subsection{Proof of Theorem~\ref{thm:lower-bound}}
\label{sec:proof-lwd-bd-lemmas-main-proof}

Throughout this proof, we will use $\mathbb{E}_{\RegM}$ to denote expectation under the model $\mathbf{Y} = \mathbf{X}\RegM + \mathbf{E}$. Let us choose any $\mathbf{U}^* \in \mathcal{U}$, $\mathbf{V}^* \in \mathcal{V}$, and $\mathbf{D}^* \in \mathcal{D}$ such that $d^*_r = \underline{c}$; such a choice is possible by rescaling an arbitrary element of $\mathcal{D}$. Denote by 
\begin{align*}
\mathcal{C}_u &\coloneqq \Bigl\{ \RegM \in \mathbb{R}^{p \times q} : \RegM = \mathbf{U} \mathbf{D}^* \mathbf{V}^{*\top}, \mathbf{U} \in \mathcal{U} \Bigr\}, \\
\mathcal{C}_v &\coloneqq \Bigl\{ \RegM \in \mathbb{R}^{p \times q} : \RegM = \mathbf{U}^* \mathbf{D}^* \mathbf{V}^\top, \mathbf{V} \in \mathcal{V} \Bigr\}.
\end{align*}
Since $\mathcal{C}_u, \mathcal{C}_v \subseteq \mathcal{C}$, it holds that 
\begin{equation}
\label{eq:frob-lwd}
\begin{aligned}
\inf_{\widehat{\RegM}} \sup_{\RegM \in \mathcal{C}}
\mathbb{E}_{\RegM}\left\Vert \widehat{\RegM} - \RegM \right\Vert_{\mathrm{F}}^2
&\geq \max \Biggl\{
\inf_{\widehat{\RegM}} \sup_{\RegM \in \mathcal{C}_u}
\mathbb{E}_{\RegM}\left\Vert \widehat{\RegM} - \RegM \right\Vert_{\mathrm{F}}^2,
\inf_{\widehat{\RegM}} \sup_{\RegM \in \mathcal{C}_v}
\mathbb{E}_{\RegM}\left\Vert \widehat{\RegM} - \RegM \right\Vert_{\mathrm{F}}^2
\Biggr\} \\
&\geq \frac{1}{2} \Biggl(
\inf_{\widehat{\RegM}} \sup_{\RegM \in \mathcal{C}_u}
\mathbb{E}_{\RegM}\left\Vert \widehat{\RegM} - \RegM \right\Vert_{\mathrm{F}}^2
+
\inf_{\widehat{\RegM}} \sup_{\RegM \in \mathcal{C}_v}
\mathbb{E}_{\RegM}\left\Vert \widehat{\RegM} - \RegM \right\Vert_{\mathrm{F}}^2
\Biggr).
\end{aligned}
\end{equation}
Consequently, it suffices to derive the lower bounds for each term on the right-hand side of \eqref{eq:frob-lwd} above.

Let us define
\begin{equation*}
\mathcal{P}_u \coloneqq \Bigl\{ \mathbf{P} \in \mathbb{R}^{r_0 \times r} : \mathbf{P}^{\top} \mathbf{P} = \mathbf{I}_r, |\mathbf{P}|_0 \leq s^*_u \Bigr\}.
\end{equation*}
Then we have that 
\begin{align*}
\mathcal{U} &= \Bigl\{ \mathbf{U}^{(0)} \mathbf{P} : \mathbf{P} \in \mathcal{P}_u \Bigr\}, \\
\mathcal{C}_u &= \Bigl\{ \RegM \in \mathbb{R}^{p \times q} : \RegM = \mathbf{U}^{(0)} \mathbf{P} \mathbf{D}^* \mathbf{V}^{*\top}, \mathbf{P} \in \mathcal{P}_u \Bigr\}.
\end{align*}
Hence, we can deduce that 
\begin{equation}
\label{eq:lwb-C-by-P}
\begin{aligned}
\inf_{\widehat{\RegM}} \sup_{\RegM \in \mathcal{C}_u}
\mathbb{E}_{\RegM}\left\Vert \widehat{\RegM} - \RegM \right\Vert_{\mathrm{F}}^2
&\geq \inf_{\widehat{\mathbf{P}}} \sup_{\mathbf{P} \in \mathcal{P}_u}
\mathbb{E}_{\mathbf{P}}\left\Vert \mathbf{U}^{(0)} (\widehat{\mathbf{P}} - \mathbf{P}) \mathbf{D}^* \mathbf{V}^{*\top} \right\Vert_{\mathrm{F}}^2 \\
&= \inf_{\widehat{\mathbf{P}}} \sup_{\mathbf{P} \in \mathcal{P}_u}
\mathbb{E}_{\mathbf{P}}\left\Vert (\widehat{\mathbf{P}} - \mathbf{P}) \mathbf{D}^* \right\Vert_{\mathrm{F}}^2 \\
&\geq \underline{c}^2 \inf_{\widehat{\mathbf{P}}} \sup_{\mathbf{P} \in \mathcal{P}_u}
\mathbb{E}_{\mathbf{P}}\left\Vert \widehat{\mathbf{P}} - \mathbf{P} \right\Vert_{\mathrm{F}}^2,
\end{aligned}
\end{equation}
where $\mathbb{E}_{\mathbf{P}}$ denotes expectation under $\RegM = \mathbf{U}^{(0)} \mathbf{P} \mathbf{D}^* \mathbf{V}^{*\top}$. The first inequality above holds since for any estimator $\widehat{\RegM}$, the induced estimator $\widehat{\mathbf{P}} = \mathbf{U}^{(0)\top}\widehat{\RegM}\mathbf{V}^* \mathbf{D}^{*-1}$ is the orthogonal projection of $\widehat{\RegM}$ onto the subspace $\{ \mathbf{U}^{(0)} \mathbf{A} \mathbf{V}^{*\top} : \mathbf{A} \in \mathbb{R}^{r_0 \times r} \}$ and therefore cannot increase the Frobenius risk on $\mathcal{C}_u$.

For any $\mathbf{P} \neq \mathbf{P}^{\prime} \in \mathcal{P}_u$, we define
\begin{align*}
\mathbf{C} &= \mathbf{U}^{(0)} \mathbf{P} \mathbf{D}^* \mathbf{V}^{*\top}, \\
\mathbf{C}^{\prime} &= \mathbf{U}^{(0)} \mathbf{P}^{\prime} \mathbf{D}^* \mathbf{V}^{*\top}.
\end{align*}
Let $\mathbb{P}$ be the distribution of $\mathbf{Y}$ under model 
\[
\mathbf{Y} = \mathbf{X} \mathbf{C} + \mathbf{E}.
\]
Under our model assumptions, the rows satisfy that 
\[
\mathbf{y}_i \overset{\mathrm{i.i.d.}}{\sim} \mathcal{N}\left( \mathbf{C}^{\top} \mathbf{x}_i, \sigma^2 \mathbf{I}_q \right) \  \text{ for } 1 \leq i \leq n.
\]
Similarly, let $\mathbb{P}^{\prime}$ be the distribution of $\mathbf{Y}$ under model 
\[
\mathbf{Y} = \mathbf{X} \mathbf{C}^{\prime} + \mathbf{E}.
\]

Let $\bar{c} = \gamma \underline{c}$. Then we can deduce that 
\begin{align*}
D_{\mathrm{KL}}\left( \mathbb{P} \| \mathbb{P}^{\prime} \right)
&= \sum_{i=1}^n
D_{\mathrm{KL}}\left(
\mathcal{N}\left( \mathbf{C}^{\top} \mathbf{x}_i, \sigma^2 \mathbf{I}_q \right)
\|
\mathcal{N}\left( (\mathbf{C}^{\prime})^{\top} \mathbf{x}_i, \sigma^2 \mathbf{I}_q \right)
\right) \\
&= \frac{1}{2\sigma^2} \sum_{i=1}^n
\left\Vert (\mathbf{C} - \mathbf{C}^{\prime})^{\top} \mathbf{x}_i \right\Vert_2^2 \\
&= \frac{1}{2\sigma^2}
\left\Vert \mathbf{X} (\mathbf{C} - \mathbf{C}^{\prime}) \right\Vert_{\mathrm{F}}^2 \\
&\overset{(i)}{\leq} \frac{\nu n}{2\sigma^2}
\left\Vert \mathbf{C} - \mathbf{C}^{\prime} \right\Vert_{\mathrm{F}}^2 \\
&= \frac{\nu n}{2\sigma^2}
\left\Vert \mathbf{U}^{(0)} (\mathbf{P} - \mathbf{P}^{\prime}) \mathbf{D}^* \mathbf{V}^{*\top} \right\Vert_{\mathrm{F}}^2 \\
&= \frac{\nu n}{2\sigma^2}
\left\Vert (\mathbf{P} - \mathbf{P}^{\prime}) \mathbf{D}^* \right\Vert_{\mathrm{F}}^2 \numberthis \label{eq:DKL-PPD} \\
&\overset{(ii)}{\leq} \frac{\bar{c}^2 \nu n}{2\sigma^2}
\left\Vert \mathbf{P} - \mathbf{P}^{\prime} \right\Vert_{\mathrm{F}}^2,
\end{align*}
where step (i) above has used Assumption~\ref{assump:constrained-norm-design}, and step (ii) above has utilized the fact that $d^*_j \leq d^*_1 \leq \bar{c}$ for all $1 \leq j \leq r$. Hence, we have that 
\begin{equation}
\label{eq:proof-lower-bd-frob-upp-bd}
\sup_{\mathbf{P} \neq \mathbf{P}^{\prime}}
\frac{D_{\mathrm{KL}}\left( \mathbb{P} \| \mathbb{P}^{\prime} \right)}
{\left\Vert \mathbf{P} - \mathbf{P}^{\prime} \right\Vert_{\mathrm{F}}^2}
\leq \frac{\bar{c}^2 \nu n}{2\sigma^2}.
\end{equation}

We set $k \coloneqq \lfloor s^*_u / r \rfloor$ and consider the subclass of $\mathcal{P}_u$ defined as
\begin{equation*}
\mathcal{P}_{u,1} \coloneqq
\left\{
\mathbf{P} =
\begin{bmatrix}
\mathbf{P}_1 \\
\mathbf{0}
\end{bmatrix}
\in \mathbb{R}^{r_0 \times r}
: \mathbf{P}_1 \in \mathrm{St}(k,r)
\right\}.
\end{equation*}
This construction is valid since $r_0 \geq s^*_u / r \geq k$, and $\mathcal{P}_{u,1} \subseteq \mathcal{P}_u$ because each $\mathbf{P} \in \mathcal{P}_{u,1}$ satisfies that $|\mathbf{P}|_0 \leq kr \leq s^*_u$.

By resorting to Lemma~\ref{lemma:covering-number-Stiefel-Frob}, there exist universal constants $C_1 > c_1 > 0$ such that for any $0 < \delta < 1$,
\begin{equation*}
r^{\frac{r(r-1)}{4}} \left( \frac{c_1}{\delta} \right)^{kr - \frac{r(r+1)}{2}}
\leq
\mathcal{N}\left( \delta; \mathrm{St}(k,r), \|\cdot\|_{\mathrm{F}} \right)
\leq
r^{\frac{r(r-1)}{4}} \left( \frac{C_1}{\delta} \right)^{kr - \frac{r(r+1)}{2}}.
\end{equation*}
Applying Lemma~\ref{lemma:local-metric-entropy} together with \eqref{eq:proof-lower-bd-frob-upp-bd}, we can show that 
\begin{equation*}
\inf_{\widehat{\mathbf{P}}} \sup_{\mathbf{P} \in \mathcal{P}_{u,1}}
\mathbb{E}_{\mathbf{P}}\left\Vert \widehat{\mathbf{P}} - \mathbf{P} \right\Vert_{\mathrm{F}}^2
\geq \frac{c_1^2}{840 C_1^2}
\left( \frac{2\sigma^2 (kr - r(r+1)/2)}{\bar{c}^2 \nu n} \wedge 1 \right).
\end{equation*}
It follows from $s^*_u \geq r(r+1)$ and the sample size condition~\eqref{eq:sample-lower-bd-requirement} that
\begin{equation*}
\inf_{\widehat{\mathbf{P}}} \sup_{\mathbf{P} \in \mathcal{P}_{u,1}}
\mathbb{E}_{\mathbf{P}}\left\Vert \widehat{\mathbf{P}} - \mathbf{P} \right\Vert_{\mathrm{F}}^2
\geq \frac{c_1^2}{840 C_1^2}
\left( \frac{\sigma^2 kr}{\bar{c}^2 \nu n} \wedge 1 \right)
\gtrsim \frac{\sigma^2 s^*_u}{\bar{c}^2 n}.
\end{equation*}
Consequently, from \eqref{eq:lwb-C-by-P} we can obtain that 
\begin{equation}
\label{eq:Pu-lwb-1}
\inf_{\widehat{\RegM}} \sup_{\RegM \in \mathcal{C}_u}
\mathbb{E}_{\RegM}\left\Vert \widehat{\RegM} - \RegM \right\Vert_{\mathrm{F}}^2
\geq \underline{c}^2 \inf_{\widehat{\mathbf{P}}} \sup_{\mathbf{P} \in \mathcal{P}_u}
\mathbb{E}_{\mathbf{P}}\left\Vert \widehat{\mathbf{P}} - \mathbf{P} \right\Vert_{\mathrm{F}}^2 
\geq \underline{c}^2 \inf_{\widehat{\mathbf{P}}} \sup_{\mathbf{P} \in \mathcal{P}_{u,1}}
\mathbb{E}_{\mathbf{P}}\left\Vert \widehat{\mathbf{P}} - \mathbf{P} \right\Vert_{\mathrm{F}}^2 
\gtrsim \frac{\sigma^2 s^*_u}{\gamma^2 n}.
\end{equation}

We next consider a second subclass of $\mathcal{P}_u$, namely
\begin{equation*}
\mathcal{P}_{u,2} \coloneqq
\left\{
\mathbf{P} =
\begin{bmatrix}
\mathbf{I}_{r-1} & \mathbf{0} \\
\mathbf{0} & \mathbf{v}
\end{bmatrix}
\in \mathbb{R}^{r_0 \times r} :
\|\mathbf{v}\|_2 = 1, |\mathbf{v}|_0 \leq s
\right\},
\end{equation*}
where $s \coloneqq (s^*_u \wedge \lfloor r_0/2 \rfloor) - r + 1$. This subclass is well-defined since $s + r - 1 \leq \lfloor r_0/2 \rfloor \leq r_0$, and it is contained in $\mathcal{P}_u$ because $(r-1) + s = s^*_u \wedge \lfloor r_0/2 \rfloor \leq s^*_u$.
In view of Lemma~\ref{lemma:packing-set-sparse-unit} and the assumption of $r_0 \geq 4r+1$, there exist a finite subset
\[
\{ \mathbf{v}^1, \ldots, \mathbf{v}^M \} \subseteq \mathbb{S}_{r_0-r+1}(s)
\]
and a universal constant $c > 0$ such that for each $j \neq \ell$,
\begin{equation}
\label{eq:v-dist}
\frac{\delta}{2}
<
\| \mathbf{v}^j - \mathbf{v}^{\ell} \|_2
\leq
2\delta
\end{equation}
and
\begin{equation}
\label{eq:logM-lowerd-bound}
\log M \geq c \bigl( (s^*_u \wedge \lfloor r_0 / 2 \rfloor) - r \bigr)
\log\left( \frac{r_0-r}{(s^*_u \wedge \lfloor r_0/2 \rfloor) - r} \right)
\quad \text{for all } 0 < \delta \leq 1.
\end{equation}

Let us define
\begin{equation*}
\mathbf{P}^j \coloneqq
\begin{bmatrix}
\mathbf{I}_{r-1} & \mathbf{0} \\
\mathbf{0} & \mathbf{v}^j
\end{bmatrix}
\in \mathcal{P}_{u,2},
\qquad j \in [M].
\end{equation*}
Then it holds that for any $j,\ell$,
\[
\left\Vert (\mathbf{P}^j - \mathbf{P}^{\ell}) \mathbf{D}^* \right\Vert_{\mathrm{F}}
= d^*_r \| \mathbf{v}^j - \mathbf{v}^{\ell} \|_2
= \underline{c} \| \mathbf{v}^j - \mathbf{v}^{\ell} \|_2,
\]
where the last equality above has used our choice of $\mathbf{D}^*$ with $d^*_r = \underline{c}$.
Thus, in light of \eqref{eq:v-dist}, we have that for each $j \neq \ell$,
\[
\frac{\underline{c} \delta}{2}
<
\left\Vert (\mathbf{P}^j - \mathbf{P}^{\ell}) \mathbf{D}^* \right\Vert_{\mathrm{F}}
\leq
2 \underline{c} \delta.
\]

Denote by $\mathbb{P}^j$ the distribution of $\mathbf{Y}$ under model 
\[
\mathbf{Y} = \mathbf{X} \mathbf{C}^j + \mathbf{E},
\qquad
\mathbf{C}^j = \mathbf{U}^{(0)} \mathbf{P}^j \mathbf{D}^* \mathbf{V}^{*\top}.
\]
It follows from \eqref{eq:DKL-PPD} that 
\[
\max_{j \neq \ell}
D_{\mathrm{KL}}\left( \mathbb{P}^j \| \mathbb{P}^{\ell} \right)
\leq \frac{2 \underline{c}^2 \nu n \delta^2}{\sigma^2}.
\]
Let us set
\begin{equation*}
\delta_n^2 \coloneqq
\frac{c \sigma^2}{8 \underline{c}^2 \nu}
\cdot \frac{(s^*_u \wedge \lfloor r_0 / 2 \rfloor) - r}{n}
\log\left( \frac{r_0-r}{(s^*_u \wedge \lfloor r_0/2 \rfloor) - r} \right).
\end{equation*}
Then under \eqref{eq:sample-lower-bd-requirement}, we have that 
$$\delta_n \leq 1.$$ 
Further, with the aid of \eqref{eq:logM-lowerd-bound} and the first inequality in \eqref{eq:slog-lwd}, we can deduce that 
\[
\log M \geq 4 \log 2.
\]
Moreover, it holds that 
\[
\max_{j \neq \ell}
D_{\mathrm{KL}}\left( \mathbb{P}^j \| \mathbb{P}^{\ell} \right)
\leq \frac{c}{4}
\bigl( (s^*_u \wedge \lfloor r_0 / 2 \rfloor) - r \bigr)
\log\left( \frac{r_0-r}{(s^*_u \wedge \lfloor r_0/2 \rfloor) - r} \right)
\leq \frac{1}{4} \log M.
\]
Then we have that 
\[
\frac{\max_{j \neq \ell} D_{\mathrm{KL}}\left( \mathbb{P}^j \| \mathbb{P}^{\ell} \right) + \log 2}{\log M}
\leq \frac{1}{2}.
\]
Consequently, an application of Proposition~15.12 of~\cite{wainwright2019high} gives that 
\begin{equation*}
\begin{aligned}
\inf_{\widehat{\mathbf{P}}} \sup_{\mathbf{P} \in \mathcal{P}_{u,2}}
\mathbb{E}_{\mathbf{P}}\left\Vert (\widehat{\mathbf{P}} - \mathbf{P}) \mathbf{D}^* \right\Vert_{\mathrm{F}}^2
&\geq \frac{\underline{c}^2}{32} \delta_n^2 \\
&= \frac{c \sigma^2}{256 \nu}
\cdot \frac{(s^*_u \wedge \lfloor r_0 / 2 \rfloor) - r}{n}
\log\left( \frac{r_0-r}{(s^*_u \wedge \lfloor r_0/2 \rfloor) - r} \right) \\
&\gtrsim \frac{\sigma^2 ((s^*_u \wedge \lfloor r_0 / 2 \rfloor) - r)}{n}
\log\left( \frac{r_0-r}{(s^*_u \wedge \lfloor r_0/2 \rfloor) - r} \right).
\end{aligned}
\end{equation*}

It follows from the assumptions of $\min(s^*_u, s^*_v) \geq r(r+1)$ and $r_0 \geq 4r+1$ that 
\[
(s^*_u \wedge \lfloor r_0 / 2 \rfloor) - r
\geq \frac{1}{2} (s^*_u \wedge \lfloor r_0 / 2 \rfloor),
\qquad
r_0 - r \geq \frac{3}{4} r_0.
\]
Then we can show that 
\begin{equation*}
\inf_{\widehat{\mathbf{P}}} \sup_{\mathbf{P} \in \mathcal{P}_{u,2}}
\mathbb{E}_{\mathbf{P}}\left\Vert (\widehat{\mathbf{P}} - \mathbf{P}) \mathbf{D}^* \right\Vert_{\mathrm{F}}^2
\gtrsim \frac{\sigma^2 (s^*_u \wedge \lfloor r_0 / 2 \rfloor)}{n}
\log\left( \frac{0.75 r_0}{s^*_u \wedge \lfloor r_0/2 \rfloor} \right).
\end{equation*}
Combining this bound with \eqref{eq:lwb-C-by-P}, we can obtain that 
\begin{equation}
\label{eq:Pu-lwb-2}
\begin{aligned}
\inf_{\widehat{\RegM}} \sup_{\RegM \in \mathcal{C}_u}
\mathbb{E}_{\RegM}\left\Vert \widehat{\RegM} - \RegM \right\Vert_{\mathrm{F}}^2
&\geq \inf_{\widehat{\mathbf{P}}} \sup_{\mathbf{P} \in \mathcal{P}_u}
\mathbb{E}_{\mathbf{P}}\left\Vert (\widehat{\mathbf{P}} - \mathbf{P}) \mathbf{D}^* \right\Vert_{\mathrm{F}}^2 \\
&\geq \inf_{\widehat{\mathbf{P}}} \sup_{\mathbf{P} \in \mathcal{P}_{u,2}}
\mathbb{E}_{\mathbf{P}}\left\Vert (\widehat{\mathbf{P}} - \mathbf{P}) \mathbf{D}^* \right\Vert_{\mathrm{F}}^2 \\
&\gtrsim \frac{\sigma^2 (s^*_u \wedge \lfloor r_0 / 2 \rfloor)}{n}
\log\left( \frac{0.75 r_0}{s^*_u \wedge \lfloor r_0/2 \rfloor} \right).
\end{aligned}
\end{equation}

Further, combining \eqref{eq:Pu-lwb-1} and \eqref{eq:Pu-lwb-2} leads to 
\begin{equation}
\label{eq:lwd-Cu}
\inf_{\widehat{\RegM}} \sup_{\RegM \in \mathcal{C}_u}
\mathbb{E}_{\RegM}\left\Vert \widehat{\RegM} - \RegM \right\Vert_{\mathrm{F}}^2
\gtrsim \frac{\sigma^2 s^*_u}{\gamma^2 n}
+ \frac{\sigma^2 (s^*_u \wedge \lfloor r_0 / 2 \rfloor)}{n}
\log\left( \frac{0.75 r_0}{s^*_u \wedge \lfloor r_0/2 \rfloor} \right).
\end{equation}
By the same token, we can deduce that 
\begin{equation}
\label{eq:lwd-Cv}
\inf_{\widehat{\mathbf{C}}} \sup_{\mathbf{C} \in \mathcal{C}_v}
\mathbb{E}_{\mathbf{C}}\left\Vert \widehat{\mathbf{C}} - \mathbf{C} \right\Vert_{\mathrm{F}}^2
\gtrsim \frac{\sigma^2 s^*_v}{\gamma^2 n}
+ \frac{\sigma^2 (s^*_v \wedge \lfloor r_0 / 2 \rfloor)}{n}
\log\left( \frac{0.75 r_0}{s^*_v \wedge \lfloor r_0/2 \rfloor} \right).
\end{equation}
Therefore, combining \eqref{eq:lwd-Cu} and \eqref{eq:lwd-Cv} with \eqref{eq:frob-lwd} yields that 
\begin{multline*}
\inf_{\widehat{\mathbf{C}}} \sup_{\mathbf{C} \in \mathcal{C}}
\mathbb{E}_{\mathbf{C}}\left\Vert \widehat{\mathbf{C}} - \mathbf{C} \right\Vert_{\mathrm{F}}^2 \\
\gtrsim \frac{\sigma^2 (s^*_u + s^*_v)}{\gamma^2 n}
+ \sigma^2 \left(
\frac{s^*_u \wedge \lfloor r_0 / 2 \rfloor}{n}
\log\left( \frac{0.75 r_0}{s^*_u \wedge \lfloor r_0/2 \rfloor} \right)
+ \frac{s^*_v \wedge \lfloor r_0 / 2 \rfloor}{n}
\log\left( \frac{0.75 r_0}{s^*_v \wedge \lfloor r_0/2 \rfloor} \right)
\right).
\end{multline*}
This concludes the proof of Theorem~\ref{thm:lower-bound}.

\clearpage
\putbib[boxinz-papers]
\end{bibunit}

\end{document}